%% file: acv-risk.tex
\documentclass[twoside]{article}

\usepackage[accepted]{aistats2020}
\usepackage[sort&compress]{natbib}

\usepackage{microtype}
\usepackage{booktabs} %
\DeclareTextSymbolDefault{\DH}{T1}
\usepackage[colorlinks,citecolor=blue,urlcolor=black,linkcolor=blue,pdfborder={0 0 0}]{hyperref}

\usepackage{array}
\usepackage{xr}
\usepackage{mathtools}
\usepackage{amsthm,amsmath,amssymb}
\usepackage{url}
\usepackage{comment}
\usepackage{graphicx}
\usepackage{enumitem}
\usepackage{wrapfig}
\usepackage{framed}
\usepackage{subcaption}
\usepackage{xspace}

\newcommand{\E}{\mathbb{E}}

\newcommand{\prox}{\operatorname{\mathrm{prox}}}

\def\namedlabel#1#2{\begingroup
    #2%
    \def\@currentlabel{#2}%
    \phantomsection\label{#1}\endgroup
}
\usepackage[capitalize]{cleveref}  

\crefname{nlem}{Lemma}{Lemmas}
\crefname{nprop}{Proposition}{Propositions}
\crefname{ncor}{Corollary}{Corollaries}
\crefname{exa}{Example}{Examples}
\crefname{theorem}{Thm.}{Thms.}
\crefname{proposition}{Prop.}{Props.}
\crefname{assumption}{Assump.}{Assumps.}
\crefname{equation}{}{}
\crefname{section}{Sec.}{Secs.}
\crefname{appendix}{App.}{Apps.}
\usepackage{datetime}
\RequirePackage[usenames,dvipsnames]{color}
\usepackage{mathrsfs}
\usepackage{autonum} 
\usepackage{scalerel,stackengine}
\stackMath
\newcommand\reallywidehat[1]{%
\savestack{\tmpbox}{\stretchto{%
  \scaleto{%
    \scalerel*[\widthof{\ensuremath{#1}}]{\kern.1pt\mathchar"0362\kern.1pt}%
    {\rule{0ex}{\textheight}}%
  }{\textheight}%
}{2.4ex}}%
\stackon[-6.9pt]{#1}{\tmpbox}%
}

\graphicspath{ {./}  {../Figures/} {../stephenson_code/} }

\allowdisplaybreaks[3]

\input{lwm-macros}
\renewcommand{\mineig}{\mathop\mathrm{mineig}}

\usepackage{todonotes}

\renewcommand{\bs}{\boldsymbol}

\newtheoremstyle{break}
  {\topsep}{\topsep}%
  {\itshape}{}%
  {\bfseries}{}%
  {\newline}{}%
\theoremstyle{break}

\DeclareMathSymbol{\shortminus}{\mathbin}{AMSa}{"39}
\newcommand{\ACV}{{\text{\bf ACV}}\xspace}
\newcommand{\CV}{{\text{\bf CV}}\xspace}
\newcommand{\ACVIJ}{{\text{\bf ACV}}^{\text{IJ}}\xspace}
\newcommand{\ProxACV}{{\text{\bf ProxACV}}\xspace}
\newcommand{\HOACV}{{\text{\bf ACV}_p^{\text{HO}}}}

\newcommand{\ProxACVIJ}{{\text{\bf ProxACV}}^{\text{IJ}}\xspace}
\newcommand{\hoproxacvest}{\acvest{i}^{\text{Prox-HO}_p}}
\newcommand{\rhoproxacvest}{\acvest{i}^{\text{Prox-RHO}_p}}
\newcommand{\proxHOACV}{{\text{\bf ProxACV}_p^{\text{HO}}}}
\newcommand{\proxRHOACV}{{\text{\bf ProxACV}_p^{\text{RHO}}}}

\newcommand{\loss}[0]{\ell} %
\newcommand{\reg}[0]{\pi} %
\newcommand{\obj}[0]{m} %

\newcommand{\est}[0]{\hat{\beta}} %
\newcommand{\cvest}[1]{\hat{\beta}_{\shortminus{#1}}} %
\newcommand{\acvest}[1]{\tilde{\beta}_{\shortminus{#1}}} %
\newcommand{\hoacvest}[1]{\acvest{#1}^{\text{HO}_p}}
\newcommand{\rhoacvest}[1]{\acvest{#1}^{\text{RHO}_p}}
\newcommand{\acvijest}[1]{\tilde{\beta}_{\shortminus{#1}}^{\text{IJ}}} %
\newcommand{\proxacvest}[1]{\tilde{\beta}_{\shortminus{#1}}^{\text{prox}}} %
\newcommand{\proxacvijest}[1]{\tilde{\beta}_{\shortminus{#1}}^{\text{prox},\text{IJ}}} %
\newcommand{\acvhess}[1]{\mathrm{H}_{{#1}}} %

\newcommand{\Ppop}{\P} %
\newcommand{\Pemp}{\P_n} %
\newcommand{\Pcv}[1]{\P_{n,\shortminus{#1}}} %
\newcommand{\xset}[0]{\mathcal{X}} %
\DeclareMathOperator{\interior}{interior} %

\renewcommand{\norm}{\staticnorm}
\newcommand{\Lip}[1]{\operatorname{Lip}_{#1}} %
\newcommand{\lineseg}{\mathcal{L}} %
\newcommand{\lam}{\lambda}
\newcommand{\Lam}{\Lambda}

\newcommand{\I}{\mathrm{I}} %
\newcommand{\gradlossboundmixed}[2]{\mathrm{B}^{\loss}_{{#1},{#2}}}

\begin{document}

\runningtitle{Approximate Cross-validation: Guarantees for Model Assessment and Selection}

\twocolumn[

\aistatstitle{%
Approximate Cross-validation:\\ Guarantees for Model Assessment and Selection}

\aistatsauthor{ Ashia Wilson \And Maximilian Kasy \And Lester Mackey }

\aistatsaddress{ Microsoft Research \And  Harvard University \And Microsoft Research} ]

\begin{abstract}
        Cross-validation (CV) is a popular approach for assessing and selecting predictive models.
        However, when the number of folds is large, CV suffers from a need to repeatedly refit a learning procedure on a large number of training datasets.
        Recent work in empirical risk minimization (ERM) approximates the expensive refitting with a single Newton step warm-started from the full training set optimizer.
        While this can greatly reduce runtime,  several open questions remain including whether these approximations lead to faithful model selection and whether they are suitable for non-smooth objectives.
        We address these questions with three main contributions:
        (i) we provide uniform non-asymptotic, deterministic model assessment guarantees for approximate CV;
        (ii) we show that (roughly) the same conditions also guarantee model selection performance comparable to CV;
        (iii) we provide a proximal Newton extension of the approximate CV framework for non-smooth prediction problems and develop improved assessment guarantees for problems such as $\ell_1$-regularized ERM.   
\end{abstract}

\section{Introduction}
Two important concerns when fitting a predictive model are \emph{model assessment} -- estimating the expected performance of the model on a future dataset sampled from the same distribution -- and \emph{model selection} -- choosing the model hyperparameters to minimize out-of-sample prediction error. 
Cross-validation (CV)~\citep{Stone1974,geisser1975predictive} is one of the most widely used techniques for assessment and selection, but it suffers from the need to repeatedly refit a learning procedure on different data subsets.
To reduce the computational burden of CV, recent work proposes to replace the expensive model refitting with an inexpensive surrogate.
For example, in the context of regularized empirical risk minimization (ERM), two popular techniques both approximate leave-one-out CV by taking Newton steps from the full-data optimized objective \cite[see, e.g.,][]{debruyne2008model,pmlr-v32-liua14,beirami2017optimal,rad2019scalable,giordano2018return}.
The literature provides single-model guarantees for the assessment quality of these Newton approximations for certain classes of regularized ERM models.
Two open questions are whether these approximations are suitable for model selection and whether they are suitable for non-smooth objectives, such as $\ell_1$-penalized losses.  As put by~\citep{Stephenson2019sparse},
``understanding the uses and limitations of approximate CV for selecting $\lam$ is one of the most important directions for future work in this area.''
We address these important open problems in this work.

Our principal contributions are three-fold. %
\begin{itemize}[topsep=0pt]
  \setlength\itemsep{-.1em}
\item We provide uniform guarantees for \emph{model assessment} using approximate CV. Specifically, we give conditions which guarantee that the difference between CV and approximate CV is uniformly bounded by a constant of order $1/n^2$, where $n$ is the number of CV folds.
In contrast to existing guarantees, our results are non-asymptotic, deterministic, and uniform in $\lam$; our results do not assume a bounded parameter space and provide a more precise convergence rate of $O(1/n^2)$. %
\item We provide guarantees for \emph{model selection}.
We show that roughly the same conditions that guarantee uniform quality assessment results also guarantee that estimators based on parameters tuned by approximate cross-validation and by cross-validation are within $O(1/n)$ distance of each other, so that the approximation error is negligible relative to the sampling variation.
\item We propose a \emph{generalization of approximate CV} that works for general non-smooth penalties. This generalization is based on the proximal Newton method~\citep{lee2014proximal}. We provide strong model assessment guarantees for this generalization and demonstrate that past non-smooth extensions of ACV fail to satisfy these strong guarantees.
\end{itemize}

\paragraph{Notation}
Let $[n] \defeq \{1,\dots,n\}$, $\mathrm{I}_d$ be the $d\times d$ identity matrix, and $\partial\varphi$ denote the subdifferential of a function $\varphi$ \citep{Rockafellar1970ConvexAnalysis}.
For any matrix or tensor $H$, we define $\opnorm{H} \defeq \sup_{v\neq 0\in\reals^d}{\opnorm{H[v]}}{/\twonorm{v}}$ where $\opnorm{v} \defeq \twonorm{v}$ is the Euclidean norm.
For any Lipschitz vector, matrix, or tensor-valued function $f$ with domain $\dom(f)$, we define the Lipschitz constant $\Lip{}(f) \defeq \sup_{x\neq y\in\dom(f)}\frac{\opnorm{f(x) - f(y)}}{\twonorm{x-y}}$.

\section{Cross-validation for Regularized Empirical Risk Minimization}
For a given datapoint $z\in\xset$ and candidate parameter vector $\beta \in\reals^d$, consider the objective function
\balignt\label{eq:objective}
\obj(z, \beta, \lambda) = \loss(z, \beta) + \lambda\reg(\beta)
\ealignt
comprised of a loss function $\loss$, a regularizer $\reg$, and a regularization parameter $\lambda \in [0,\infty]$.
 Common examples of loss functions are the least squares loss for regression and the exponential and logistic losses for classification; common examples of regularizers are the $\ell_2^2$ (ridge) and $\ell_1$ (Lasso) penalties.
Our interest is in assessing and selecting amongst estimators fit via \emph{regularized empirical risk minimization (ERM)}:%
$$
\est(\lambda) 
    \defeq \begin{cases}
    \argmin_\beta \loss(\Pemp,\beta) + \lambda \reg(\beta) & \lam \in [0,\infty)\\
    \argmin_\beta \reg(\beta) & \lam = \infty.
    \end{cases}
$$
Here, $\Pemp \defeq \frac{1}{n} \sum_{i=1}^n \delta_{z_i}$ is an empirical distribution over a given training set with datapoints $z_1, \dots, z_n \in\xset$, and we overload notation to write
$\loss(\mu, \beta) 
    \defeq \int \loss(z, \beta) d\mu(z)
\text{ and }
\obj(\mu, \beta, \lambda) 
    \defeq \loss(\mu, \beta) + \lam\reg(\beta)
$
for any measure $\mu$ on $\xset$ under which $\loss$ is integrable.

A standard tool for both model assessment and model selection is the leave-one-out cross-validation (CV)\footnote{We will focus on leave-one-out CV for concreteness, but our results directly apply to any variant of CV, including $k$-fold and leave-pair-out CV, by treating each fold as a (dependent) datapoint. Leave-pair-out CV is often recommended for AUC estimation~\citep{pmlr-v8-airola10a,AIROLA20111828} but is demanding even for small datasets as ${n \choose 2}$ folds are required.} estimate of risk \citep{Stone1974,geisser1975predictive}
\balignt\label{eq:cv-err}
\CV(\lambda) = \tfrac{1}{n} \textsum_{i=1}^n \loss(z_i, \cvest{i}(\lambda))
\ealignt
which is based on the leave-one-out estimators
\balignt\label{eq:cvest}
\cvest{i}(\lambda) 
&= \argmin_\beta \loss(\Pcv{i},\beta) + \lambda\reg(\beta)\\
&= \argmin_\beta \tfrac{1}{n}\textsum_{j \neq i} \loss(z_j,\beta) + \lambda\reg(\beta) 
\ealignt
for $\Pcv{i} \defeq \frac{1}{n} \sum_{j\neq i} \delta_{z_j}$.
Unfortunately, performing leave-one-out CV entails  solving an often expensive optimization problem $n$ times for every value of $\lam$ evaluated; this makes model selection with leave-one-out CV especially burdensome.

\section{Approximating Cross-validation}
To provide a faithful estimate of CV while reducing its computational cost,~\citet{beirami2017optimal} (see also \citep{rad2019scalable})
considered the following \emph{approximate cross-validation} (ACV) \emph{error}
\balignt\label{eq:acv1} 
\ACV(\lambda) 
    &\defeq \frac{1}{n}\sum_{i=1}^n \loss(z_i, \acvest{i}(\lambda)) 
\ealignt
based on the approximate leave-one-out CV estimators
\balignt
\label{eq:acvest}
\acvest{i}(\lambda) 
   \hspace{-.075cm} = \hspace{-.075cm}\est(\lambda)\hspace{-.075cm} +\hspace{-.075cm}\Hess_\beta \obj(\Pcv{i}, \est(\lambda), \lambda)^{-1} \frac{\grad_\beta \loss(z_i, \est(\lambda))}{n} %
\ealignt
In effect, \ACV replaces the task of solving a leave-one-out optimization problem \cref{eq:cvest} with taking a single Newton step \cref{eq:acvest} and realizes computational speed-ups when the former is more expensive than the latter. This approximation requires that the objective be everywhere twice-differentiable and therefore does not directly apply to non-smooth ERM problems such as the Lasso. We revisit this issue in \cref{sec:proximal}.

\subsection{Optimizer Comparison}
Each ACV estimator \cref{eq:acvest} can also be viewed as the optimizer of a second-order Taylor approximation to the leave-one-out objective \cref{eq:cvest}, expanded around the full training sample estimate $\est(\lambda)$:
\balign
\acvest{i}(\lambda) 
    = \textstyle\argmin_\beta&\   \widehat{\obj}_2(\Pcv{i}, \beta, \lam; \est(\lam)) \qtext{for}\\
\widehat{\obj}_{2}(\Pcv{i}, \beta, \lam; \est(\lam))
   &\defeq \sum_{k=0}^2 \textfrac{\nabla_\beta^k \obj(\Pcv{i}, \est(\lambda),\lam)[\beta - \est(\lambda)]^{\otimes k}}{k!}.
\ealign
This motivates our optimization perspective on analyzing \ACV.  To understand how well \ACV approximates \CV we need only understand how well the optimizers of two related optimization problems approximate one another.
As a result, the workhorse of our analysis is the following key lemma, proved in
\cref{App:opt-comp}, which controls the difference between the optimizers of similar objective functions.
In essence, two optimizers are close if their objectives (or objective gradients) are close and at least one objective has a sharp---that is, not flat---minimum.
\begin{lemma}[Optimizer comparison] \label{optimizer-comparison} 
Suppose
\balignt\label{eq:obj} 
x_{\varphi_1} \in \argmin_x \varphi_1(x)  \qtext{and} x_{\varphi_2} \in \argmin_x \varphi_2(x).
\ealignt
If each $\varphi_i$ admits an $\nu_{\varphi_i}$ error bound \cref{eq:error_bound}, defined in \cref{def:eb-and-gg} below, 
then
\begin{align}\label{eq:opt_comp_error_bound}
 &\nu_{\varphi_1}(\twonorm{x_{\varphi_1} - x_{\varphi_2}})
 + \nu_{\varphi_2}(\twonorm{x_{\varphi_1} - x_{\varphi_2}}) \\
    &\leq \varphi_2(x_{\varphi_1}) - \varphi_1(x_{\varphi_1}) - (\varphi_2(x_{\varphi_2}) 
    - \varphi_1(x_{\varphi_2})).
 \end{align}
If $\varphi_2-\varphi_1$ is differentiable and $\varphi_2$ has $\nu_{\varphi_2}$ gradient growth \cref{eq:gradient_growth}, defined in \cref{def:eb-and-gg} below, %
 then
\balignt
\label{eq:opt_comp_growth}
    &\nu_{\varphi_2}(\twonorm{x_{\varphi_1} - x_{\varphi_2}}) %
    \leq 
    \inner{x_{\varphi_1}-x_{\varphi_2}}{\grad(\varphi_2-\varphi_1)(x_{\varphi_1})}.
 \ealignt
\end{lemma}
This result relies on two standard ways of measuring the sharpness of objective function minima:
\begin{definition}[Error bound and gradient growth]
\label{def:eb-and-gg}
Consider the generalized inverse $\nu(r) \defeq \inf\{s : \omega(s) \geq r\}$ of any non-decreasing function $\omega$ with $\omega(0) = 0$.
We say a function $\varphi$ admits an \emph{$\nu$ error bound}~\citep{bolte2017errorbound} if 
\balignt
\label{eq:error_bound}
&\nu(\twonorm{x - x^*}) \leq  \varphi(x) - \varphi(x^*) 
\ealignt
{for} 
$x^* = \argmin_{x'} \varphi(x')$
{and all} $x\in\reals^d$.
We say a function $\varphi$ has \emph{$\nu$ gradient growth}~\citep{Nesterov08} if
$\varphi$ is subdifferentiable and
\balignt\label{eq:gradient_growth}
&\nu(\twonorm{x- y}) 
\leq \inner{y - x}{u - v}
\ealignt
for all $x, y\in \reals^d$ 
{and all} $u\in\partial\varphi(y), v\in\partial\varphi(x).$
\end{definition}
Notably, if $\varphi$ is $\mu$-strongly convex, then $\varphi$ admits an $\nu_\varphi(r) \defeq \frac{\mu}{2}r^2$ error bound and $\nu_\varphi(r) \defeq \mu r^2$  gradient growth,
but even non-strongly-convex functions can satisfy  quadratic error bounds \citep{Karimi2016PL}.
\subsection{Model Assessment}
We now present a deterministic, non-asymptotic approximation error result for $\ACV$ when used to approximate $\CV$ for a collection of models indexed by $\lambda \in\Lambda$.
Importantly for the model selection results that follow, \cref{thm:acv-approximates-cv} shows that the ACV error is an $O(1/n^2)$ approximation to CV error uniformly in $\lambda$:
\begin{theorem}[\ACV-\CV assessment error]\label{thm:acv-approximates-cv}
If \cref{gradlossboundmixed,curvedobj,HessobjLipschitz} below 
hold for some $\Lambda \subseteq [0,\infty]$ and each $(s,r) \in\{ (0,3),(1,3),(1,4)\}$,
then, for each $\lam \in \Lambda$,
\balignt\label{eq:acv-cv-approx}
|\ACV(\lambda) - \CV(\lambda)|&\leq 
   \frac{\kappa_2}{n^2}\frac{ \gradlossboundmixed{0}{3}}{c_\obj^2}
	+\frac{\kappa_2}{n^3} \frac{\gradlossboundmixed{1}{3}}{c_\obj^3} 
	+\frac{\kappa_2^2}{n^4} \frac{\gradlossboundmixed{1}{4}}{2c_\obj^4},
\ealignt 
\newcommand{\acvcvapproxspacing}{{\hspace{-.095cm}}}
\acvcvapproxspacing for  $\kappa_p\acvcvapproxspacing \defeq \acvcvapproxspacing \sup_{\lam\geq 0}\hspace{-.07cm} \tfrac{C_{\loss, p+1} + \lambda C_{\reg,p+1}}{p!(c_\loss + \lambda c_{\pi} \indic{\lam \geq \lam_\pi})} \acvcvapproxspacing\leq \acvcvapproxspacing \max(\tfrac{C_{p+1,\lam_\reg}}{p!c_\loss}, \acvcvapproxspacing \tfrac{C_{\reg,p+1}}{p!c_\reg}).$

\end{theorem}

This result, proved in \cref{App:higher-order-assess}, relies on the following three assumptions:
\begin{assumption}[Curvature of objective] 
\label{curvedobj} 
For some $c_\loss, c_\reg, c_\obj > 0$ and $\lam_\reg <\infty$, all $i \in [n]$, and all $\lam, \lam'$ in a given $\Lambda \subseteq [0,\infty]$,
$\obj(\Pcv{i},\cdot, \lam)$ has $\nu_\obj(r) = c_\obj r^2$ gradient growth and,
for $c_{\lam',\lam}\defeq c_\loss + \lambda' c_{\pi} \indic{\lam \geq \lam_\pi}$,
\balignt\label{eq:curvature_obj}
\Hess_\beta \obj(\Pcv{i}, \est(\lam),\lambda') \succeq  c_{\lam',\lam} \mathrm{I}_d. %
\ealignt
\end{assumption}
\begin{assumption}[Bounded moments of loss derivatives] \label{gradlossboundmixed}
For given $s,r\geq 0$ and $\Lambda \subseteq [0,\infty]$,  $\gradlossboundmixed{s}{r} <\infty$ where
\balign%
\gradlossboundmixed{s}{r} 
    &\defeq %
    \sup_{\lambda\in\Lambda}\textfrac{1}{n}\textsum_{i=1}^n
	\Lip{}(\grad_\beta\loss(z_i,\cdot))^s\twonorm{\grad_\beta \loss(z_i, \est(\lambda))}^r.
\ealign
\end{assumption}
\begin{assumption}[Lipschitz Hessian of objective] \label{HessobjLipschitz}
For a given $\Lambda \subseteq [0,\infty]$ and
 some $C_{\loss,3}, C_{\reg,3}<\infty$, 
\balignt
\Lip{}(\Hess_\beta \obj(\Pcv{i},\cdot, \lambda)) 
\leq C_{\loss,3} + \lam C_{\reg,3},\ \forall \lambda\in\Lambda, i\in[n].
\ealignt
\end{assumption}
\cref{curvedobj} ensures  the leave-one-out objectives have curvature near their minima, while \cref{gradlossboundmixed} bounds the average discrepancy between the full-data and leave-one-out objectives.  Together, \cref{curvedobj,gradlossboundmixed} ensure that the leave-one-out estimates $\cvest{i}(\lambda)$ are not too far from the full-data estimate $\est(\lambda)$ on average.
Meanwhile, \cref{HessobjLipschitz} ensures that the leave-one-out objective is well-approximated by its second-order Taylor expansion and hence that the ACV estimates $\acvest{i}(\lambda)$ are close to the CV estimates $\cvest{i}(\lambda)$. 

\cref{thm:acv-approximates-cv}
most resembles the (fixed dimension) results obtained by~\citet[Sec. A.9]{rad2019scalable}, who show $|\ACV(\lambda) - \CV(\lambda)| =  o_p(1/n)$ under \iid\ sampling of the datapoints $z_i$, mild regularity conditions, and the convergence assumptions $\est(\lam) \toprob \beta^\ast(\lam)$ and $\cvest{i}(\lam) \toprob \beta^\ast(\lam)$ for some deterministic $\beta^\ast(\lam)$. Notably, their guarantees target only the linear prediction setting where $\loss(z_i,\beta) = \phi(y_i, \inner{x_i}{\beta})$, and the dependence of the constants in their bound on $\lambda$ is not discussed.  

For each $\lambda$,  \citet[Thm. 1]{beirami2017optimal} provide an asymptotic, probabilistic analysis of the \ACV estimators \cref{eq:acvest} under an assumption that $\est(\lam) \toprob \beta^*(\lam) \in \interior(\Theta)$ where $\Theta$ is compact. Specifically, for each value of $\lam$, they guarantee that $\infnorm{\cvest{i}(\lam) - \acvest{i}(\lam)} = O_p({C_\lam/}{n^2})$ for a constant $C_\lam$ depending on $\lam$ in a way that is not discussed.
Our \cref{thm:acv-approximates-cv} is a consequence of the following similar bound on the estimators employed by cross-validation \cref{eq:cvest} and approximate cross-validation \cref{eq:acvest} (see \cref{thm:acv-approximates-cv-HO-full} in \cref{App:higher-order-assess}):
$
\twonorm{\acvest{i}(\lambda) - \cvest{i}(\lambda)} 
    \leq 
   \frac{\kappa_2}{c_m^2 n^2}\twonorm{\grad_\beta \loss(z_i, \est(\lambda))}^2.$
In comparison to both~\citep{rad2019scalable} and \citep{beirami2017optimal}, our results are non-asymptotic, deterministic, and uniform in $\lam$. They provide a more precise convergence rate than \cite[Sec. A.9]{rad2019scalable}, hold outside of the linear prediction setting, and require no compactness assumptions on the domain of $\beta$.

While \citep{rad2019scalable,beirami2017optimal} assume both a strongly convex objective and a bounded parameter space, our analysis shows that a separate boundedness assumption on the parameter space is unnecessary; strong convexity alone ensures that $\hat{\beta}(\lambda)$ is uniformly bounded in $\lambda$ even when the objective and its gradients are unbounded in $\beta$. 
Subsequently, our results apply both to strictly convex objectives (like unregularized logistic regression) when restricted to a compact set and to strongly convex objectives (like ridge-regularized logistic regression) without any domain restrictions. 

Moreover, the assumptions underlying \cref{thm:acv-approximates-cv} and the other results in this work all 
hold under standard, easily verified conditions on the objective:
\begin{proposition}[Sufficient conditions for assumptions] \label{assumptions-hold}
\leavevmode
\vspace{-2\baselineskip}
\begin{enumerate}[leftmargin=.5cm]
\item \cref{HessobjLipschitz} holds for $\Lambda \subseteq [0,\infty]$ 
with  $C_{\reg,3} = \Lip{}(\Hess \reg)$ and $C_{\loss,3} = \max_{i\in[n]} \Lip{}(\Hess_\beta \loss(\Pcv{i},\cdot))$.
\item If $\reg$ admits an error bound \cref{eq:error_bound} with increasing $\nu_\reg$, and $\loss$ is nonnegative, then 
\balignt\label{eq:cvest-bounded}
\est(\lam) \to \est(\infty) \qtext{as} \lam \to \infty.
\ealignt 
\item \cref{curvedobj} holds for $\Lambda \subseteq [0,\infty]$ if $\obj(\Pcv{i}, \cdot, \lam)$ is $c_\obj$-strongly convex  $\forall\lambda\in\Lambda$ and $i\in[n]$, $\reg$ is strongly convex on a neighborhood of $\est(\infty)$, and \cref{eq:cvest-bounded} holds.
\item 
 \cref{gradlossboundmixed} holds for $\Lambda \subseteq [0,\infty]$ and $(s,r)$ with
$\label{eq:gradlossboundmixed-sufficient}
\gradlossboundmixed{s}{r} \leq \frac{1}{n}\sum_{i=1}^n L_i^s(\twonorm{\grad_\beta\loss(z_i, \est(\infty))}
+
\frac{n-1}{n}\frac{L_i}{c_\obj} \twonorm{\grad_\beta\loss(\Pemp, \est(\infty))})^r
$
if $\obj(\Pcv{i}, \cdot, \lam)$ is $c_\obj$-strongly convex  
and $L_i \defeq  \Lip{}(\grad_\beta\loss(z_i,\cdot)) < \infty$ for each $\lambda\in\Lambda$ and $i\in[n]$.
\end{enumerate}
\end{proposition}
While this result, proved in \cref{app:assumptions-hold}, is a deterministic statement, it has an immediate probabilistic corollary:
if the datapoints $z_1, \dots, z_n$ are i.i.d.\ draws from a distribution $\Ppop$, then, under the conditions of \cref{assumptions-hold}, 
\balignt
C_{\loss,3} 
	\leq \, &3\Esub{Z\sim\Ppop}[ \Lip{}(\Hess_\beta\loss(Z,\cdot))^{4}]^{1/4} \qtext{and}\\
\gradlossboundmixed{s}{r} 
	\leq\, &2\Esub{Z\sim\Ppop}[ \Lip{}(\grad_\beta\loss(Z,\cdot))^{2s} (\twonorm{\grad_\beta\loss(Z, \est(\infty))} \\
		&+ \frac{1}{c_m}\Lip{}(\grad_\beta\loss(Z,\cdot))\twonorm{\grad_\beta\loss(Z, \est(\infty))})^{2r}]^{1/2}
\ealignt 
with high probability by Markov's inequality.\footnote{Note that $\est(\infty) = \argmin_\beta \reg(\beta)$ is data-independent.}

\subsection{Infinitesimal Jackknife}
\citet{giordano2018return} (see also \citep{debruyne2008model,pmlr-v32-liua14}) recently studied a second approximation to leave-one-out cross-validation,
 \balignt\label{eq:acvij}
\ACVIJ(\lambda) \defeq \frac{1}{n}\sum_{i=1}^n \loss(z_i, \acvijest{i}(\lambda)),
\ealignt
based on the \emph{infinitesimal jackknife} (IJ)~\citep{jaeckel1972infinitesimal,efron1982jackknife} 
estimate
 \balignt\label{eq:acvijest}
 \acvijest{i}(\lambda)\hspace{-.075cm} \defeq \hspace{-.075cm}\est(\lambda)\hspace{-.075cm} +\hspace{-.075cm}  \Hess_\beta \obj(\Pemp, \est(\lambda), \lambda)^{-1} \frac{\grad_\beta \loss(z_i, \est(\lambda))}{n},
\ealignt
A potential computational advantage of $\ACVIJ$ over \ACV is that $\ACVIJ$ requires only a single Hessian inversion, while \ACV performs $n$ Hessian inversions.\footnote{Note however that for many losses the Hessians of \cref{eq:acvest,eq:acvijest} differ only by a rank-one update so that all $n$ Hessians can be inverted in time comparable to inverting $1$.}

The following theorem, proved in \cref{app:acv-approximates-acvij}, shows that, under conditions similar to those of \cref{thm:acv-approximates-cv}, \ACV and $\ACVIJ$ are nearly the same. %
\begin{theorem}[$\ACVIJ$-\ACV assessment error]\label{thm:acv-approximates-acvij} 
If \cref{gradlossboundmixed,curvedobj} hold for some $\Lambda \subseteq [0,\infty]$ and each $(s,r) \in \{(1,2), (2,2), (3,2)\}$,
then, for each $\lambda \in \Lambda$,
\balignt
|\ACV^{\text{\em IJ}}(\lambda) - \ACV(\lambda)| &\leq \frac{\gradlossboundmixed{1}{2}}{c_{\lam,\lam}^2n^2} +\frac{\gradlossboundmixed{2}{2}}{c_{\lam,\lam}^3n^3} +  \frac{\gradlossboundmixed{3}{2}}{2c_{\lam,\lam}^4n^4}.  %
\ealignt 
\end{theorem}
\cref{thm:acv-approximates-acvij} ensures that all of the assessment and selection guarantees for $\ACV$ in this work also extend to $\ACVIJ$. 
In particular, \cref{thm:acv-approximates-cv,thm:acv-approximates-acvij} together imply $\sup_{\lambda\in\Lambda} |\ACVIJ(\lambda)-\CV(\lambda)| = O(1/n^2)$.
A similar $\ACVIJ$-\CV comparison could be derived from the general infinitesimal jackknife analysis of \citet[Cor. 1]{giordano2018return}, which gives
$\twonorm{\acvijest{i}(\lam) - \cvest{i}(\lam)} \leq C_
\lam/n^2$.
However, the constant $C_\lam$ in \citep[Cor. 1]{giordano2018return} is unbounded in $\lam$ for non-strongly convex regularizers.
Our results only demand curvature from $\reg$ in the neighborhood of its minimizer and thereby establish $\sup_{\lambda\in\Lambda} |\ACVIJ(\lambda)-\CV(\lambda)| = O(1/n^2)$ even for robust, non-strongly convex regularizers like the pseudo-Huber penalty~\cite[Sec. A6.8]{HartleyZi2004},
$
\pi_\delta(\beta) = \sum_{j=1}^d \delta^2 (\sqrt{1+(\beta_j/\delta)^2} - 1).
$
In addition, our analyses avoid the compact domain assumption of \citep[Cor. 1]{giordano2018return}.

\subsection{Higher-order Approximations to CV}
\label{sec:higher-order}

The optimization perspective adopted in this paper naturally points towards generalizations of the estimators \cref{eq:acvest} and \cref{eq:acvijest}. In particular, stronger assessment guarantees can be provided for regularized higher-order Taylor approximations of the objective function. For example, for the regularized $p$-th order approximation,
\balignt
\ACV_p(\lambda) &\defeq \frac{1}{n} \sum_{i=1}^n \loss(z_i, \rhoacvest{i}(\lambda)) \qtext{with}\\
\rhoacvest{i}(\lambda) &\defeq \argmin_\beta \widehat{\obj}_p(\Pcv{i},\beta, \lambda;\est(\lambda))\\ &+ \frac{\Lip{}(\grad_\beta^p \obj(\Pcv{i},\cdot, \lambda)) }{p+1}\|\beta - \est(\lambda)\|_2^{p+1},
\ealignt
where $\hat{m}_p$ is a $p$-th order Taylor expansion of the objective defined by $\hat{f}_{p}(\beta;\est(\lambda))\defeq \sum_{k=0}^{p} \frac{1}{k!} \nabla^k f(\est(\lambda))[\beta - \est(\lambda)]^{\otimes k} $, %
we obtain the following improved assessment guarantee, proved in \cref{App:higher-order-assess}:
\begin{theorem}[$\ACV_p$-\CV assessment error]\label{thm:acv-approximates-cv-ho} If \cref{hocurvedobj,,gradlossboundmixed,,LipschitzpObj} hold for some $\Lambda \subseteq[0,\infty]$ and each $(s,r) \in \{ (0,p+1), (1,p+1), (1,2p)\}$, then, for $\kappa_p$ defined in \cref{thm:acv-approximates-cv} and each $\lam \in \Lambda$, 
\balignt
|\ACV_p(\lambda) \hspace{-.075cm}-\hspace{-.075cm} \CV(\lambda)| 
\hspace{-.075cm}\leq \hspace{-.075cm}\frac{2\kappa_p}{n^p} (\frac{\gradlossboundmixed{0}{p+1}}{c_{\lam,\lam}^p} \hspace{-.075cm} + \hspace{-.075cm}\frac{\gradlossboundmixed{1}{p+1}}{nc_{\lam,\lam}^{p+1}}\hspace{-.05cm} +\hspace{-.075cm} \frac{\kappa_{p}\gradlossboundmixed{1}{2p}}{n^{p}c_{\lam,\lam}^{2p}}).
\ealignt
\end{theorem}
This result relies on the following curvature and smoothness assumptions, which replace \cref{curvedobj,HessobjLipschitz}.

\newcounter{tmp}
\setcounter{assumption}{\thetmp}
\begingroup
\setcounter{tmp}{\value{assumption}-2}
\renewcommand\theassumption{\arabic{assumption}b}
\begin{assumption}[Curvature of objective] 
\label{hocurvedobj} 
For some $c_\loss, c_\reg > 0$ and $\lam_\reg <\infty$, all $i \in [n]$, and all $\lam$ in a given $\Lambda \subseteq [0,\infty]$,
$\obj(\Pcv{i},\cdot, \lam)$ has $\nu_\obj(r) = c_{\lam,\lam} r^2$ gradient growth for $c_{\lam,\lam} \defeq c_\loss + \lambda c_\reg\indic{ \lam \geq \lam_\pi}$.
\end{assumption}
\setcounter{assumption}{\value{assumption}+1} 
\renewcommand\theassumption{\arabic{assumption}b}
\begin{assumption}[Lipschitz $p$-th derivative]\label{LipschitzpObj}
For some $C_{\loss, p+1}, C_{\pi,p+1} < \infty$, a given $\Lambda \subseteq [0,\infty]$, and $\forall i \in [n]$%
\balignt
\Lip{}(\grad_\beta^p \obj(\Pcv{i},\cdot, \lambda)) 
\leq C_{\loss,p+1} + \lam C_{\reg,p+1},\, \, \forall \lambda \in \Lambda.
\ealignt 
\end{assumption}
\endgroup
\setcounter{assumption}{\thetmp +5}
Unregularized higher-order IJ approximations to CV were considered in \citep{debruyne2008model,pmlr-v32-liua14} and recently analyzed by~\citep{giordano2019higher}.
A result similar to \cref{thm:acv-approximates-cv-ho} could be derived from~\cite[Thm.~1]{giordano2019higher}, which controls the approximation error of an \emph{unregularized} IJ version of $\rhoacvest{i}(\lambda)$,
but that work %
additionally assumes bounded lower-order derivatives. 
More generally, the framework in \cref{App:higher-order-assess} provides assessment results for objectives that satisfy weaker curvature conditions than \cref{curvedobj}.

\subsection{Model Selection}
Often, \CV is used not only to assess a model but also to select a high-quality model for subsequent use. The technique requires training a model with many different values of $\lam$ and selecting the one with the lowest \CV error.
If \ACV is to be used in its stead, we would like to guarantee that the model selected by \ACV has test error comparable to that selected by \CV.  

When \CV and \ACV are uniformly close (as in \cref{thm:acv-approximates-cv}), we know that any minimizer of \CV nearly minimizes \ACV as well, so it suffices to show that all near minimizers of \ACV have comparable test error. 
However, this task is made difficult by the potential multimodality of \ACV and \CV.  As we see in \cref{fig:multimodal_ridge}, even for a benign objective function like the ridge regression objective with quadratic error loss and quadratic penalty, \ACV and \CV can have multiple minimizers. 
Interestingly, %
our next result, proved in \cref{App:StrongCurvACV}, shows that any near minimizers of \ACV must produce estimators that are 
$O(1/\sqrt{n})$ close. %
\begin{figure}[!htbp]
\begin{subfigure}[t]{.23\textwidth}
   \includegraphics[width=\textwidth]{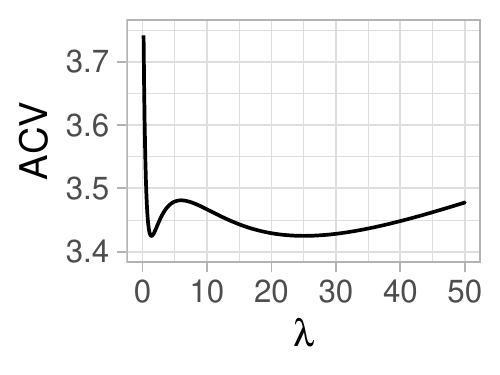}
   \caption{Ridge $\ACV$ \cref{eq:acvest}}%
    \label{fig:multimodal_ridge}
 \end{subfigure}
\begin{subfigure}[t]{.23\textwidth}
        \includegraphics[width=\textwidth]{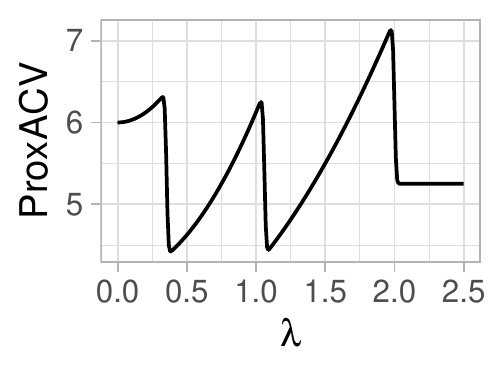}
   \caption{Lasso $\ProxACV$ \cref{eq:proxacvest}}%
   \label{fig:multimodal_lasso}
   \end{subfigure}%
  \caption{\textbf{Multimodality of $\ACV$ and \ProxACV}: (a) $\ACV$ for 
$\loss(z,\beta) = (\beta-z)^\top A (\beta-z), \reg(\beta) = \twonorm{\beta}^2,$
  with $A=\diag(1,40)$, sample mean $\bar{z} = \frac{1}{\sqrt{n}}(1.3893, 1.5)$, and sample covariance $\I_2$.
  (b) \ProxACV for 
$\loss(z, \beta) = \half\twonorm{\beta-z}^2, \reg(\beta)=\onenorm{\beta},$
 with sample mean $\bar{z} = \frac{1}{\sqrt{n}}(\sqrt{1/8}, \sqrt{9/8}, 2 )$ and sample covariance $\I_3$.}
\end{figure}

\begin{theorem}[\ACV-\CV selection error]\label{Thm:WeakCurvACV}
 If \cref{curvedobj,gradlossboundmixed,HessobjLipschitz} hold for  $\Lambda \subseteq [0,\infty]$ and each $(s,r) \in \{(0,3),(1,2),(1,3),(1,4)\}$,
  then $\forall\lambda', \lambda\in\Lambda$ with $\lambda'< \lambda$,
 \balignt
&\twonorm{\est(\lambda) - \est(\lambda')}^2
	\leq  \frac{2}{c_\obj} (\frac{4\gradlossboundmixed{0}{2}}{c_\loss n} + \Delta \ACV +\frac{\gradlossboundmixed{1}{2}}{n^2c_\obj^2} ) 
	\text{ and } \\
\label{eq:weak-risk}
&\twonorm{\est(\lambda_\ACV) - \est(\lambda_\CV)}^2 
	\leq 
	\frac{8}{c_\obj} (\frac{\gradlossboundmixed{0}{2}}{c_\loss n}
	\hspace{-.1cm}+\hspace{-.1cm}  
	\frac{A''c_\obj^2+\gradlossboundmixed{1}{2}}{4c_\obj^2n^2})
 \ealignt
 for 
 $\Delta \ACV \hspace{-.075cm}\defeq\hspace{-.075cm} \ACV(\lambda)\hspace{-.05cm} -\hspace{-.05cm} \ACV(\lambda')$,
$\lambda_\ACV \in \argmin_{\lam \in \Lam} \ACV(\lam)$,  $\lam_\CV \in \argmin_{\lam \in \Lam} \CV(\lam)$, $A'' \defeq 2({\kappa_2}{}\frac{ \gradlossboundmixed{0}{3}}{c_\obj^2}
	+\frac{\kappa_2}{n} \frac{\gradlossboundmixed{1}{3}}{c_\obj^3} 
	+\frac{\kappa_2^2}{n^2} \frac{\gradlossboundmixed{1}{4}}{2c_\obj^4})$, and $\kappa_p$ defined in \cref{thm:acv-approximates-cv}.
\end{theorem}
However, the bound \cref{eq:weak-risk}, established in  \cref{App:WeakCurvACV}, only guarantees an approximation error of the same $O(1/\sqrt{n})$ statistical level of the problem and does not fully exploit the $O(1/n^2)$ accuracy provided by the \ACV estimator \cref{eq:acvest}. 
Fortunately, we obtain a strengthened $O(1/n)$ guarantee if the objective Hessian is Lipschitz and the minimizers of the loss and regularizer are sufficiently distinct (as measured by $\twonorm{\grad \reg(\est(0))}$).
\begin{theorem}[Strong \ACV-\CV selection error]%
\label{Thm:StrongCurvACV} 
 If 
 \cref{gradlossboundmixed,curvedobj,boundedHessobj,HessobjLipschitz} hold %
for some $\Lambda \subseteq [0,\infty]$ with $0\in\Lambda$ and each $(s,r) \in \{(0,3),(1,1),(1,2),(1,3),(1,4)\}$
and $\|\grad \reg(\est(0))\|_2 >0$,
 then for all $\lambda', \lambda\in\Lambda$ with $\lambda'< \lambda$,
\balignt
&|\twonorm{\est(\lambda) - \est(\lambda')}- \frac{1}{n}\frac{A}{c_\obj}|^2
\leq  {\frac{A^2+ 2c_\obj A'}{n^2c_\obj^2} +\frac{2\Delta \ACV}{c_\obj}} \text{ and} \\
&|\twonorm{\est(\lambda_\ACV) - \est(\lambda_\CV)}- \frac{1}{n}\frac{A}{c_\obj}|^2 
\leq {\frac{A^2+ 2c_\obj A' + 2c_\obj A''}{n^2c_\obj^2}}
\ealignt
 for  
 $
 A
 	\hspace{-.075cm} \defeq\hspace{-.075cm}
	 \frac{\gradlossboundmixed{1}{1} +\gradlossboundmixed{0}{2}\kappa_2 }{c_\loss/2} 
	 \hspace{-.015cm}   + \hspace{-.015cm} \frac{\gradlossboundmixed{0}{2}C_{\pi,2}\kappa_1^2 }{ \|\grad \reg(\est(0))\|_2 c_\obj},
$
$A'\hspace{-.075cm} \defeq \hspace{-.075cm}\frac{\gradlossboundmixed{1}{2}}{c_\obj^2}$,
 $\lambda_\ACV \in \argmin_{\lam \in \Lam} \ACV(\lam)$, $\lam_\CV \in \argmin_{\lam \in \Lam} \CV(\lam)$, and  $\kappa_p, \Delta \ACV,$ and $A''$ defined in \cref{thm:acv-approximates-cv,Thm:WeakCurvACV}.
 \end{theorem}
In fact, \cref{Thm:StrongCurvACV}, proved in \cref{App:StrongCurvACV}, implies the bound  \[\textstyle\twonorm{\est(\lambda) - \est(\lambda')} = O(\max(1/n, \sqrt{\Delta \ACV })),\] 
for \emph{any} values of $\lam$ and $\lam'$, even if they are not near-minimizers of \ACV. 
This result relies on the following additional assumption on the Hessian of the objective, which along with the identifiability condition $\|\grad \reg(\est(0))\|_2 >0$ and the curvature of the loss, ensures that two penalty parameters $(\lam,\lam')$ are close whenever their estimators $(\est(\lam), \est(\lam'))$ are close.
\begin{assumption}[Bounded Hessian of objective] 
\label{boundedHessobj} 
For a given $\Lambda \subseteq [0,\infty]$ and some  $C_{\loss,2}, C_{\reg,2} <\infty $
\balignt\label{eq:bdded_hess_obj}
\Hess_\beta \obj(\Pemp, \beta,\lambda) \psdle  (C_{\loss,2} + \lam C_{\loss,2})\mathrm{I}_d, \quad\forall \lambda \in \Lambda, \beta\in\reals^d.
\ealignt
\end{assumption}
\cref{Thm:StrongCurvACV} 
further ensures that the models selected by \CV and \ACV have estimators within $O(1/n)$ of one another.
Importantly, this approximation error is often negligible compared to the typical $\Omega(1/\sqrt{n})$ statistical estimation error of regularized ERM.
%

\subsection{Failure Modes}\label{sec:acv-failure}
One might hope that our \ACV results extend to objectives that do not meet all of our assumptions, such as the Lasso. 
For instance, by leveraging the extended definition of an influence function for non-smooth regularized empirical risk minimizers~\citep{avella2017influence}, we may define a non-smooth extension of $\ACVIJ$ that accommodates objectives with undefined Hessians. 
In the case of squared error loss with an $\ell_1$ penalty, 
$\obj(\Pemp, \beta, \lambda) = 
\frac{1}{2n}\sum_{i=1}^n \twonorm{\beta-z_i}^2 + \lambda \onenorm{\beta},$
this amounts to 
 using
$\acvijest{i}(\lambda) 
    \defeq \est(\lambda) - \frac{1}{n}( 
    \staticindic{\est(\lambda)_j \neq 0} (z_{ij} - \est(\lambda)_j) )_{j=1}^d
$ in the definition \cref{eq:acvij} of $\ACVIJ$.
Analogous Lasso extensions of \ACV and $\ACVIJ$ have been proposed and studied by \citep{ObuchiCV2016,ObuchiCV2018,rad2019scalable,wang2018,Stephenson2019sparse}. 
However, as the following example proved in \cref{app:assesment_failure} demonstrates, these extensions do not satisfy the strong uniform assessment and selection guarantees of the prior sections.

\begin{proposition} \label{prop:assesment_failure}
Suppose $\loss(z,\beta) = \frac{1}{2}(\beta - z)^2$ and $\pi(\beta) = |\beta|$.
Consider a dataset with $n/4$ datapoints taking each of the values in $\{\bar{z} - a,\bar{z} - b, \bar{z} + b, \bar{z} + a\}$ for $\bar{z} = \sqrt{2/n}$ and $a, b > 0$ satisfying $a^2+b^2 =2$ and $a+b =2 \sqrt{2/\pi}$.
Then $\lambda = \bar{z}$ minimizes $\ACVIJ$ and $\ACVIJ(\bar{z}) - \CV(\bar{z}) = \frac{n}{4(n-1)^2}
   ( 1 - \frac{4}{\sqrt{n\pi}} + \tfrac{2}{n} ).$
\end{proposition}
The example in \cref{prop:assesment_failure} was constructed to have the same relevant moments as the normal distribution with variance $1$ and mean $\sqrt{2/n}$.
Notably this $\Omega(1/n)$ assessment error occurs even in the simplest case of $d = 1$;
higher-dimensional counterexamples are obtained straightforwardly by creating copies of this example for each dimension.
The example demonstrates a failure of deterministic uniform assessment for the Lasso extension of $\ACVIJ$,
and similar counterexamples can be constructed for penalties with well-defined (but non-smooth) second derivatives, like the patched Lasso penalty 
$\pi(\beta) = \sum_j \min(|\beta_j|,\tfrac{\delta}{2} +\tfrac{\beta_j^2}{2\delta}),$ for which the standard $\ACVIJ$ is well-defined.
\begin{proposition} \label{prop:patchedlasso_failure}
Suppose $\loss(z,\beta) = \frac{1}{2}(\beta - z)^2$ and $\pi(\beta) = \min(|\beta|,\tfrac{\delta}{2} +\tfrac{\beta^2}{2\delta})$.
Consider a dataset with $n/4$ datapoints taking each of the values in $\{\bar{z} - a,\bar{z} - b, \bar{z} + b, \bar{z} + a\}$, where $\bar{z} = 2 \delta$ and $a, b > 0$ satisfy $a^2+b^2 =1$ and $a+b =2 \sqrt{2/\pi}$.
Then $\ACVIJ(\delta) - \CV(\delta) = \delta \sqrt{2/\pi} \cdot
  \tfrac{1}{n}  
  +o(\tfrac{1}{n})$.
\end{proposition}
The proof of \cref{prop:patchedlasso_failure} is contained in \cref{app:patchedlasso_failure}.
In the following section we propose a modification of $\ACV$ that addresses these problems.%

\section{Proximal ACV}
\label{sec:proximal}
Many objective functions involve non-smooth regularizers that violate the assumptions of the preceding section. Common examples are the $\ell_1$-regularizer $\reg = \onenorm{\cdot}$, often used to engender sparsity for high-dimensional problems, and the 
elastic net~\citep{zou2005regularization}, SLOPE~\citep{bogdan2013statistical}, and nuclear norm~\citep{fazel2001rank} penalties. 
To accommodate non-smooth regularization when approximating CV, several works have proposed either approximating the penalty with a smoothed version~\citep{liu2018fast,wang2018,rad2019scalable} or, for an $\ell_1$ penalty, restricting the approximating CV techniques to the support of the full-data estimator~\citep{Stephenson2019sparse,ObuchiCV2018,ObuchiCV2016} as in \cref{sec:acv-failure}. 
Experimental evidence with the $\ell_1$ penalty suggests these techniques perform well when the support remains consistent across all leave-one-out estimators but can fail otherwise (see~\cite[App.\,D]{Stephenson2019sparse} for an example of failure). 

To address the potential inaccuracy of standard \ACV when coupled with non-smooth regularizers, we recommend use of the proximal operator,
\balignt\label{eq:prox-operator}
{\bf \prox}_{f}^{H}(v)\defeq \argmin_{\beta \in \mathbb{R}^d} \half \| v- \beta\|_{H}^2   %
+ f(\beta),
\ealignt
defined for any positive semidefinite matrix $H$ and function $f$. Specifically, 
we propose the following \emph{proximal approximate CV error} 
\balignt
\ProxACV(\lambda) &= \frac{1}{n}\sum_{i=1}^n \loss (z_i, \proxacvest{i}(\lambda)) \label{eq:ACVprox}
\ealignt
based on the approximate leave-one-out estimators,
\balignt
&\proxacvest{i}(\lambda) 
    	=\prox_{\lambda \reg}^{\acvhess{\loss,i}}( \est(\lambda) - \acvhess{\loss,i}^{-1}g_{\loss,i})\label{eq:proxacvest}\\
    	&
	\defeq \argmin_{\beta\in\reals^d}\half \|\est(\lambda) - \beta\|_{\acvhess{\loss,i}}^2 + \beta ^\top g_{\loss,i}+ \lambda \pi(\beta)
\ealignt
with $\acvhess{\loss, i} \hspace{-.04cm}=\hspace{-.04cm} \Hess_\beta \loss(\Pcv{i}, \est(\lambda)\hspace{-.03cm})$ and $g_{\loss, i}\hspace{-.04cm} = \hspace{-.04cm}\nabla_\beta \loss(\Pcv{i}, \est(\lambda)\hspace{-.03cm})$. %
This estimator optimizes a second-order Taylor expansion of the loss about $\est(\lambda)$ plus the exact regularizer. 
For many standard objectives, the estimator \cref{eq:proxacvest} can be computed significantly more quickly than the exact leave-one-out estimator.
Indeed, state-of-the-art solvers like glmnet~\citep{glmnet} for $\ell_1$-penalized generalized linear models and QUIC~\citep{Hsieh14} for sparse covariance matrix estimation use a sequence of proximal Newton steps like \cref{eq:proxacvest} to optimize their non-smooth objectives.  Using \ProxACV instead entails running these methods for only a single step instead of running them to convergence.
In \cref{sec:experiments}, we give an example of the speed-ups obtainable with this approach.

\subsection{Model Assessment}
A chief advantage of \ProxACV is that it is $O(1/n^2)$ close to \CV uniformly in $\lambda$ even when the regularizer $\reg$ lacks the smoothness or curvature previously assumed in \cref{curvedobj,HessobjLipschitz}:
\begin{theorem}[\ProxACV-\CV assessment error]\label{Thm:ProxACVAssessment}
If 
 \cref{gradlossboundmixed,,proxcurvedobj,HesslossLipschitz} hold for $ \Lambda\subseteq [0,\infty]$ with $0 \in \Lam$ and each %
$(s,r) \in\{ (0,3),(1,3),(1,4)\}$, then, $\forall\lam \in \Lambda$, 
\balignt\label{eq:acv-cv-approx-prox}
|\ProxACV(\lambda)\hspace{-.075cm} -\hspace{-.075cm} \CV(\lambda)| \hspace{-.075cm}\leq \hspace{-.075cm}
   \frac{C_{\loss,{3}}}{n^2} \big(\frac{\gradlossboundmixed{0}{3}}{2c_\obj^3}\hspace{-.075cm} +\hspace{-.075cm} \frac{\gradlossboundmixed{1}{3}}{2nc_\obj^{4}}\hspace{-.075cm} +\hspace{-.075cm} \frac{C_{\loss,{3}}\gradlossboundmixed{1}{4}}{8n^2c_\obj^6}\hspace{-.025cm}\big).
\ealignt 
\end{theorem}
This result, proved in \cref{app:proofProxACVAssessment}, relies on the following modifications of \cref{curvedobj,HessobjLipschitz}: 
\begingroup
\setcounter{tmp}{\value{assumption}}
\setcounter{assumption}{\value{assumption}-4} 
\renewcommand\theassumption{\arabic{assumption}c}
 \begin{assumption}[Curvature of objective] 
 \label{proxcurvedobj} 
For $c_\obj>0$, 
 all $ i \in [n]$, and all $\lambda$ in a given $\Lambda \subseteq [0,\infty]$, 
 $\obj(\Pcv{i}, \cdot,\lambda)$ has $\nu_\obj(r) = c_\obj r^2$ gradient growth, and $\reg$ is convex.
\end{assumption}
\setcounter{assumption}{\value{assumption}+1} 
\renewcommand\theassumption{\arabic{assumption}c}
\begin{assumption}[Lipschitz Hessian of loss] \label{HesslossLipschitz}
For all $ i \in [n]$,
$
\Lip{}(\Hess_\beta \loss(\Pcv{i},\cdot)) \leq C_{\loss,3} < \infty$.
 \end{assumption}
 \endgroup
 \setcounter{assumption}{\thetmp}

 Hence, \ProxACV provides a faithful estimate of \CV for the non-smooth Lasso, elastic net, SLOPE, and nuclear norm penalties whenever a strongly convex loss with Lipschitz Hessian is used.

\subsubsection{Infinitesimal Jackknife}
We also propose the following approximation to \CV, 
\balignt
\ProxACV^{\text{IJ}}(\lambda) &= \frac{1}{n}\sum_{i=1}^n \loss (z_i, \proxacvijest{i}(\lambda))\label{eq:ACVproxij},
\ealignt
based on the infinitesimal jackknife-based estimators
\balignt\label{eq:proxacvij}
\proxacvijest{i}(\lambda) &\defeq {\bf \prox}_{\lambda \reg}^{ \acvhess{\loss}}( \est(\lambda) - \acvhess{\loss}^{-1}g_{\loss,i})
\ealignt
with $\acvhess{\loss} = \Hess_\beta \loss(\Pemp, \est(\lambda))$.
This approximation is sometimes computationally cheaper than \cref{eq:ACVprox} as the same Hessian is used for every estimator.
The following result, proved in \cref{app:proxacv-close-proxacvij}, shows that \ProxACV and $\ProxACVIJ$ are close under our usual assumptions.

\begin{theorem}[$\ProxACVIJ$-\ProxACV assessment error]\label{thm:proxacv-close-proxacvij}
If \cref{proxcurvedobj,gradlossboundmixed} hold for  $\Lambda \subseteq [0,\infty]$ and each  %
$(s,r) \in \{(1,2), (2,2),(3,2)\}$, then for each $\lambda \in \Lambda$, 
\balignt
&|\ProxACV(\lambda) - \ProxACV^{\text{\em IJ}}(\lambda)| \\
&\leq 
\frac{1}{n^2c_\obj^2} \gradlossboundmixed{1}{2}+  \frac{1}{2n^4c_\obj^4}\gradlossboundmixed{3}{2} + \frac{1}{n^3c_\obj^3} \gradlossboundmixed{2}{2}.
\ealignt
\end{theorem}

\cref{thm:proxacv-close-proxacvij,Thm:ProxACVAssessment} imply that $|\ProxACV^{\text{IJ}}(\lambda) - \CV(\lambda)| = O(1/n^2)$ for any $\lam \in \Lam$, and subsequently, all assessment and selection guarantees for $\ProxACV$ in this paper also extend to $\ProxACVIJ$. %
\subsection{Model Selection}
The following theorem, proved in \cref{app:ProxACVSelection}, establishes a model selection guarantee for \ProxACV.  
 \begin{theorem}[\ProxACV-\CV selection error]\label{Thm:ProxACVSelection}
If \cref{gradlossboundmixed,,proxcurvedobj,HesslossLipschitz} hold for $\Lambda \subseteq [0,\infty]$ and each $(s,r) = \{(0,3),(1,2),(1,3), (1,4)\}$, then $\forall \lambda' < \lambda \in \Lambda$,
\balignt
&\twonorm{\est(\lambda) - \est(\lambda')}^2
    \leq \frac{2}{n c_\obj}( \frac{4\gradlossboundmixed{0}{2} }{c_\obj } 
+ \frac{\gradlossboundmixed{1}{2}}{nc_\obj^2} + \Delta\ProxACV) \\
&\text{and }
\label{eq:weak-risk-prox}
 \twonorm{\est(\lambda_{\bf PACV})  -  \est(\lambda_{\CV})}^2 
 \leq \frac{2}{n c_\obj}( \frac{4\gradlossboundmixed{0}{2} }{c_\obj } 
+ \frac{\gradlossboundmixed{1}{2}}{nc_\obj^2} + \tilde{A})
 \ealignt
for $\lambda_{\CV} \in$ $\text{\em argmin}_{\lambda \in \Lambda} \CV(\lambda)$, $\lambda_{\bf PACV} \in \text{\em argmin}_{\lambda \in \Lambda}$ $\ProxACV(\lambda)$ and $\tilde{A}\hspace{-.05cm} \defeq \hspace{-.075cm}
   \frac{C_{\loss,{3}}}{n^2} \big(\frac{\gradlossboundmixed{0}{3}}{c_\obj^3}\hspace{-.075cm} +\hspace{-.075cm} \frac{\gradlossboundmixed{1}{3}}{nc_\obj^{4}}\hspace{-.075cm} +\hspace{-.075cm} \frac{C_{\loss,{3}}\gradlossboundmixed{1}{4}}{4n^2c_\obj^6}\hspace{-.025cm}\big)$.
\end{theorem}
Notably and unlike the \ACV selection results of \cref{Thm:WeakCurvACV,Thm:StrongCurvACV}, \cref{Thm:ProxACVSelection} demands no curvature or smoothness from the regularizer.
Moreover, for $\ell_1$-penalized problems, the $O(1/\sqrt{n})$ error bound is tight as the following example illustrates.

\begin{proposition} \label{prop:lack-curvature}
Suppose $\loss(z,\beta) = \frac{1}{2}(\beta - z)^2$ and $\pi(\beta) = |\beta|$.
Consider a dataset evenly split between the values $a=\sqrt{2}$ and $b=2\sqrt{2/n} - \sqrt{2}$ for $n\geq 4$.  Then $\bar{z} = \frac{1}{n}\sum_{i=1}^n z_i = \sqrt{2/n}$, and $\ProxACV(0) - \ProxACV(\bar{z}) = \frac{5}{2n^2}$, but $\est(0) - \est(\bar{z}) = \bar{z} - 0 = \sqrt{2/n}$.
\end{proposition}
The proof of this proposition is contained in \cref{app:lack-curvature}. 
At the heart of this counterexample is multimodality, 
which can occur for $\ell_1$ penalized objectives (see \cref{fig:multimodal_lasso}), much as it did for the ridge example of \cref{fig:multimodal_ridge}. 
In particular, 
for $\ell_1$ regularized objectives,
the modes of $\ACV$ and \ProxACV 
can be $\Omega(1/\sqrt{n})$ apart.
While this example prevents us from obtaining an $O(1/n)$  deterministic model selection bound for \ProxACV in the worst case, it is possible that \cref{Thm:ProxACVSelection} can be generically strengthened (as in \cref{Thm:StrongCurvACV}) when the the minimizers of the loss and regularizer are sufficiently separated. 
In addition, the possibility of a strong \emph{probabilistic} model selection bound is not precluded. 
\section{Experiments}
\label{sec:experiments}

We present two sets of experiments to illustrate the value of the newly proposed \ProxACV procedure. The first compares the assessment quality of $\ProxACV$ and prior non-smooth \ACV proposals. The second compares the speed of $\ProxACV$ to exact $\CV$. See \url{https://github.com/aswilson07/ApproximateCV} for code reproducing all experiments.
\subsection{\ProxACV versus \ACV and $\ACVIJ$} \label{sec:logistic}
To compare \ProxACV with prior non-smooth extensions of \ACV and $\ACVIJ$,
we adopt the code and the $\ell_1$-regularized logistic regression experimental setup of \citet[App.\,F]{Stephenson2019sparse}. We use the $\beta\in\reals^{151}$ setting, changing only the number of datapoints to $n = 150$ and non-zero weights to $75$ (see \cref{app:exp_details} for more details).
For two ranges of $\lambda$ values, we compare exact \CV with the approximations provided by \ProxACV \cref{eq:proxacvest} and the prior non-smooth extensions of \ACV and $\ACVIJ$ discussed in \cref{sec:acv-failure} and detailed in \cref{app:exp_details}.
 \cref{fig:l1_logistic} (top) shows that for sufficiently large $\lambda$ all three approximations closely match $\CV$.  However, as noted in \citep[App.\,F]{Stephenson2019sparse}, the non-smooth extension of $\ACVIJ$ provides an extremely poor approximation leading to grossly incorrect model selection as $\lambda$ decreases.
 Moreover, the approximation provided by the non-smooth extension of \ACV also deteriorates as $\lambda$ decreases; this is especially evident in the small $\lambda$ range of \cref{fig:l1_logistic} (bottom), where the relative error of the \ACV approximation exceeds 100\%.  Meanwhile, \ProxACV provides a significantly more faithful approximation of \CV across the range of large and small $\lambda$ values. 

 \begin{figure}[!htbp]\centering
        {\includegraphics[width=.45\textwidth]{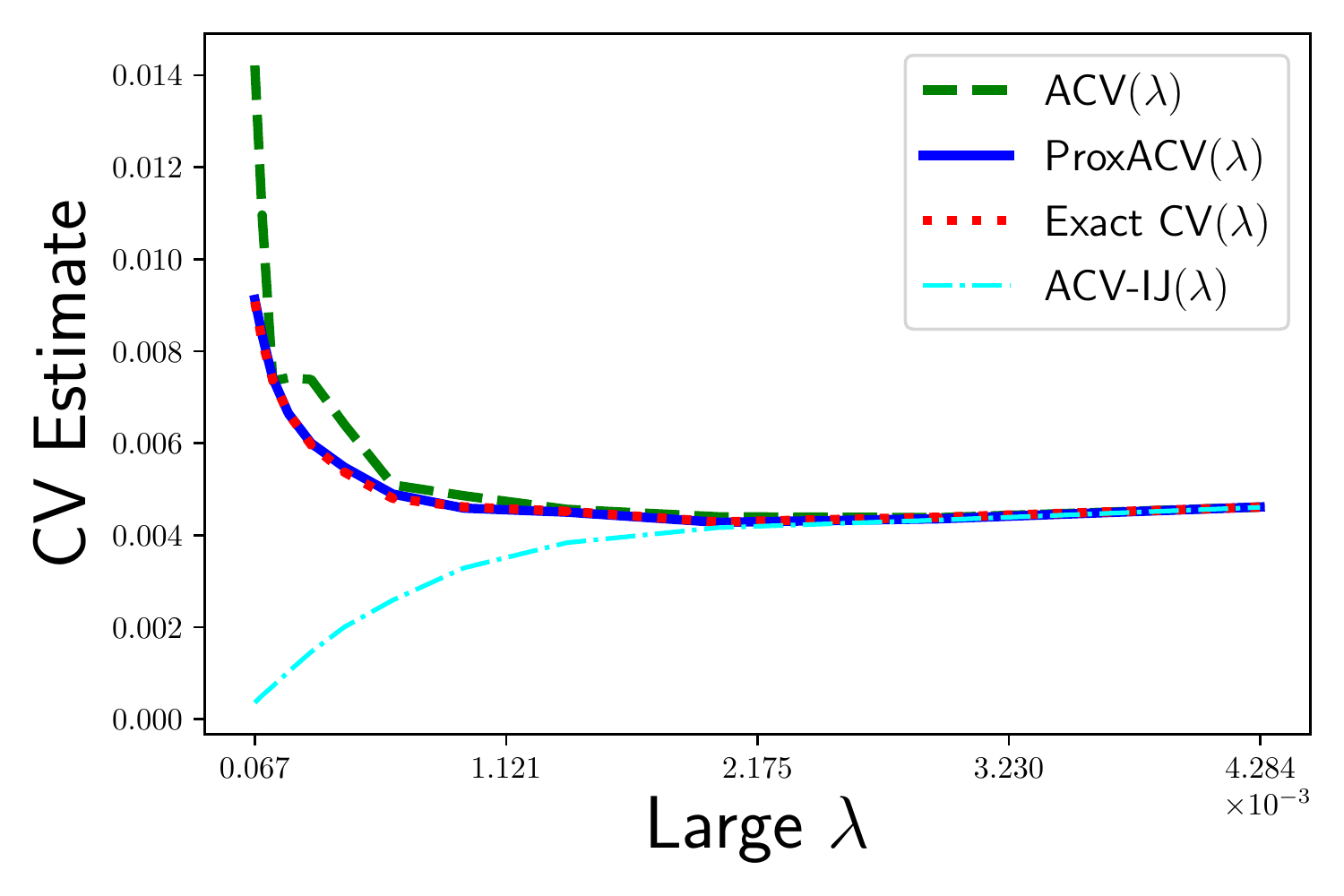}}
        {\includegraphics[width=.45\textwidth]{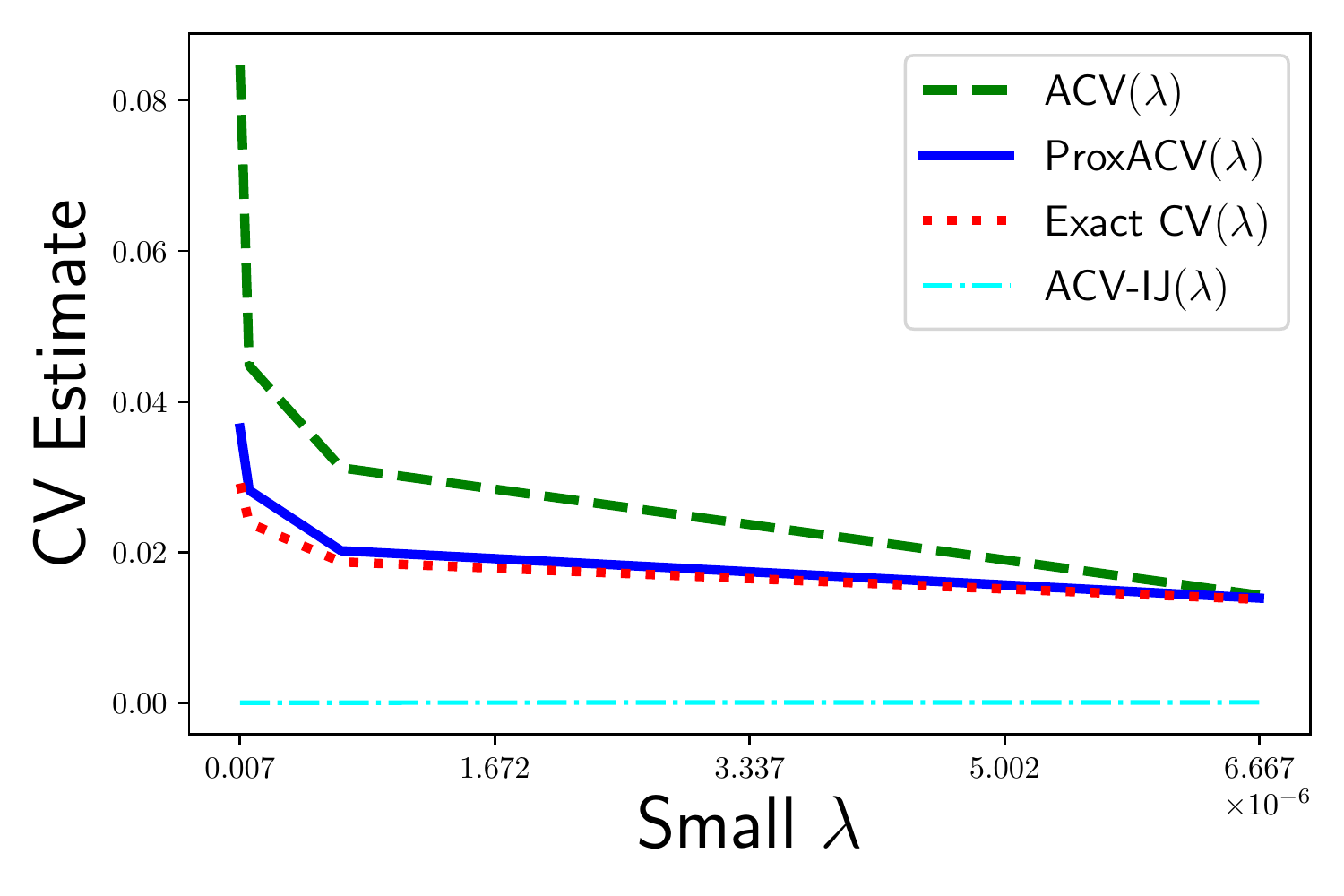}}
  \caption{{\bf \ProxACV vs. \ACV and $\ACVIJ$:} %
  Fidelity of non-smooth \CV approximations in the $\ell_1$-regularized logistic regression setup of \cref{sec:logistic}.
  }
  \label{fig:l1_logistic}
  \end{figure}

  \begin{figure}[!htbp]\centering
   \scalebox{.9}{
   \includegraphics[width=.25\textwidth]{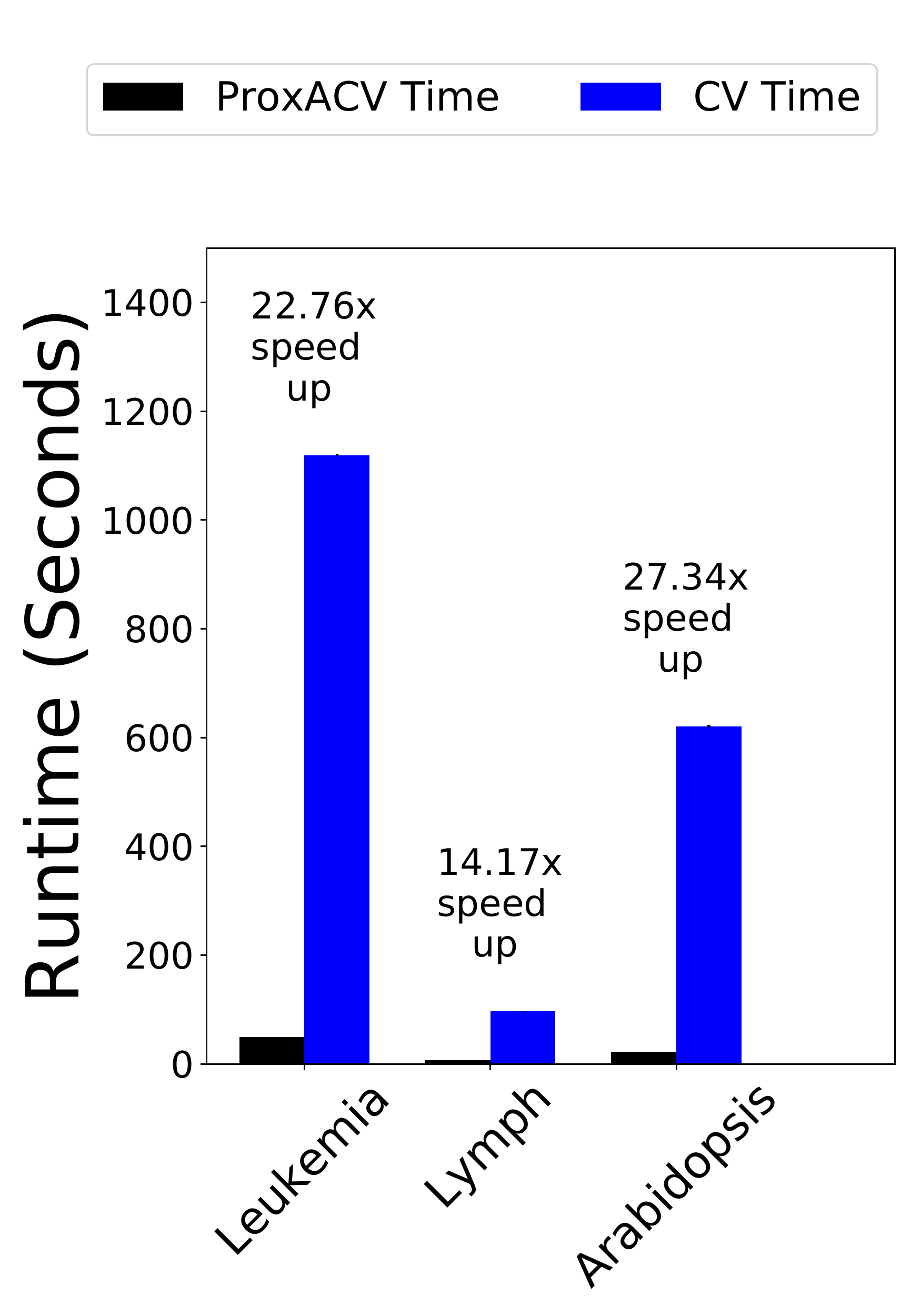}
   }
        \scalebox{.9}{
        \includegraphics[width=.25\textwidth]{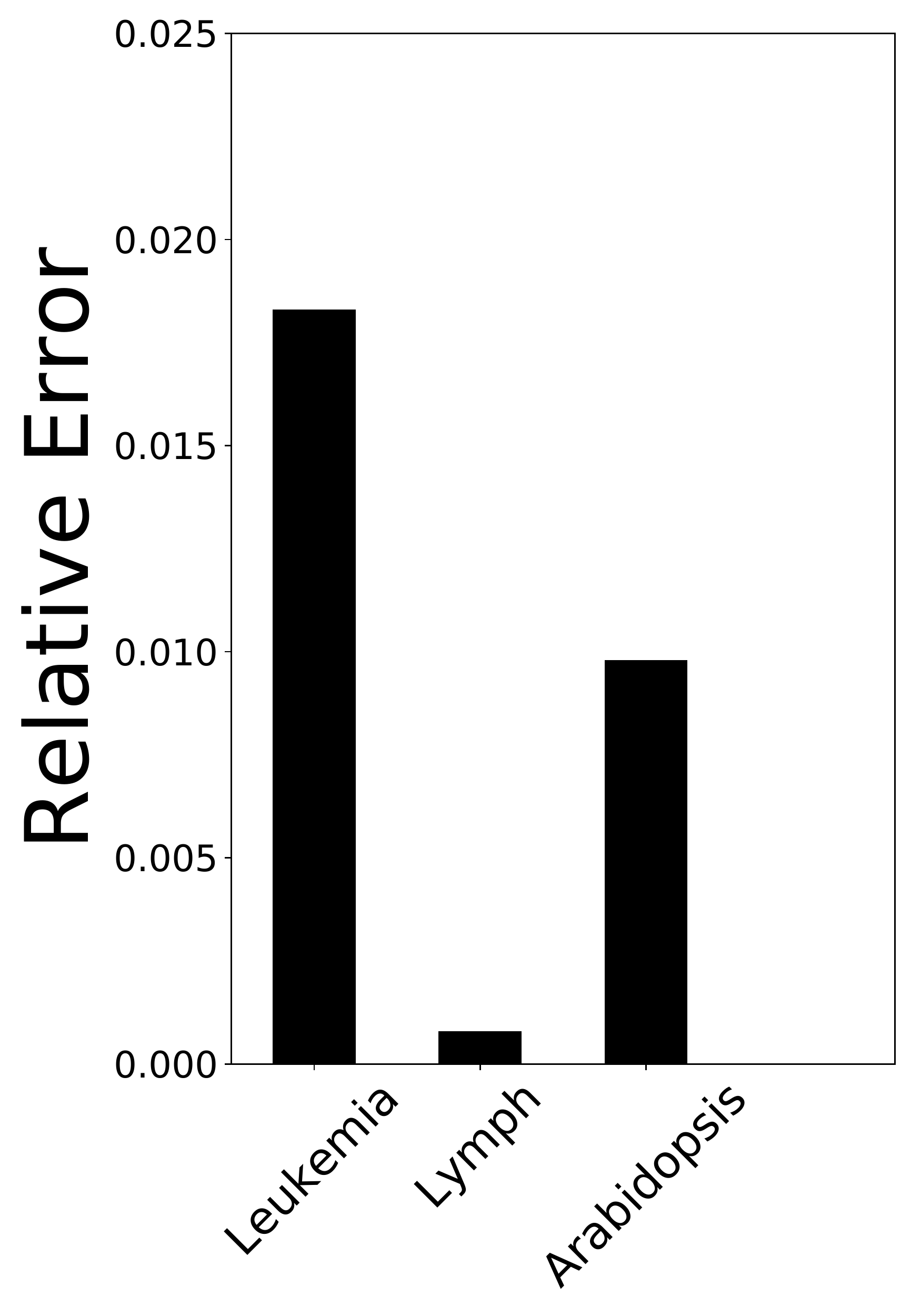}
        }
\caption{{\bf \ProxACV vs. \CV:} Speed-up and relative error of \ProxACV over exact \CV on three biological datasets using QUIC sparse inverse covariance estimation (see \cref{sec:speedup}). 
\ProxACV provides a faithful estimate of \CV with significant computational gains.} %
\label{fig:QUIC_plots}
  \end{figure}

\subsection{ProxACV Speed-up}\label{sec:speedup}

We next benchmark the speed-up of \ProxACV over \CV on the task of sparse inverse covariance estimation. 
using three biological data sets preprocessed by~\citet{Li2010}: Arabidopsis ($p = 834$, $n =118$), Leukemia ($p =1, 225$, $n = 72$), and Lymph ($p = 587$, $n= 148$). 
We employ the standard graphical Lasso objective for matrices $\beta \in \reals^{p\times p}$ (see \cref{app:exp_details} for details) and
compute our \CV and full-data estimators using the released Matlab implementation of the state-of-the-art graphical Lasso solver, QUIC~\citep{Hsieh14}.
Since QUIC optimizes $\obj(\Pcv{i}, \beta, \lambda)$ using a proximal Newton algorithm, we compute our proximal ACV estimators by running QUIC for a single proximal Newton step instead of running it to convergence.  We follow the exact experimental setup of \cite[Fig. 2]{Hsieh14} which employs a penalty of $\lam = 0.5$ for all datasets. 
The timing for each leave-one-out iteration of \CV and \ProxACV was computed using a single core on a 2.10 GHz Intel Xeon E5-4650 CPU.    
In \cref{fig:QUIC_plots}, we display the average relative error, $1  - \ProxACV(\lambda)/\CV(\lambda)$, and running time ($\pm$ 1 standard deviation) over 10 independent runs.
  We see that \ProxACV delivers  14 - 27-fold average speed-ups over \CV with relative errors below 0.02 in each case.
 \paragraph{Importance of curvature}
\cref{Thm:ProxACVAssessment} relies on the curvature $c_\obj$ of the objective, and, in general, such a curvature assumption is necessary for \ProxACV to provide a faithful approximation.
The graphical Lasso objective is strictly but not strongly convex, but the default $\lam$ choice of \citep{Hsieh14} effectively limits the domain of $\obj$ to a compact set with a sizable curvature. 
 However, as $\lam$ decreases, the effective domain of $\obj$ grows, and the curvature decays leading to a worse approximation. 
 For example, when $\lambda = 0.25$ on the Arabidopsis dataset, we obtain a 97.43-fold average speed-up but with $0.137$ mean relative error.

\subsubsection*{Acknowledgments}
We thank Kim-Chuan Toh, Matyas Sustik, and Cho-Jui Hsieh  for sharing their covariance estimation data and Will Stephenson for sharing his approximate cross-validation code.
We also thank the anonymous reviewers for their role in improving this manuscript.
Special thanks to Gary Chamberlain who inspired this project -- Rest in peace.
\bibliography{refs}
\bibliographystyle{abbrvnat}

\appendix \onecolumn

\section{Proof of \cref{optimizer-comparison}: Optimizer comparison}\label{App:opt-comp}
The first claim \cref{eq:opt_comp_error_bound} follows immediately from the definition of the error bound \cref{eq:error_bound}.

To establish the second claim, we note that our (sub)differentiability assumptions and the optimality of $x_{\varphi_1}$ and $x_{\varphi_2}$ imply that 
$0 \in \partial \varphi_2(x_{\varphi_2})$
and $0 = u + \grad (\varphi_1-\varphi_2)(x_{\varphi_1})$
for some $u \in \partial\varphi_2(x_{\varphi_1})$. Gradient growth \cref{eq:gradient_growth} now implies
\balignt
\nu_{\varphi_2}(\twonorm{x_{\varphi_1} - x_{\varphi_2} })
    \leq 
    \inner{x_{\varphi_1} - x_{\varphi_2}}{u - 0}
    = 
    \inner{x_{\varphi_1} - x_{\varphi_2}}{\grad (\varphi_2- \varphi_1)(x_{\varphi_1})}.
\ealignt
\section{Proof of \cref{thm:acv-approximates-cv,thm:acv-approximates-cv-ho}:
$\ACV$-\CV and $\ACV_p$-\CV assessment error} \label{App:higher-order-assess}
\cref{thm:acv-approximates-cv,thm:acv-approximates-cv-ho} will follow from the following more detailed statement, proved in \cref{sec:proof-acv-approximates-cv-HO-full}.
Consider the higher-order gradient estimator 
\balignt
\HOACV(\lambda) &\defeq \frac{1}{n} \sum_{i=1}^n \loss(z_i, \hoacvest{i}(\lambda)) \qtext{with}\hoacvest{i}(\lambda) \defeq \argmin_\beta \widehat{\obj}_p(\Pcv{i},\beta, \lambda;\est(\lambda)),
\ealignt
which recovers our approximate CV error~\eqref{eq:acv1} and estimate~\eqref{eq:acvest} when $p=2$. 
We will make use of the following assumptions which generalize \cref{curvedobj,hocurvedobj}.
\begingroup
\setcounter{tmp}{\value{assumption}}
\setcounter{assumption}{\value{assumption}-4} 
\renewcommand\theassumption{\arabic{assumption}d}
\begin{assumption}[Curvature of objective] 
\label{hocurvedobjfull} 
For some $q,  c_\obj> 0$, all $i \in [n]$, and all $\lam$ in a given $\Lambda \subseteq [0,\infty]$,
$\obj(\Pcv{i},\cdot, \lam)$ has $\nu_\obj(r) = c_\obj r^q$ gradient growth
\end{assumption}
\setcounter{assumption}{\value{assumption}-1} 
\renewcommand\theassumption{\arabic{assumption}e}
\begin{assumption}[Curvature of Taylor approximation] 
\label{hocurvedtaylorfull} 
For some $p, q,c_\loss,c_\reg > 0$ and $\lam_\reg <\infty$, all $i \in [n]$, and all $\lam$ in a given $\Lambda \subseteq [0,\infty]$,
 $\widehat{\obj}_p(\Pcv{i},\cdot,\lambda;\est(\lambda))$   has $\nu(r) = c_{\lam,\lam} r^q$ gradient growth, where $c_{\lam,\lam}\defeq c_\loss + \lam c_\reg\indic{\lam \geq \lam_\reg}$.
\end{assumption}
\setcounter{assumption}{\value{assumption}-1} 
\renewcommand\theassumption{\arabic{assumption}f}
\begin{assumption}[Curvature of regularized Taylor approximation] 
\label{rhocurvedtaylorfull} 
For some $p, q,c_\loss,c_\reg > 0$ and $\lam_\reg <\infty$, all $i \in [n]$, and all $\lam$ in a given $\Lambda \subseteq [0,\infty]$,
 $\widehat{\obj}_p(\Pcv{i},\cdot,\lambda;\est(\lambda))+ \frac{\Lip{}(\grad_\beta^p \obj(\Pcv{i},\cdot, \lambda))}{p+1}\twonorm{\cdot - \est(\lambda)}^{p+1}$   has $\nu(r) = c_{\lam,\lam} r^q$ gradient growth, where $c_{\lam,\lam}\defeq c_\loss + \lam c_\reg\indic{\lam \geq \lam_\reg}$. %
\end{assumption}
\endgroup
\setcounter{assumption}{\thetmp} %

\begin{theorem}[$\ACV_p$-\CV and $\HOACV$-\CV assessment error]\label{thm:acv-approximates-cv-HO-full}
If \cref{hocurvedobjfull} holds for some $\Lambda\subseteq[0,\infty]$, then, for all $\lambda \in \Lambda$ and $i \in [n]$,
\balignt \label{eq:ho_curve_bound_1}
\twonorm{\est(\lambda) - \cvest{i}(\lambda)}^{q-1} \leq \frac{1}{n} \frac{1}{c_\obj}\twonorm{\grad_\beta \loss(z_i, \est(\lambda))}.
\ealignt

If \cref{hocurvedobjfull,hocurvedtaylorfull,LipschitzpObj} hold for some $\Lambda\subseteq[0,\infty]$, then, for all $\lam \in\Lam$ and $i\in[n]$, 
 \begin{subequations}
 \balignt
\twonorm{\hoacvest{i}(\lam) - \cvest{i}(\lambda)}^{q-1}
    &\leq 
       \kappa_{p,\lambda}^\lambda \twonorm{\cvest{i}(\lambda)- \est(\lambda) }^p
          \label{eq:acv-cv-est-ho-bound1}
\ealignt
for $\kappa_{p,\lambda}^\lambda \defeq  \frac{C_{\loss,{p+1}} + \lambda C_{\pi,{p+1}}}{p!(c_\loss + \lam c_\reg \indic{\lam \geq \lambda_\reg})}$.

If \cref{hocurvedobjfull,rhocurvedtaylorfull,LipschitzpObj} hold for some $\Lambda\subseteq[0,\infty]$, then, for all $\lam \in\Lam$ and $i\in[n]$, 
\balignt
\twonorm{\rhoacvest{i}(\lam) - \cvest{i}(\lambda)}^{q-1}
    &\leq 
       2\kappa_{p,\lambda}^\lambda \twonorm{\cvest{i}(\lambda)- \est(\lambda) }^p .
          \label{eq:acv-cv-est-rho-bound1}
        \ealignt
        \end{subequations}
 If~\cref{hocurvedobjfull,hocurvedtaylorfull,gradlossboundmixed,LipschitzpObj} hold for some $\Lambda\subseteq[0,\infty]$ and each $(s,r) \in \{(0,\frac{p+(q-1)^2}{(q-1)^2}), (1,\frac{2p}{(q-1)^2}), (1,\frac{p+q-1}{(q-1)^2})\}$, then, for all $\lam \in \Lam$,
  \begin{subequations}
\balignt
&|\HOACV(\lambda) - \CV(\lambda)| \\
	&\leq  
		\frac{1}{n^{\frac{p}{(q-1)^2}}} \frac{(\kappa_{p,\lambda}^\lambda)^{\frac{1}{q-1}}}{c_\obj^{\frac{p}{(q-1)^2}}} \gradlossboundmixed{0}{\frac{p+(q-1)^2}{(q-1)^2}}  
		+ \half\frac{1}{n^{\frac{2p}{(q-1)^2}}} \frac{(\kappa_{p,\lambda}^\lambda)^{\frac{2}{q-1}}}{c_m^{\frac{2p}{(q-1)^2}}} \gradlossboundmixed{1}{\frac{2p}{(q-1)^2}} 
		+ \frac{1}{n^{\frac{p+q-1}{(q-1)^2}}}\frac{(\kappa_{p,\lambda}^\lambda)^{\frac{1}{q-1}}}{c_\obj^{\frac{p+q-1}{(q-1)^2}}}  \gradlossboundmixed{1}{\frac{p+q-1}{(q-1)^2}}
	\qtext{and} 
	\label{eq:acv-cv-ho-bound2_1} 
	\ealignt
If~\cref{hocurvedobjfull,rhocurvedtaylorfull,gradlossboundmixed,LipschitzpObj} hold for some $\Lambda\subseteq[0,\infty]$ and each $(s,r) \in \{(0,\frac{p+(q-1)^2}{(q-1)^2}), (1,\frac{2p}{(q-1)^2}), (1,\frac{p+q-1}{(q-1)^2})\}$, then, for all $\lam \in \Lam$,
	\balignt
&|\ACV_p(\lambda) - \CV(\lambda)| \\
	&\leq 
		\frac{1}{n^{\frac{p}{(q-1)^2}}} \frac{(2\kappa_{p,\lambda}^\lambda)^{\frac{1}{q-1}}}{c_\obj^{\frac{p}{(q-1)^2}}} \gradlossboundmixed{0}{\frac{p+(q-1)^2}{(q-1)^2}}  
		+ \half\frac{1}{n^{\frac{2p}{(q-1)^2}}} \frac{(2\kappa_{p,\lambda}^\lambda)^{\frac{2}{q-1}}}{c_m^{\frac{2p}{(q-1)^2}}} \gradlossboundmixed{1}{\frac{2p}{(q-1)^2}} 
		+ \frac{1}{n^{\frac{p+q-1}{(q-1)^2}}}\frac{(2\kappa_{p,\lambda}^\lambda)^{\frac{1}{q-1}}}{c_\obj^{\frac{p+q-1}{(q-1)^2}}}  \gradlossboundmixed{1}{\frac{p+q-1}{(q-1)^2}}.
	\label{eq:acv-cv-ho-bound2}
\ealignt
\end{subequations}
\end{theorem}

\cref{thm:acv-approximates-cv} follows from \cref{thm:acv-approximates-cv-HO-full} with $p = q = 2$ since \cref{curvedobj} implies $\mu = c_\loss + \lam c_\reg \indic{\lam \geq \lam_\reg}$ strong convexity and hence $\nu(r) = \mu r^2$ gradient growth for each $\widehat{\obj}_2(\Pcv{i},\cdot,\lambda;\est(\lambda))$. 

\cref{thm:acv-approximates-cv-ho} follows from \cref{thm:acv-approximates-cv-HO-full} with $q = 2$ since \cref{hocurvedobj,LipschitzpObj} and the following lemma imply that each 
$\obj(\Pcv{i},\cdot, \lambda)$ and 
$\widehat{\obj}_p(\Pcv{i},\cdot, \lambda;\est(\lambda)) + \frac{\Lip{}(\grad_\beta^p \obj(\Pcv{i},\cdot, \lambda)) }{p+1}\|\cdot - \est(\lambda)\|_2^{p+1}$ has $\mu = c_\loss + \lam c_\reg \indic{\lam \geq \lam_\reg}$ strong convexity and hence $\nu(r) = \mu r^2$ gradient growth.
\begin{lemma}[Curvature of regularized Taylor approximation]\label{lem:taylor_strong_convex}
If $\varphi$ is $\mu$ strongly convex and $\grad^p\varphi$ is Lipschitz, then 
$\Phi(x) \defeq \widehat{\varphi}_p(x;w)+ \frac{\Lip{}(\grad^p\varphi)}{(p+1)!}\|x - w\|_2^{p+1}$ is $\mu$  strongly convex.
\end{lemma}
\begin{proof}
This result is inspired by \citep[Thm.~1]{Nesterov2019}. In particular, by Taylor's theorem with integral remainder, we can bound the residual between a function and its Taylor approximation as 
\begin{talign}
 |\varphi(x) - \widehat{\varphi}_p(x;w)| \leq \frac{\Lip{}(\grad^p\varphi)}{(p+1)!}\|x - w\|_2^{p+1}
\end{talign}
Note also that for $d(x) = \frac{1}{p}\|x\|^p $
\balignt \label{eq:hessdp}
\nabla^2 d(x) = (p-2)\|x\|^{p-4}xx^\top + \|x\|^{p-2}\mathrm{I}_d \psdge \|x\|^{p-2}\mathrm{I}_d.
\ealignt
For $p\geq2$, applying the same reasoning to $\inner{\nabla f(\cdot)}{h}$ and $\inner{\nabla^2 f(\cdot)h}{h}$ we can similarly conclude:%
\balignt
 \opnorm{\nabla \varphi(x) - \nabla \widehat{\varphi}_p(x;w)}&\leq\frac{\Lip{}(\grad^p\varphi)}{p!}\twonorm{x - w}^{p} \\
  \opnorm{\nabla^2 \varphi(x) - \nabla^2 \widehat{\varphi}_p(x;w)}&\leq \frac{\Lip{}(\grad^p\varphi)}{(p-1)!}\twonorm{x - w}^{p-1} .
\ealignt
Subsequently, for any direction $h \in \mathbb{R}^d$%
\balignt
\inner{(\nabla^2 \varphi(x) - \nabla^2 \widehat{\varphi}_p(x;w)) h}{h} \leq \opnorm{ \nabla^2 \varphi(x) - \nabla^2 \widehat{\varphi}_p(x;w)}\cdot\|h\|_2^2\leq \frac{\Lip{}(\grad^p\varphi)}{(p-1)!}\|x - w\|_2^{p-1}\cdot\|h\|_2^2, 
\ealignt
and therefore, 
\balignt
\nabla^2 \varphi(x)  \psdle \nabla^2 \widehat{\varphi}_p(x;w) + \frac{\Lip{}(\grad^p\varphi)}{(p-1)!}\|x - w\|_2^{p-1}\mathrm{I}_{d} \overset{\eqref{eq:hessdp}}{\psdle} \nabla^2\Phi(x).
\ealignt
\end{proof}
\subsection{Proof of \cref{thm:acv-approximates-cv-HO-full}: $\ACV_p$-\CV and $\HOACV$-\CV assessment error}
\label{sec:proof-acv-approximates-cv-HO-full}
\subsubsection{Proof of \cref{eq:ho_curve_bound_1}: Proximity of CV and full-data estimators}%
We begin with a lemma that translates the polynomial gradient growth of our objective into a bound on the difference between a full-data estimator $\est(\lam)$ and a leave-one-out estimator $\cvest{i}(\lam)$.
\begin{lemma}[Proximity of CV and full-data estimators]\label{lem:cvest-est-comparison}
Fix any $\lambda\in[0,\infty)$ and $i\in[n]$.
If $\loss(z_i, \cdot)$ is differentiable, and $\obj(\Pcv{i},\cdot,\lambda)$ has $\nu_{\obj}(r) = c_\obj r^q$ gradient growth \cref{eq:gradient_growth} for $c_\obj >0$ and $q > 0$, then 
\balignt
\twonorm{\est(\lambda) - \cvest{i}(\lambda)}^{q-1} %
&\leq \frac{1}{n}\frac{1}{c_\obj}\twonorm{\nabla_\beta \loss(z_i, \est(\lambda))}. \label{eq:cvest-est-comparison}
\ealignt
\end{lemma}
\begin{proof}
The result follows from the Optimizer Comparison \cref{optimizer-comparison} with $\varphi_1(\beta) = \obj(\Pemp, \beta,\lambda)$ and $\varphi_2(\beta) = \obj(\Pcv{i},\beta, \lambda)$
and Cauchy-Schwarz, as
\balignt
c_m\twonorm{\est(\lambda) - \cvest{i}(\lambda)}^{q} 
&\leq \inner{\est(\lambda) - \cvest{i}(\lambda)}{\grad_\beta \obj(\Pcv{i}, \est(\lambda), \lambda)-\nabla_\beta \obj(\Pemp, \est(\lambda), \lambda)} \\
&=\frac{1}{n}\inner{\cvest{i}(\lambda)-\est(\lambda) }{\nabla_\beta \loss(z_i, \est(\lambda))}
\leq
\frac{1}{n}\twonorm{\est(\lambda) - \cvest{i}(\lambda)}  \twonorm{\nabla_\beta \loss(z_i, \est(\lambda))}.
\ealignt
\end{proof}
Now fix any $\lambda\in\Lambda$ and $i\in[n]$.
If $\lambda = \infty$, then $\est(\lambda)=\cvest{i}(\lambda)$, ensuring the result \cref{eq:ho_curve_bound_1}.
If $\lambda \neq \infty$, 
then our assumptions and \cref{lem:cvest-est-comparison} immediately establish the result \cref{eq:ho_curve_bound_1}. %
\subsubsection{Proof of \cref{eq:acv-cv-est-ho-bound1}: Proximity of  $\HOACV$ and CV estimators} \label{proof-acv-cv-est-ho-bound1_1}
The result \cref{eq:acv-cv-est-ho-bound1} will follow from a general Taylor comparison lemma that bounds the optimizer error introduced by approximating part of an objective with its Taylor polynomial.
\begin{lemma}[Taylor comparison]\label{taylor-comparison}
Suppose
\begin{align}\label{eq:obj} 
x_{\varphi} \in \argmin_x \varphi(x) + \varphi_0(x) \qtext{and} x_{\widehat{\varphi}_p} \in \argmin_x \widehat{\varphi}_p(x;w) + \varphi_0(x).
\end{align}
for $\widehat{\varphi}_p(x;w) \defeq \sum_{i=0}^{p} \frac{1}{i!}\nabla^i \varphi(w)[x-w]^{\otimes i}$
the $p$-th-order Taylor polynomial of $\varphi$ about a point $w$.
If $\grad^p \varphi$ is Lipschitz 
and $\widehat{\varphi}_p(\cdot;w)+\varphi_0$ has $\nu(r) = \mu r^q$ gradient growth \cref{eq:gradient_growth} for $\mu>0$ and $q > 0$, then 
\balignt\label{eq:taylor-comparison-bound}
\twonorm{x_{\varphi} - x_{\widehat{\varphi}_p}}^{q-1}
    \leq \frac{\Lip{}(\grad^p\varphi)}{\mu}\frac{1}{p!}\twonorm{x_{\varphi} - w}^p.
\ealignt
\end{lemma}
\begin{proof}
Define $f(x) = \inner{x_{\widehat{\varphi}_p}-x_{\varphi}}{\grad \varphi(x)}$.
The result follows from the Optimizer Comparison \cref{optimizer-comparison} with $\varphi_1 =  \varphi  +  \varphi_0$ and $\varphi_2 = \widehat{\varphi}_p(\cdot;w) + \varphi_0$, Taylor's theorem with integral remainder, and Cauchy-Schwarz as
\balignt
\mu\twonorm{x_{\varphi} - x_{\widehat{\varphi}_p}}^{q} 
    &\leq \inner{x_{\varphi} - x_{\widehat{\varphi}_p}}{\grad_x \widehat{\varphi}_p(x_{\varphi};w) - \grad \varphi(x_{\varphi})}
    = f(x_{\varphi}) - \sum_{i=0}^{p-1} \frac{1}{i!} \grad^{i}f(w)[x_{\varphi}-w]^{\otimes i}
    \\
    &\leq \frac{\Lip{}(\grad^{p-1}f)}{p!}\twonorm{x_{\varphi} - w}^p
    \leq \twonorm{x_{\varphi} - x_{\widehat{\varphi}_p}} \frac{\Lip{}(\grad^p\varphi)}{p!}\twonorm{x_{\varphi} - w}^p.
\ealignt
\end{proof}
To see this, fix any $\lambda\in\Lambda$ and $i\in[n]$, and consider the choices $\varphi = \obj(\Pcv{i},\cdot, \lambda)$, $\varphi_0 \equiv 0$, and $w = \est(\lambda)$.
By \cref{hocurvedtaylorfull}, $\widehat{\varphi}_p(\cdot;w)+\varphi_0$ has $\nu(r) = \mu r^q$ gradient growth for $\mu = c_\loss + \lam c_\reg \indic{\lam \geq \lambda_\reg}$.
Since $\Lip{}(\grad^p\varphi)\leq C_{\loss,p+1} + \lambda C_{\pi,p+1}$ by \cref{LipschitzpObj}, the desired result \cref{eq:acv-cv-est-ho-bound1} follows from \cref{taylor-comparison}.

\subsubsection{Proof of \cref{eq:acv-cv-est-rho-bound1}: Proximity of $\ACV_p$ and \CV estimators}\label{proof-acv-cv-est-ho-bound1}
The result \cref{eq:acv-cv-est-rho-bound1} will follow from a regularized Taylor comparison lemma that bounds the optimizer error introduced by approximating part of an objective with a regularized Taylor polynomial.
\begin{lemma}[Regularized Taylor comparison]\label{taylor-comparison-1}
Suppose
\begin{talign}\label{eq:obj} 
x_{\varphi} \in \argmin_x \varphi(x) + \varphi_0(x) \qtext{and} x_{\widehat{\varphi}_p} \in \argmin_x \widehat{\varphi}_p(x;w) + \frac{\Lip{}(\grad^p\varphi)}{(p+1)!}\|x - w\|_2^{p+1}+ \varphi_0(x).
\end{talign}
for $\widehat{\varphi}_p(x;w) \defeq \sum_{i=0}^{p} \frac{1}{i!}\nabla^i \varphi(w)[x-w]^{\otimes i}$
the $p$-th-order Taylor polynomial of $\varphi$ about a point $w$.
If $\grad^p \varphi$ is Lipschitz 
and $\widehat{\varphi}_p(\cdot;w)+\frac{\Lip{}(\grad^p\varphi)}{(p+1)!}\|\cdot - w\|_2^{p+1}+\varphi_0$ has $\nu(r) = \mu r^q$ gradient growth \cref{eq:gradient_growth} for $\mu>0$ and $q > 0$, then 
\balignt\label{eq:taylor-comparison-bound}
\twonorm{x_{\varphi} - x_{\widehat{\varphi}_p}}^{q-1}
    \leq \frac{2\Lip{}(\grad^p\varphi)}{\mu}\frac{1}{p!}\twonorm{x_{\varphi} - w}^p.
\ealignt
\end{lemma}
\begin{proof}
Define $f(x) = \inner{x_{\widehat{\varphi}_p}-x_{\varphi}}{\grad \varphi(x)}$.
The result follows from the Optimizer Comparison \cref{optimizer-comparison} with $\varphi_1 =  \varphi  +  \varphi_0$ and $\varphi_2 = \widehat{\varphi}_p(\cdot;w)+ \frac{\Lip{}(\grad^p\varphi)}{(p+1)!}\|\cdot - w\|_2^{p+1}+ \varphi_0$, Taylor's theorem with integral remainder, and Cauchy-Schwarz as%
\balignt
\mu\twonorm{x_{\varphi} - x_{\widehat{\varphi}_p}}^{q} 
    &\leq \inner{x_{\varphi} - x_{\widehat{\varphi}_p}}{\grad_x \widehat{\varphi}_p(x_{\varphi};w) - \grad \varphi(x_{\varphi})} +\frac{\Lip{}(\grad^p\varphi)}{p!}\inner{(x_{\varphi} - x_{\widehat{\varphi}_p})\|x_\varphi -w\|_2^{p-1}}{x_\varphi - w}\\
    &= f(x_{\varphi}) - \sum_{i=0}^{p-1} \frac{1}{i!} \grad^{i}f(w)[x_{\varphi}-w]^{\otimes i}
    + \frac{\Lip{}(\grad^p\varphi)}{p!}\inner{(x_{\varphi} - x_{\widehat{\varphi}_p})\|x_\varphi -w\|_2^{p-1}}{(x_\varphi - w)}\\
    &\leq \frac{\Lip{}(\grad^{p-1}f)}{p!}\twonorm{x_{\varphi} - w}^p +\twonorm{x_{\varphi} - x_{\widehat{\varphi}_p}} \frac{\Lip{}(\grad^p\varphi)}{p!}\twonorm{x_{\varphi} - w}^p
    \leq \twonorm{x_{\varphi} - x_{\widehat{\varphi}_p}} \frac{2\Lip{}(\grad^p\varphi)}{p!}\twonorm{x_{\varphi} - w}^p.
\ealignt
\end{proof}

Fix any $\lambda\in\Lambda$ and $i\in[n]$, and consider the choices $\varphi = \obj(\Pcv{i},\cdot, \lambda)$, $\varphi_0 \equiv 0$, and $w = \est(\lambda)$.  By \cref{rhocurvedtaylorfull}, $\widehat{\varphi}_p(\cdot;w)+\frac{\Lip{}(\grad^p\varphi)}{(p+1)!}\|\cdot - w\|_2^{p+1}+\varphi_0$ has  $\nu(r) = \mu r^q$ gradient growth for  $\mu = c_\loss + \lam c_\reg \indic{\lam \geq \lambda_\reg}$.
Since $\Lip{}(\grad^p\varphi)\leq C_{\loss,p+1} + \lambda C_{\pi,p+1}$ by \cref{LipschitzpObj}, the desired result \cref{eq:acv-cv-est-rho-bound1} follows from \cref{taylor-comparison-1}.

\subsubsection{Proof of \cref{eq:acv-cv-ho-bound2_1}: Proximity of $\HOACV$ and \CV}
\label{sec:acvp-cv}
Fix any $\lambda\in\Lambda$.
To control the discrepancy between $\HOACV(\lambda)$ and $\CV(\lambda)$, we first rewrite the difference using Taylor's theorem with Lagrange remainder: 
\balignt
  \HOACV(\lambda) - \CV(\lambda)  
  &= \frac{1}{n}\sum_{i=1}^n\loss(z_i,\hoacvest{i}(\lam) ) - \loss(z_i,\cvest{i}(\lambda) ) \\
 &= \frac{1}{n}\sum_{i=1}^n\inner{\grad_\beta \loss(z_i, \cvest{i}(\lambda) )}{ \hoacvest{i}(\lam)-\cvest{i}(\lambda)}
  + \half\Hess_\beta \loss(z_i, \tilde{s}_{i})[\hoacvest{i}(\lam)-\cvest{i}(\lambda)]^{\otimes 2}
 \ealignt
 for some $\tilde{s}_{i}\in \{ t\hoacvest{i}+ (1- t)\cvest{i}(\lambda) :t \in [0,1]\}$.
 We next use the mean-value theorem to expand each function $\inner{\grad_\beta \loss(z_i, \cdot)}{ \hoacvest{i}(\lam)-\cvest{i}(\lambda)}$ around the full-data estimator $\est(\lambda)$:
 \balignt
   \HOACV(\lambda) - \CV(\lambda)  
  &= \frac{1}{n}\sum_{i=1}^n\inner{\grad_\beta \loss(z_i, \est(\lambda) )}{\hoacvest{i}(\lam)-\cvest{i}(\lambda)}
  + \half\Hess_\beta \loss(z_i, \tilde{s}_{i})[\hoacvest{i}(\lam)-\cvest{i}(\lambda)]^{\otimes 2} \notag \\
  &+ \inner{\Hess_\beta \loss(z_i, s_i)(\cvest{i}(\lambda) - \est(\lambda))}{\hoacvest{i}(\lam)-\cvest{i}(\lambda)}
 \ealignt
  for some  $s_{i}\in \{ t\est(\lambda)+ (1- t)\cvest{i}(\lambda) :t \in [0,1]\}$.
Finally, we invoke Cauchy-Schwarz, the definition of the operator norm, the estimator proximity results \cref{eq:ho_curve_bound_1} and \cref{eq:acv-cv-est-ho-bound1}, and \cref{gradlossboundmixed} to obtain
\balignt
 	|\HOACV(\lambda) - \CV(\lambda)| 
\leq 
	\,&\frac{1}{n} \sum_{i=1}^n \twonorm{\grad_\beta \loss(z_i, \est(\lambda))}\twonorm{\hoacvest{i}(\lam) - \cvest{i}(\lambda)}
	+\half\opnorm{\Hess_\beta \loss(z_i, \tilde{s}_{i}) } \twonorm{\hoacvest{i}(\lam)-\cvest{i}(\lambda) }^2\\
	&+\opnorm{\Hess_\beta \loss(z_i, s_{i}) } \twonorm{\est(\lambda)-\cvest{i}(\lambda) }\twonorm{\hoacvest{i}(\lam)-\cvest{i}(\lambda) }\\
\leq 
	\,&\frac{1}{n} \sum_{i=1}^n 
	(\kappa_{p,\lambda}^\lambda)^{\frac{1}{q-1}} \twonorm{\grad_\beta \loss(z_i, \est(\lambda))}\twonorm{\cvest{i}(\lambda) - \est(\lambda)}^{\frac{p}{q-1}} \\
	&+ \half (\kappa_{p,\lambda}^\lambda)^{\frac{2}{q-1}}\opnorm{\Hess_\beta \loss(z_i, \tilde{s}_{i}) } \twonorm{\cvest{i}(\lambda)-\est(\lambda) }^{\frac{2p}{q-1}}\\
	&+(\kappa_{p,\lambda}^\lambda)^{\frac{1}{q-1}}\opnorm{\Hess_\beta \loss(z_i, s_{i}) } \twonorm{\est(\lambda)-\cvest{i}(\lambda) }^{\frac{p+q-1}{q-1}}\\
\leq 
	\,&\frac{1}{n^{\frac{p}{(q-1)^2}}} \frac{(\kappa_{p,\lambda}^\lambda)^{\frac{1}{q-1}}}{c_\obj^{\frac{p}{(q-1)^2}}} \frac{1}{n} \sum_{i=1}^n \twonorm{\grad_\beta \loss(z_i, \est(\lambda))}^{\frac{p+(q-1)^2}{(q-1)^2}}  \\
	&+\half\frac{1}{n^{\frac{2p}{(q-1)^2}}} \frac{(\kappa_{p,\lambda}^\lambda)^{\frac{2}{q-1}}}{c_m^{\frac{2p}{(q-1)^2}}} \frac{1}{n} \sum_{i=1}^n\opnorm{\Hess_\beta \loss(z_i, \tilde{s}_{i}) } \twonorm{\grad_\beta \loss(z_i, \est(\lambda)) }^{\frac{2p}{(q-1)^2}}\\
	&+\frac{1}{n^{\frac{p+q-1}{(q-1)^2}}}\frac{(\kappa_{p,\lambda}^\lambda)^{\frac{1}{q-1}}}{c_\obj^{\frac{p+q-1}{(q-1)^2}}} \frac{1}{n} \sum_{i=1}^n\opnorm{\Hess_\beta \loss(z_i, s_{i}) } \twonorm{\grad_\beta \loss(z_i,\est(\lambda)) }^{\frac{p+q-1}{(q-1)^2}}\\
&\leq 
	\frac{1}{n^{\frac{p}{(q-1)^2}}} \frac{(\kappa_{p,\lambda}^\lambda)^{\frac{1}{q-1}}}{c_\obj^{\frac{p}{(q-1)^2}}} \gradlossboundmixed{0}{\frac{p+(q-1)^2}{(q-1)^2}}  
	+ \half\frac{1}{n^{\frac{2p}{(q-1)^2}}} \frac{(\kappa_{p,\lambda}^\lambda)^{\frac{2}{q-1}}}{c_m^{\frac{2p}{(q-1)^2}}} \gradlossboundmixed{1}{\frac{2p}{(q-1)^2}} 
	+ \frac{1}{n^{\frac{p+q-1}{(q-1)^2}}}\frac{(\kappa_{p,\lambda}^\lambda)^{\frac{1}{q-1}}}{c_\obj^{\frac{p+q-1}{(q-1)^2}}}  \gradlossboundmixed{1}{\frac{p+q-1}{(q-1)^2}}.
\ealignt

\subsubsection{Proof of \cref{eq:acv-cv-ho-bound2}: Proximity of $\ACV_p$ and \CV}
\label{sec:acvp-cv-1}
The proof of the bound \cref{eq:acv-cv-ho-bound2} is identical to that of the bound \cref{eq:acv-cv-ho-bound2_1} once we substitute $2\kappa_{p,\lambda}^\lambda$ for $\kappa_{p,\lambda}^\lambda$ by invoking \cref{eq:acv-cv-est-rho-bound1} in place of \cref{eq:acv-cv-est-ho-bound1}.

\section{Proof of \cref{assumptions-hold}: Sufficient conditions for assumptions}\label{app:assumptions-hold}

We prove each of the independent claims in turn.
\paragraph{\cref{HessobjLipschitz} holds} This first claim follows from the triangle inequality and the definition of the Lipschitz constant $\Lip{}$.

\paragraph{$\mbi{\est(\lam) \to \est(\infty)}$}
For each $\lambda\in[0,\infty)$, by the Optimizer Comparison \cref{optimizer-comparison} with $\varphi_2 = \reg$ and $\varphi_1 = \frac{1}{\lam} \obj(\Pemp, \cdot, \lam)$ and the nonnegativity of $\loss$,
\balignt
\nu_\reg(\twonorm{\est(\lam) - \est(\infty)})
    \leq \frac{1}{\lam}(\loss(\Pemp, \est(\infty)) - \loss(\Pemp, \est(\lam)))
    \leq \frac{1}{\lam}\loss(\Pemp, \est(\infty)).
\ealignt
Therefore, $\nu_\reg(\twonorm{\est(\lam) - \est(\infty)}) \to 0$ as $\lam \to \infty$. 
Now, since $\nu_\reg$ is increasing, its inverse $\omega_\reg$ is increasing with $\omega_\reg(0) = 0$, and hence we have $\twonorm{\est(\lam) - \est(\infty)} \to 0$ as  $\lam \to \infty$.

\paragraph{\cref{curvedobj} holds}
Fix any $\Lambda\subseteq[0,\infty]$, and let $\mineig$ denote the minimum eigenvalue.
The local strong convexity of $\reg$ implies that there exist a neighborhood $\mathcal{N}$ of $\est(\infty)$ and some $c_\reg > 0$ for which $\Hess \reg(\beta) \geq c_\reg \mathrm{Id}$ for all $\beta\in\mathcal{N}$.
Since $\est(\lam) \to \est(\infty)$ as $\lam \to \infty$, there exists $\lam_\reg < \infty$ such that $\est(\lam) \in \mathcal{N}$ for all $\lam \geq \lam_\reg$. 
Hence, for any $\lam, \lam' \in \Lambda$ and $i\in[n]$, we may use the $c_\obj$-strong convexity of $\obj(\Pcv{i}, \cdot,\lambda')$ and $\obj(\Pcv{i},\cdot,0) = \loss(\Pcv{i}, \cdot)$ to conclude that
\balignt
\mineig(\Hess_\beta \obj(\Pcv{i}, \est(\lambda),\lambda')) 
=
\mineig(\Hess_\beta \loss(\Pcv{i}, \est(\lambda)) + \lam' \Hess_\beta \reg(\est(\lambda)))
\geq 
\max(c_\obj, (c_\obj+\lam' c_\reg) \indic{\lam \geq \lambda_\reg}).
\ealignt
Furthermore, the $c_\obj$-strong convexity and differentiability of $\obj(\Pcv{i}, \cdot,\lambda)$ 
imply that $\obj(\Pcv{i}, \cdot,\lambda)$ has $\nu_\obj(r) = c_\obj r^2$ gradient growth.
Thus, \cref{curvedobj} is satisfied for $\Lambda$.

\paragraph{\cref{gradlossboundmixed} holds}
Fix any $\Lambda\subseteq[0,\infty]$ and $\lambda\in\Lambda$.
For each $i\in[n]$, the triangle inequality and the definition of the Lipschitz constant imply
\balignt
\twonorm{\grad_\beta\loss(z_i, \est(\lam))}
&\leq 
\twonorm{\grad_\beta\loss(z_i, \est(\infty))}
+
\twonorm{\grad_\beta\loss(z_i, \est(\lam))-\grad_\beta\loss(z_i, \est(\infty))} \\
&\leq 
\twonorm{\grad_\beta\loss(z_i, \est(\infty))}
+
L_i \twonorm{\est(\lam)-\est(\infty)}.
\ealignt
Moreover, since $\obj(\Pcv{i}, \cdot,\lambda)$ is $c_\obj$-strongly convex and the minimum eigenvalue is a concave function, Jensen's inequality gives for each $\beta$
\balignt
\mineig(\obj(\Pemp, \beta,\lambda))
= \mineig(\frac{1}{n-1} \sum_{i=1}^n \obj(\Pcv{i}, \beta,\lambda))
\geq \frac{1}{n-1} \sum_{i=1}^n \mineig(\obj(\Pcv{i}, \beta,\lambda))
\geq \frac{n}{n-1} c_\obj.
\ealignt
Hence $\obj(\Pemp, \cdot,\lambda)$ has  $\nu_\obj(r) = \frac{n}{n-1}c_\obj r^2$ gradient growth, and the Optimizer Comparison \cref{optimizer-comparison} with $\varphi_2 = \lambda\reg$ and $\varphi_1 = \obj(\Pemp, \cdot, \lam)$ and Cauchy-Schwarz imply
\balignt
\frac{n}{n-1}c_\obj\twonorm{\est(\lam) - \est(\infty)}
    \leq \twonorm{\grad_\beta \loss(\Pemp, \est(\infty))}.
\ealignt
Therefore, 
\balignt
\gradlossboundmixed{s}{r}
	&\leq \textfrac{1}{n}\textsum_{i=1}^n
	L_i^s(\twonorm{\grad_\beta\loss(z_i, \est(\infty))}
	+
	\frac{n-1}{n} \frac{L_i}{c_\obj} \twonorm{\grad_\beta \loss(\Pemp, \est(\infty))})^r
	< \infty.
\ealignt

\section{Proof of \cref{thm:acv-approximates-acvij}: $\ACVIJ$-\ACV assessment error}\label{app:acv-approximates-acvij}
We will prove the following more detailed statement from which \cref{thm:acv-approximates-acvij} immediately follows.
\begin{theorem}[$\ACVIJ$-\ACV assessment error]\label{thm:acv-approximates-acvij-full} If \cref{curvedobj} holds for $\Lambda \subseteq [0,\infty]$, then,
for each $\lam \in\Lambda$,
\balignt
\twonorm{\acvijest{i}(\lambda) - \acvest{i}(\lambda)} 
\leq \frac{\opnorm{\Hess_\beta \loss(z_i,\est(\lambda))} \twonorm{\grad_\beta \loss(z_i,\est(\lambda))}}{c_{\lambda,\lambda}^2n^2}
\label{eq:acvest-acvestij-bound}
\ealignt
 where $c_{\lambda,\lambda} \defeq c_\loss + \lam c_\reg \indic{\lam \geq \lambda_\reg}$.
If, in addition, \cref{gradlossboundmixed} holds for $\Lambda$ and each $(s,r) \in \{(1,2), (2,2), (3,2)\}$, then 
\balignt
|\ACV^{\text{\em IJ}}(\lambda) - \ACV(\lambda)| &\leq \frac{\gradlossboundmixed{1}{2}}{c_{\lambda,\lambda}^2n^2}  +\frac{\gradlossboundmixed{2}{2}}{c_{\lambda,\lambda}^3n^3} +  \frac{\gradlossboundmixed{3}{2}}{2c_{\lambda,\lambda}^4n^4}\label{eq:acv_diff_acvij}. %
\ealignt 
\end{theorem}

\subsection{Proof of \cref{eq:acvest-acvestij-bound}: Proximity of \ACV and $\ACVIJ$ estimators}
We begin with a lemma that controls the discrepancy between two Newton (or, more generally, proximal Newton) estimators.  Recall the definition of the proximal operator $\prox_H^{\varphi_0}$ \cref{eq:prox-operator}. 
\begin{lemma}[Proximal Newton comparison]\label{lem:proximal-comparison}
 For any $\beta, g\in\reals^d$, invertible $H, \tilde{H} \in \reals^{d\times d}$, and convex $\varphi_0$,
 the proximal Newton estimators
 $$\beta_H = \prox_H^{\varphi_0}(\beta- H^{-1}g)
 \qtext{and}
 \beta_{\tilde{H}} = \prox_{\tilde{H}}^{\varphi_0}(\beta- \tilde{H}^{-1}g)$$
satisfy
\begin{talign}
 \twonorm{\beta_H - \beta_{\tilde{H}}}
    \leq 
    \frac{\twonorm{(\tilde H - H)(\beta_{H} - \beta)}}{\mineig(\tilde{H})\vee 0} \leq \frac{\opnorm{\tilde H - H}\twonorm{\beta_{H} - \beta}}{\mineig(\tilde{H})\vee 0}.
\end{talign}
\end{lemma}
\begin{proof}
If $\mineig(\tilde{H}) \leq 0$, the claim is vacuous, so assume  $\mineig(\tilde{H}) > 0$.
 Writing $\varphi_2(x) = \half \|\beta- \tilde{H}^{-1}g -x \|_{\tilde{H}}^2 + \varphi_0(x)$ and  $\varphi_1(x) = \half \|\beta- H^{-1}g -x \|_{H}^2 + \varphi_0(x)$, note that $\beta_{\tilde{H}} = \argmin_x \varphi_2(x)$ and $\beta_{H} = \argmin_x \varphi_1(x)$ by the definition of the proximal operator \cref{eq:prox-operator}. Importantly, $\varphi_2$ is subdifferentiable and satisfies the gradient growth property with $\nu_{\varphi_2}(r) = \mineig(\tilde{H})r^2$. Invoking the Optimizer Comparison \cref{optimizer-comparison} and Cauchy-Schwarz, we have
\balignt
 \mineig(\tilde{H}) \twonorm{\beta_H - \beta_{\tilde{H}}}^2
    \leq \inner{\tilde{H}(\beta - \beta_H)- g  - H(\beta - \beta_H) + g}{\beta_H - \beta_{\tilde{H}}} \leq \twonorm{(\tilde{H} - H)(\beta - \beta_H)}\twonorm{\beta_H - \beta_{\tilde{H}} }.
\ealignt
Rearranging both sides gives the first advertised inequality.  
\end{proof}
Now fix any $\lambda\in\Lambda$ and $i\in [n]$, and let 
\balignt
\tilde{H} = \Hess_\beta \obj(\Pcv{i}, \est(\lambda),\lambda)
\qtext{and}
H = \Hess_\beta \obj(\Pemp, \est(\lambda),\lambda) = \frac{n}{n-1} \frac{1}{n} \sum_{j=1}^n \Hess_\beta \obj(\Pcv{j}, \est(\lambda),\lambda).
\ealignt
By \cref{curvedobj}, $\mineig(\tilde{H}) \geq c_{\lambda,\lam}$.
Moreover, \cref{curvedobj}, the concavity of the minimum eigenvalue, and Jensen's inequality imply
\balignt
\mineig(H) \geq \frac{n}{n-1} \frac{1}{n} \sum_{j=1}^n \mineig(\Hess_\beta \obj(\Pcv{j}, \est(\lambda),\lambda)) \geq \frac{n}{n-1} c_{\lambda,\lam} \geq c_{\lambda,\lam}.
\ealignt
Hence, we may apply \cref{lem:proximal-comparison} with $\beta_H = \acvijest{i}(\lam)$, $\beta_{\tilde{H}} = \acvest{i}(\lam)$, $\beta = \est(\lam)$, and $\varphi_0\equiv 0$ to find that
\balignt
\twonorm{ \acvijest{i}(\lambda) - \acvest{i}(\lambda) }%
&\leq \frac{1}{c_{\lambda,\lam}} \opnorm{\Hess_\beta \obj(\Pemp,\est(\lambda), \lambda) - \Hess_\beta \obj(\Pcv{i},\est(\lambda), \lambda)}\twonorm{\acvijest{i}(\lambda) - \est(\lambda) }\\
&= \frac{1}{n^2} \frac{1}{c_{\lambda,\lam}} \opnorm{\Hess_\beta \loss(z_i,\est(\lambda)) }\twonorm{\Hess_\beta \obj(\Pemp, \est(\lambda),\lambda)^{-1}\nabla_\beta \loss(z_i,\est(\lambda) )}\\
&\leq  \frac{1}{n^2}\frac{1}{c_{\lam,\lam}^2}\opnorm{\Hess_\beta \loss(z_i,\est(\lambda)) }\twonorm{\grad_\beta \loss(z_i,\est(\lambda))}. %
\ealignt

\subsection{Proof of \cref{eq:acv_diff_acvij}: Proximity of \ACV and $\ACVIJ$}
Fix any $\lambda\in\Lambda$.
To control the discrepancy between $\ACV(\lambda)$ and $\ACVIJ(\lambda)$, we first rewrite the difference using Taylor's theorem with Lagrange remainder: 
\balignt
  \ACVIJ(\lambda) - \ACV(\lambda) 
  &= \frac{1}{n}\sum_{i=1}^n\loss(z_i,\acvijest{i}(\lambda) ) - \loss(z_i,\acvest{i}(\lambda) ) \\
 &= \frac{1}{n}\sum_{i=1}^n\inner{\grad_\beta \loss(z_i, \acvest{i}(\lambda) )}{ \acvijest{i}(\lambda)-\acvest{i}(\lambda)}
  + \half\Hess_\beta \loss(z_i, \tilde{s}_{i})[\acvest{i}(\lambda)-\acvijest{i}(\lambda)]^{\otimes 2}
 \ealignt
 for some $\tilde{s}_{i}\in \{ t\acvest{i}(\lambda)+ (1- t)\acvijest{i}(\lambda) :t \in [0,1]\}$.
 We next use the mean-value theorem to expand each function $\inner{\grad_\beta \loss(z_i, \cdot)}{ \acvijest{i}(\lambda)-\acvest{i}(\lambda)}$ around the full-data estimator $\est(\lambda)$:
 \balignt
    \ACVIJ(\lambda) - \ACV(\lambda)   
  &= \frac{1}{n}\sum_{i=1}^n\inner{\grad_\beta \loss(z_i, \est(\lambda) )}{ \acvijest{i}(\lambda)-\acvest{i}(\lambda)}
  + \half\Hess_\beta \loss(z_i, \tilde{s}_{i})[\acvest{i}(\lambda)-\acvijest{i}(\lambda)]^{\otimes 2} \notag \\
  &+ \inner{\Hess_\beta \loss(z_i, s_i)(\acvest{i}(\lambda) - \est(\lambda))}{\acvijest{i}(\lambda)-\acvest{i}(\lambda)}
 \ealignt
  for some  $s_{i}\in \{ t\est(\lambda)+ (1- t)\acvest{i}(\lambda) :t \in [0,1]\}$.
  Now, by \cref{curvedobj}, we have 
  \balignt
  \twonorm{\acvest{i}(\lambda) - \est(\lambda) } 
  	= \frac{1}{n} \twonorm{\textrm{H}_i^{-1}\grad_\beta \loss(z_i,\est(\lambda))}
	\leq \frac{1}{n c_{\lambda,\lam}} \twonorm{\grad_\beta \loss(z_i,\est(\lambda))}.
  \ealignt
Combining these observations with Cauchy-Schwarz, the definition of the operator norm, the estimator proximity result \cref{eq:acvest-acvestij-bound}, the definition of the Lipschitz constant $\Lip{}(\grad_\beta\loss(z_i,\cdot))$, and \cref{gradlossboundmixed} we obtain
\balignt
 	| \ACVIJ(\lambda) - \ACV(\lambda) | 
\leq 
	\,&\frac{1}{n} \sum_{i=1}^n \twonorm{\grad_\beta \loss(z_i, \est(\lambda))}\twonorm{\acvest{i}(\lambda) -\acvijest{i}(\lambda)}
	+\half\opnorm{\Hess_\beta \loss(z_i, \tilde{s}_{i}) } \twonorm{\acvest{i}(\lambda)-\acvijest{i}(\lambda) }^2\\
	&+\opnorm{\Hess_\beta \loss(z_i, s_{i}) } \twonorm{\acvest{i}(\lambda) - \est(\lambda) }\twonorm{\acvest{i}(\lambda)-\acvijest{i}(\lambda)}\\
	 \leq\,&\frac{1}{n^2c_{\lambda,\lam}^2} \frac{1}{n}\sum_{i=1}^n\opnorm{\Hess_\beta \loss(z_i,\est(\lambda)) }\twonorm{\grad_\beta \loss(z_i, \est(\lambda))}^2  \\
   & + \half\frac{1}{n^4c_{\lambda,\lam}^4}\frac{1}{n} \sum_{i=1}^n  \opnorm{\Hess_\beta \loss(z_i, \tilde{s}_{i}) } \opnorm{\Hess_\beta \loss(z_i,\est(\lambda)) }^2\twonorm{\grad_\beta \loss(z_i,\est(\lambda))}^2\\
    &+ \frac{1}{n^3c_{\lambda,\lam}^3} \frac{1}{n}\sum_{i=1}^n  \opnorm{\Hess_\beta \loss(z_i, s_{i}) }  \opnorm{\Hess_\beta \loss(z_i,\est(\lambda)) }\twonorm{\grad_\beta \loss(z_i,\est(\lambda))}^2 \\
    \leq\,&\frac{1}{n^2c_{\lambda,\lam}^2} \gradlossboundmixed{1}{2}+  \frac{1}{2n^4c_{\lambda,\lam}^4}\gradlossboundmixed{3}{2} +\frac{1}{n^3c_{\lambda,\lam}^3}\gradlossboundmixed{2}{2}.
\ealignt

\section{Proof of \cref{Thm:WeakCurvACV}: \ACV-\CV selection error}
\label{App:WeakCurvACV}
The first claim follows immediately from the following more detailed version of \cref{Thm:WeakCurvACV}.
\begin{theorem}[\ACV proximity implies $\est$ proximity]\label{Thm:WeakCurvACV-full} Suppose 
\cref{HessobjLipschitz,,curvedobj}  hold for some $\Lambda \subseteq [0,\infty]$
and each $(s,r) \in \{(0,2),(1,2)\}$. %
Then, for all $\lambda', \lambda\in\Lambda$ with $\lambda'< \lambda$,
\balignt\label{eq:weak_curvature_bound}
\|\est(\lambda) - \est(\lambda')\|_2^2&\leq C_{1,\lam,\lam'}\left(\frac{\tilde{C}_{2,\lam,\lam'}}{n} + \ACV(\lambda) - \ACV(\lambda') + \frac{C_{3,\lam,\lam'}}{n^2}\right),
\ealignt
for $C_{1,\lam,\lam'}$ and $C_{3,\lam,\lam'}$ defined in \cref{Thm:StrongCurvACV} and 
\balignt
 \tilde{C}_{2,\lam,\lam'} &= \frac{3\gradlossboundmixed{0}{2} }{c_\loss + \lambda c_\pi\indic{\lam \geq \lam_\reg}} + \frac{\gradlossboundmixed{0}{2} }{c_\loss + \lambda' c_\pi\indic{\lam' \geq \lam_\reg}}.
\ealignt
\end{theorem}
\begin{proof}
Fix any $\lambda', \lambda\in\Lambda$ with $\lambda'< \lambda$.
We will proceed precisely an in the proof of \cref{Thm:StrongCurvACV-full}, except we will provide alternative bounds for the quantities $\Delta T_2$ and $\Delta T_3$ in the loss decomposition \cref{eq:delta-T3}.
First, we apply Cauchy-Schwarz, the definition of the operator norm, the triangle inequality, and the arithmetic-geometric mean inequality in turn to find
\balignt
|\Delta T_2| 
&= \frac{1}{n}|\frac{1}{n} \sum_{i=1}^n \inner{ \grad_\beta\loss(z_i, \est(\lambda))-\grad_\beta\loss(z_i, \est(\lambda'))}{\Hess_\beta \obj(\Pcv{i}, \est(\lambda),\lam)^{-1}\grad_\beta \loss(z_i, \est(\lambda))}| \\
&\leq \frac{1}{n^2} \sum_{i=1}^n\opnorm{\Hess_\beta \obj(\Pcv{i}, \est(\lambda),\lam)^{-1}}\twonorm{\grad_\beta\loss(z_i, \est(\lambda))}( \twonorm{\grad_\beta\loss(z_i, \est(\lambda))}+ \twonorm{\grad_\beta\loss(z_i, \est(\lambda'))})\\
&\leq \frac{1}{n^2} \sum_{i=1}^n\opnorm{\Hess_\beta \obj(\Pcv{i}, \est(\lambda),\lam)^{-1}}(\frac{3}{2}\twonorm{\grad_\beta\loss(z_i, \est(\lambda))}^2 +\half \twonorm{\grad_\beta\loss(z_i, \est(\lambda'))}^2)\\
&\leq \frac{1}{n}\frac{2\gradlossboundmixed{0}{2} }{c_\loss + \lambda c_\pi\indic{\lam \geq \lam_\reg}},
\label{eq:weak_T2_bound}
\ealignt 
where we have used \cref{curvedobj} and \cref{gradlossboundmixed} for $(s,r) = (0,2)$ in the final line.

Next, we again apply the triangle inequality, the definition of the operator norm, the arithmetic-geometric mean inequality, \cref{curvedobj}, and \cref{gradlossboundmixed} for $(s,r) = (0,2)$ to obtain
\balignt
|\Delta T_3 |
= \,&|\frac{1}{n^2}\sum_{i=1}^n \inner{\grad_\beta\loss(z_i, \est(\lambda'))}{\Hess_\beta \obj(\Pcv{i}, \est(\lambda), \lambda)^{-1}\grad_\beta \loss(z_i, \est(\lambda))} \\
&- \frac{1}{n^2}\sum_{i=1}^n \inner{\grad_\beta\loss(z_i, \est(\lambda'))}{\Hess_\beta \obj(\Pcv{i}, \est(\lambda'), \lambda')^{-1}\grad_\beta \loss(z_i, \est(\lambda'))} |\\
\leq \,&\frac{1}{n}\frac{1}{n}\sum_{i=1}^n \twonorm{\grad_\beta\loss(z_i, \est(\lambda'))}\opnorm{\Hess_\beta \obj(\Pcv{i}, \est(\lambda), \lambda)^{-1}} \twonorm{\grad_\beta\loss(z_i, \est(\lambda))}\\
&+ \twonorm{\grad_\beta\loss(z_i, \est(\lambda'))}^2\opnorm{\Hess_\beta \obj(\Pcv{i}, \est(\lambda'), \lambda')^{-1}} \\
\leq \,&\frac{1}{n}\frac{1}{n}\sum_{i=1}^n (\half \twonorm{\grad_\beta\loss(z_i, \est(\lambda'))}^2 + \half\twonorm{\grad_\beta\loss(z_i, \est(\lambda))}^2) \opnorm{\Hess_\beta \obj(\Pcv{i}, \est(\lambda), \lambda)^{-1}} \\
&+ \twonorm{\grad_\beta\loss(z_i, \est(\lambda'))}^2\opnorm{\Hess_\beta \obj(\Pcv{i}, \est(\lambda'), \lambda')^{-1}} \\
\leq  \,&\frac{1}{n}(\frac{\gradlossboundmixed{0}{2} }{c_\loss + \lambda c_\pi\indic{\lam \geq \lam_\reg}}+\frac{\gradlossboundmixed{0}{2} }{c_\loss + \lambda' c_\pi\indic{\lam' \geq \lam_\reg}}).
\label{eq:weak_T3_bound}
\ealignt
Plugging the bounds \cref{eq:weak_T2_bound,eq:weak_T3_bound} into the proof \cref{Thm:StrongCurvACV-full} yields the result.
\end{proof}
The second claim \cref{eq:weak-risk} follows the first and the following bound on $|\ACV(\lambda_\CV) - \ACV(\lambda_\ACV)|$.
\begin{lemma}\label{lem:bound_ACV} 
Suppose \cref{gradlossboundmixed,,curvedobj,HessobjLipschitz} hold for some $\Lambda \subseteq [0,\infty]$ and each $(s,r) \in\{ (0,3),(1,3),(1,4)\}$.
If $\lambda_\ACV \in \argmin_{\lam \in \Lam} \ACV(\lam)$ and $\lam_\CV \in \argmin_{\lam \in \Lam} \CV(\lam)$, then 
\balignt
0 \leq \ACV(\lambda_\CV) - \ACV(\lambda_\ACV) \leq 2(\frac{\kappa_2}{n^2}\frac{ \gradlossboundmixed{0}{3}}{c_\obj^2}
	+\frac{\kappa_2}{n^3} \frac{\gradlossboundmixed{1}{3}}{c_\obj^3} 
	+\frac{\kappa_2^2}{n^4} \frac{\gradlossboundmixed{1}{4}}{2c_\obj^4}).
\ealignt
\end{lemma}

\begin{proof}
Since $\lambda_\CV$ minimizes \CV and $\lambda_\ACV$ minimizes \ACV,
\balignt
0 &\leq \ACV(\lambda_\CV) - \ACV(\lambda_\ACV) \leq \ACV(\lambda_\CV)  - \ACV(\lambda_\ACV) + \CV(\lambda_\ACV) - \CV(\lambda_\CV).
\ealignt
The result now follows from two applications of \cref{thm:acv-approximates-cv}.
\end{proof}

\section{Proof of \cref{Thm:StrongCurvACV}: Strong \ACV-\CV selection error}
\label{App:StrongCurvACV}
The first claim follows immediately from the following more detailed version of \cref{Thm:StrongCurvACV}, proved in \cref{sec:proof-StrongCurvACV-full}.
\begin{theorem}[Strong \ACV proximity implies $\est$ proximity]\label{Thm:StrongCurvACV-full} 
Suppose 
 \cref{gradlossboundmixed,,curvedobj,,boundedHessobj,HessobjLipschitz} hold %
for some $\Lambda \subseteq [0,\infty]$ with $0\in\Lambda$ and each $(s,r) \in \{(0,2),(1,1),(1,2)\}$. %
 Suppose also $\|\grad \reg(\est(0))\|_2 >0$. 
 Then for all $\lambda', \lambda\in\Lambda$ with $\lambda'< \lambda$,
 \balignt\label{eq:strong_risk_bound}
&\twonorm{\est(\lambda) - \est(\lambda')}^2  
 \leq C_{1,\lam,\lam'}\left(C_{2,\lam,\lam'}/n\twonorm{\est(\lambda) - \est(\lambda')} + \ACV(\lambda) - \ACV(\lambda') + C_{3,\lam,\lam'}/n^2\right) \qtext{and hence} \\
 &|\twonorm{\est(\lam) - \est(\lam')} - \frac{C_{1,\lam,\lam'}C_{2,\lam,\lam'}}{2n}| \leq  \sqrt{\frac{C_{1,\lam,\lam'}^2C_{2,\lam,\lam'}^2}{4n^2} +C_{1,\lam,\lam'}( \ACV(\lambda) - \ACV(\lambda')) + \frac{C_{1,\lam,\lam'}C_{3,\lam,\lam'}}{n^2}}, \label{eq:strong_risk_bound_after_quadratic_formula}
\ealignt 
where 
\balignt
C_{1,\lam,\lam'} &=   \frac{2}{c_\obj}\frac{\lam - \lam'}{\lam + \lam'}\frac{n-1}{n}, \\
    C_{2,\lam,\lam'} &=  \frac{2\gradlossboundmixed{1}{1}}{c_\loss + \lambda c_\pi\indic{\lam \geq \lam_\reg}}
    +
    \frac{2\gradlossboundmixed{0}{2}\kappa_{2,\lam'}^{\lam'}}{c_\loss + \lambda' c_\pi\indic{\lam \geq \lam_\reg}} 
    +\frac{n-1}{n} \frac{\gradlossboundmixed{0}{2}C_{\pi,2}\kappa_{1,\lam}^{\lam}\kappa_{1,\lam'}^{\lam}}{\|\grad \reg(\est(0))\|_2 c_\obj},\\
    C_{3,\lam,\lam'} &= \frac{\gradlossboundmixed{1}{2}}{c_\obj^2}
\ealignt
for $\kappa_{p,\lam'}^{\lam}$ defined in \cref{thm:acv-approximates-cv-HO-full}.
\end{theorem}

The second claim follows directly from the \cref{Thm:StrongCurvACV-full} bound \cref{eq:strong_risk_bound_after_quadratic_formula} and \cref{lem:bound_ACV}.

\subsection{Proof of \cref{Thm:StrongCurvACV-full}}\label{sec:proof-StrongCurvACV-full}
Fix any $\lambda', \lambda\in\Lambda$ with $\lambda'< \lambda$.
The statement \cref{eq:strong_risk_bound_after_quadratic_formula} follows directly from \cref{eq:strong_risk_bound} and the quadratic formula, so we will focus on establishing the bound \cref{eq:strong_risk_bound}.
We begin by writing the difference in estimator training losses as a difference in $\ACV$ values plus a series of error terms:
\balignt\label{delta_acv_decomposition}
\loss(\Pemp, \est(\lambda))  - \loss(\Pemp, \est(\lambda'))
    &= 
    \ACV(\lambda) - \ACV(\lambda') + 
    \Delta T_1 
    - \Delta T_2 - \Delta T_3 
\ealignt
for 
\balignt
\Delta T_1 & \defeq \widehat{\ACV}(\lambda) - \ACV(\lambda) + 
    \ACV(\lambda') - \widehat{\ACV}(\lambda'),\\
\widehat{\ACV}(\lambda) 
    &\defeq \loss(\Pemp, \est(\lambda))
    + \frac{1}{n}\sum_{i=1}^n \inner{\grad_\beta\loss(z_i, \est(\lambda))}{\acvest{i}(\lambda) - \est(\lambda)} ,\\
\Delta T_2 & \defeq \frac{1}{n}\sum_{i=1}^n \inner{\grad_\beta\loss(z_i, \est(\lambda))-\grad_\beta\loss(z_i, \est(\lambda'))}{\acvest{i}(\lambda) - \est(\lambda)}, \qtext{and} \\
\Delta T_3 &\defeq \frac{1}{n}\sum_{i=1}^n \inner{\grad_\beta\loss(z_i, \est(\lambda'))}{\acvest{i}(\lambda) - \est(\lambda) - (\acvest{i}(\lambda') - \est(\lambda')) }. \label{eq:delta-T3}
\ealignt
Here, $\widehat{\ACV}(\lambda)$ arises by first-order Taylor-expanding each $\loss(\Pemp, \acvest{i}(\lam))$ about $\est(\lam)$ in the expression of $\ACV(\lambda)$.
To complete the proof, we will bound $\Delta T_1$, $\Delta T_2$, $\Delta T_3$, and $\loss(\Pemp, \est(\lambda))  - \loss(\Pemp, \est(\lambda'))$ in turn.
\subsubsection{Bounding $\Delta T_1$}
To control $\Delta T_1$, we will appeal to the following lemma which shows that $\widehat{\ACV}$ provides an $O(1/n^2)$ approximation to \ACV, uniformly in $\lambda$.
The proof can be found \cref{sec:ACV-ACV_2-difference}.
\begin{lemma}[\ACV-$\widehat{\ACV}$ approximation error]\label{ACV-ACV2-difference}
Suppose 
 \cref{curvedobj,,gradlossboundmixed} hold for some $\Lambda \subseteq [0,\infty]$ and $(s,r) = (1,2)$. 
    Then, for each $\lambda \in \Lambda$,
    \balignt\label{eq:ACV-ACV2-difference}
    |\ACV(\lambda) - \widehat{\ACV}(\lambda)| \leq 
    \frac{1}{n^2}\frac{\gradlossboundmixed{1}{2}}{2c_\obj^2}.
    \ealignt
\end{lemma}
Applying \cref{ACV-ACV2-difference} to $\lam$ and $\lam'$, we obtain
\balignt\label{eq:final_T1_bound}
|\Delta T_1|
	\leq \frac{1}{n^2}\frac{\gradlossboundmixed{1}{2}}{c_\obj^2}.
\ealignt
\subsubsection{Bounding $\Delta T_2$}
We employ the mean value theorem, Cauchy-Schwarz, the definition of the operator norm, \cref{curvedobj}, and
\cref{gradlossboundmixed} for $(s,r)=(1,1)$ to find that
\balignt
|\Delta T_2| 
&= \frac{1}{n}|\frac{1}{n} \sum_{i=1}^n \inner{ \grad_\beta\loss(z_i, \est(\lambda))-\grad_\beta\loss(z_i, \est(\lambda'))}{\Hess_\beta \obj(\Pcv{i}, \est(\lambda),\lam)^{-1}\grad_\beta \loss(z_i, \est(\lambda))}| \\
&= \frac{1}{n}|\frac{1}{n} \sum_{i=1}^n \inner{ \Hess_\beta\loss(z_i, s_{\lam,\lam'})(\est(\lambda)- \est(\lambda'))}{\Hess_\beta \obj(\Pcv{i}, \est(\lambda),\lam)^{-1}\grad_\beta \loss(z_i, \est(\lambda))}| \\
&\leq \frac{1}{n} \twonorm{\est(\lambda)- \est(\lambda')} \frac{1}{n}\sum_{i=1}^n \opnorm{\Hess_\beta \obj(\Pcv{i}, \est(\lambda),\lam)^{-1}}\opnorm{\Hess_\beta\loss(z_i, s_{\lam,\lam'})}\twonorm{\grad_\beta \loss(z_i, \est(\lambda))}  \\
&\leq \frac{1}{n}\twonorm{\est(\lambda) - \est(\lambda')}\frac{\gradlossboundmixed{1}{1} }{c_\loss + \lambda c_\pi\indic{\lam \geq \lam_\reg}}\label{eq:T2-bound}
\ealignt 
for some convex combination $s_{\lam,\lam'}$ of  $\est(\lambda)$ and $\est(\lambda')$.
\subsubsection{Bounding $\Delta T_3$}
We next show that the double difference term $\Delta T_3$ is controlled by estimator proximity $\twonorm{\est(\lambda) - \est(\lambda')}$ and regularization parameter proximity $|\lam - \lam'|$ times an extra factor of $1/n$. This result is proved in \cref{App:lem_T3_bound}. %
\begin{lemma}\label{lem:T3_bound} 
Suppose  \cref{gradlossboundmixed,,HessobjLipschitz,curvedobj,boundedHessobj} hold for some $\Lambda \subseteq [0,\infty]$ and each $(s,r)=\{(0,2),(1,1)\}$.
Then, for all $\lam, \lam' \in \Lam$, 
\balignt\label{eq:T3_bound}
|\Delta T_3| 
    &\leq 
    \frac{1}{n}\twonorm{\est(\lambda) - \est(\lambda')} \left(
    \frac{\gradlossboundmixed{1}{1}}{c_\loss + \lambda c_\pi\indic{\lam \geq \lam_\reg}}
    +
    \frac{2\gradlossboundmixed{0}{2}\kappa_{2,\lam'}^{\lam'}}{c_\loss + \lambda' c_\pi\indic{\lam \geq \lam_\reg}} 
    \right) \\
    &+ \frac{1}{n}|\lambda - \lambda'| \frac{\gradlossboundmixed{0}{2}C_{\pi,2}}{(c_\loss+c_\reg\lam\indic{\lam \geq \lam_\reg})(c_\loss+c_\reg\lam' \indic{\lam \geq \lam_\reg})}
\ealignt
for $\Delta T_3$ defined in \cref{eq:delta-T3} and $\kappa_{2,\lam'}^{\lam'}$ defined in \cref{thm:acv-approximates-cv-HO-full}.
\end{lemma}
We combine \cref{eq:T3_bound} with the following bound on $|\lambda - \lambda'|$ proved in \cref{App:lambda_diff_bound}:
\begin{lemma}[$\est$ proximity implies $\lam$ proximity]\label{lem:lambda_diff_bound}
Suppose  \cref{curvedobj,boundedHessobj} hold for some $\Lambda \subseteq [0,\infty]$ with $0\in\Lam$.
Then, for all $\lam, \lam' \in \Lam$, 
\balignt\label{eq:lambda_diff}
|\lambda- \lambda'|& \leq \frac{n-1}{n}\frac{(C_{\loss,2} + \lambda C_{\reg,2})(C_{\loss,2} + \lambda'C_{\reg,2})}{ c_\obj  }  \frac{1}{\|\grad \reg(\est(0))\|_2 }\| \est(\lambda) - \est(\lambda')\|_2.
\ealignt
\end{lemma}
Together, \cref{eq:T3_bound} and \cref{eq:lambda_diff} imply 
\balignt\label{eq:final_T3_bound}
|\Delta T_3| 
    &\leq 
    \frac{1}{n}\twonorm{\est(\lambda) - \est(\lambda')} \left(
    \frac{\gradlossboundmixed{1}{1}}{c_\loss + \lambda c_\pi\indic{\lam \geq \lam_\reg}}
    +
    \frac{2\gradlossboundmixed{0}{2}\kappa_{2,\lam'}^{\lam'}}{c_\loss + \lambda' c_\pi\indic{\lam \geq \lam_\reg}} 
    +\frac{n-1}{n} \frac{\gradlossboundmixed{0}{2}C_{\pi,2}\kappa_{1,\lam}^{\lam}\kappa_{1,\lam'}^{\lam}}{\|\grad \reg(\est(0))\|_2 c_\obj}\right).
\ealignt

\subsubsection{Putting the pieces together}
Our final lemma, proved in \cref{sec:initbound}, establishes that, due to the curvature of the loss, two estimators with similar training loss must also be close in Euclidean norm.
\begin{lemma}[Loss curvature]\label{initbound}
Suppose that for some $c_\obj >0$ and  $0 \leq \lam' < \lam \leq \infty$ and all $i \in [n]$, $\obj(\Pcv{i},\cdot, \lam)$ and $\obj(\Pcv{i},\cdot, \lam')$ have $\nu_\obj(r) = c_\obj r^2$ gradient growth.  Then
\balignt
\frac{c_\obj}{2}\, \twonorm{\est(\lambda) - \est(\lambda')}^2 \frac{\lambda+\lambda'}{\lambda - \lambda'} \leq \frac{1}{n}\sum_{i=1}^n \loss(\Pcv{i}, \est(\lambda))  - \loss(\Pcv{i}, \est(\lambda')) 
= \frac{n-1}{n}(\loss(\Pemp, \est(\lambda))  - \loss(\Pemp, \est(\lambda'))).
\ealignt
\end{lemma}
The advertised result \cref{eq:strong_risk_bound} now follows by combining \cref{initbound} with the loss difference decomposition \cref{delta_acv_decomposition} and the component bounds \cref{eq:final_T1_bound,eq:T2-bound,eq:final_T3_bound}.

\subsection{Proof of \cref{ACV-ACV2-difference}: \ACV-$\widehat{\ACV}$ approximation error}
\label{sec:ACV-ACV_2-difference}
By Taylor's theorem with Lagrange remainder, 
    \balignt
    &\ACV(\lambda) - \widehat{\ACV}(\lambda)
    = 
    \frac{1}{2n}\sum_{i=1}^n \Hess_\beta\loss(z_i, s_i)[\acvest{i}(\lambda) - \est(\lambda)]^{\otimes 2}
\ealignt
for some $
s_{i} 
    \in \lineseg_i = \{ t\est(\lambda)+ (1- t)\acvest{i}(\lambda) :t \in [0,1]\} $. 
    \cref{curvedobj} implies that $\obj(\Pcv{i}, \est(\lambda), \lambda)$ is $c_\obj$ strongly convex and hence that $\Hess_\beta\obj(\Pcv{i}, \est(\lambda), \lambda) \psdge c_\obj \I_d$. Therefore, we may apply the definition of the operator norm and \cref{gradlossboundmixed} to find that
   \balignt
    |\ACV(\lambda) - \widehat{\ACV}(\lambda)| 
    &\leq 
        \frac{1}{2n}\sum_{i=1}^n\opnorm{ \Hess_\beta\loss(z_i, s_i)} \twonorm{\acvest{i}(\lambda) - \est(\lambda)}^{2}\\
    &= 
        \frac{1}{2n}\sum_{i=1}^n\opnorm{ \Hess_\beta\loss(z_i, s_i)}
        \frac{1}{n^2} \twonorm{\obj(\Pcv{i}, \est(\lambda), \lambda)^{-1}\grad_\beta \loss(z_i, \est(\lambda))}^{2}\\
     & \leq \frac{1}{2n}\sum_{i=1}^n\opnorm{ \Hess_\beta\loss(z_i, s_i)} \frac{1}{n^2}\frac{1}{c_\obj^2}\twonorm{\grad_\beta\loss(z_i,\est(\lambda))}^2\\
     &\leq \frac{1}{n^2}\frac{\gradlossboundmixed{1}{2}}{2c_\obj^2}.
    \ealignt

\subsection{Proof of \cref{lem:T3_bound}: $\Delta T_3$-bound}
\label{App:lem_T3_bound}
Fix any $\lam,\lam' \in\Lam$. We first expand $\Delta T_3$ into three terms:
\balignt
\Delta T_3 
&= \frac{1}{n^2}\sum_{i=1}^n \inner{\grad_\beta\loss(z_i, \est(\lambda'))}{\Hess_\beta \obj(\Pcv{i}, \est(\lambda), \lambda)^{-1}(\grad_\beta \loss(z_i, \est(\lambda)) - \grad_\beta \loss(z_i, \est(\lambda')))} \\
&+  \frac{1}{n^2}\sum_{i=1}^n \inner{\grad_\beta\loss(z_i, \est(\lambda'))}{(\Hess_\beta \obj(\Pcv{i}, \est(\lambda), \lambda')^{-1} - \Hess_\beta \obj(\Pcv{i}, \est(\lambda'), \lambda')^{-1})\grad_\beta \loss(z_i, \est(\lambda'))}\\
&+ \frac{1}{n^2}\sum_{i=1}^n \inner{\grad_\beta\loss(z_i, \est(\lambda'))}{(\Hess_\beta \obj(\Pcv{i}, \est(\lambda), \lambda)^{-1} - \Hess_\beta \obj(\Pcv{i}, \est(\lambda), \lambda')^{-1})\grad_\beta \loss(z_i, \est(\lambda'))}\\
&\defeq \Delta T_{31} + \Delta T_{32} + \Delta T_{33}.
\ealignt
Precisely as in \cref{eq:T2-bound} we obtain
\balignt
|\Delta T_{31}| &\leq \frac{1}{n}\twonorm{\est(\lambda) - \est(\lambda')}\frac{\gradlossboundmixed{1}{1} }{c_\loss + \lambda c_\pi\indic{\lam \geq \lam_\reg}}.
\ealignt
Furthermore, we may use Cauchy-Schwarz, the definition of the operator norm, \cref{HessobjLipschitz,curvedobj}, and \cref{gradlossboundmixed} with $(s,r) = (0,2)$ to find
\balignt
|\Delta T_{32}| 
 & \leq \frac{1}{n} \frac{1}{n} \sum_{i=1}^n \twonorm{\grad_\beta\loss(z_i, \est(\lambda'))}^2 \opnorm{\Hess_\beta \obj(\Pcv{i}, \est(\lambda), \lambda')^{-1}}\opnorm{\Hess_\beta \obj(\Pcv{i}, \est(\lambda'), \lambda')^{-1}}\\
 &\qquad\qquad\qquad\cdot\opnorm{\Hess_\beta \obj(\Pcv{i}, \est(\lambda), \lambda') - \Hess_\beta \obj(\Pcv{i}, \est(\lambda'), \lambda')} \\
 & \leq  \frac{\gradlossboundmixed{0}{2}(C_{\loss,3} + \lambda' C_{\pi,3})}{(c_\loss + \lambda' c_\pi\indic{\lam \geq \lam_\reg})(c_\loss + \lambda' c_\pi\indic{\lam' \geq \lam_\reg})} \frac{1}{n}\twonorm{\est(\lambda) - \est(\lambda')}.
\ealignt
Finally, Cauchy-Schwarz, the definition of the operator norm, \cref{boundedHessobj,curvedobj}, and \cref{gradlossboundmixed} with $(s,r) = (0,2)$ yield
\balignt
|\Delta T_{33}| &\leq  \frac{1}{n} \frac{1}{n} \sum_{i=1}^n\twonorm{\Hess_\beta \obj(\Pcv{i}, \est(\lambda), \lambda)^{-1}}\twonorm{\grad_\beta\loss(z_i, \est(\lambda'))}^2\opnorm{\Hess_\beta \obj(\Pcv{i}, \est(\lambda), \lambda') - \Hess_\beta \obj(\Pcv{i}, \est(\lambda), \lambda)}\\
&\quad \cdot \twonorm{ \Hess_\beta \obj(\Pcv{i}, \est(\lambda), \lambda')^{-1}}\\
&=  \frac{1}{n}|\lambda - \lambda'|\opnorm{\Hess_\beta \pi( \est(\lambda))} \frac{1}{n} \sum_{i=1}^n\twonorm{\Hess_\beta \obj(\Pcv{i}, \est(\lambda), \lambda)^{-1}}\twonorm{ \Hess_\beta \obj(\Pcv{i}, \est(\lambda), \lambda')^{-1}} \twonorm{\grad_\beta \loss(z_i, \est(\lambda'))}^2\\
&\leq\frac{1}{n}|\lambda - \lambda'| \frac{\gradlossboundmixed{0}{2}C_{\pi,2}}{(c_\loss+c_\reg\lam\indic{\lam \geq \lam_\reg})(c_\loss+c_\reg\lam' \indic{\lam \geq \lam_\reg})}.
\ealignt
We obtain the desired result by applying the triangle inequality and summing these three estimates.

\subsection{Proof of \cref{lem:lambda_diff_bound}: $|\lambda - \lambda'|$-bound}
\label{App:lambda_diff_bound}
We begin with a lemma that allows us to rewrite a regularization parameter difference in terms of an estimator difference.
\begin{lemma}\label{rewrite_lam_diff}
Fix any $\lam,\lam' \in [0,\infty]$. If $\grad_\beta \obj(\Pemp, \cdot, \lam')$ is absolutely continuous, then 
\balignt 
0 
&= (\lambda - \lambda')\grad_\beta\reg(\est(\lambda)) + \E[\Hess_\beta\obj(\Pemp,S_{\lambda,\lambda'}, \lambda')](\est(\lambda) - \est(\lambda')) \\
&= \frac{\lambda' - \lambda}{\lam}\grad_\beta\loss(\Pemp,\est(\lambda)) + \E[\Hess_\beta\obj(\Pemp,S_{\lambda,\lambda'}, \lambda')](\est(\lambda) - \est(\lambda')).
\ealignt
for $S_{\lambda, \lambda'}$ distributed uniformly on the set $\{t \est(\lambda) + (1 - t) \est(\lambda') \, :\, t \in [0,1]\}$.
\end{lemma}
\begin{proof}
The first order optimality conditions for $\est(\lam)$ and $\est(\lam')$ and the absolute continuity of 
$\grad_\beta \obj(\Pemp, \cdot, \lam')$ imply that
\balignt 
0 &= \grad_\beta \obj(\Pemp,\est(\lambda), \lambda) 
= \grad_\beta \loss(\Pemp,\est(\lambda)) + \lambda \grad\reg(\est(\lambda)) \\
&= (\lambda - \lambda')\grad_\beta\reg(\est(\lambda)) + \grad_\beta \obj(\Pemp,\est(\lambda), \lambda')  \\
&= (\lambda - \lambda')\grad_\beta\reg(\est(\lambda)) + (\grad_\beta \obj(\Pemp,\est(\lambda), \lambda')  - \grad_\beta \obj(\Pemp,\est(\lambda'), \lambda') ) \\
&= (\lambda - \lambda')\grad_\beta\reg(\est(\lambda)) + \E[\Hess_\beta\obj(\Pemp,S_{\lambda,\lambda'}, \lambda')](\est(\lambda) - \est(\lambda')) \\
&= \frac{\lambda' - \lambda}{\lam}\grad_\beta\loss(\Pemp,\est(\lambda)) + \E[\Hess_\beta\obj(\Pemp,S_{\lambda,\lambda'}, \lambda')](\est(\lambda) - \est(\lambda'))
\ealignt
by Taylor's theorem with integral remainder.
\end{proof}

Now fix any $\lam,\lam' \in \Lam$.
Since $\grad_\beta \obj(\Pemp, \cdot, \lam')$, $\grad_\beta \obj(\Pemp, \cdot, \lam)$, and $\grad_\beta \obj(\Pemp, \cdot, 0) = \grad_\beta \loss(\Pemp, \cdot)$ are absolutely continuous by \cref{boundedHessobj}, we may apply \cref{rewrite_lam_diff} first to $(\lam, \lam')$, then to $(\lam,0)$, and finally to $(0,\lam)$ to obtain
\balignt 
0 
&= \frac{\lambda'- \lambda}{\lambda}\grad_\beta \loss(\Pemp,\est(\lambda))+ \E[\Hess_\beta\obj(\Pemp,S_{\lambda,\lambda'}, \lambda')](\est(\lambda) - \est(\lambda'))\\
&= \frac{\lambda'- \lambda}{\lambda}\E[\Hess_\beta\loss(\Pemp,S_{\lambda, 0})](\est(\lambda) - \est(0))+ \E[\Hess_\beta\obj(\Pemp,S_{\lambda,\lambda'}, \lambda')](\est(\lambda) - \est(\lambda'))\\
&= (\lambda- \lambda')\E[\Hess_\beta \loss(\Pemp,S_{\lambda, 0})]\E[\nabla^2_\beta \obj(\Pemp,S_{0,\lambda}, \lambda)]^{-1}\grad_\beta \reg(\est(0)) + \E[\Hess_\beta\obj(\Pemp,S_{\lambda,\lambda'}, \lambda')](\est(\lambda) - \est(\lambda'))
\ealignt
where $S_{\lambda, \lambda'}$ is distributed uniformly on the set $\{t \est(\lambda) + (1 - t) \est(\lambda') \, :\, t \in [0,1]\}$ and $S_{\lambda, 0},S_{0,\lambda}$ are distributed uniformly on the set $\{t \est(\lambda) + (1 - t) \est(0) \, :\, t \in [0,1]\}$. Rearranging and taking norms gives the identity
\balignt
|\lambda - \lambda'|\twonorm{\grad_\beta \reg(\est(0))} = \twonorm{\E[\nabla^2_\beta \obj(\Pemp,S_{\lambda, 0}, \lambda)]\E[\Hess_\beta \loss(\Pemp,S_{\lambda, 0})]^{-1}\E[\Hess_\beta\obj(\Pemp,S_{\lambda,\lambda'}, \lambda')](\est(\lambda) - \est(\lambda'))}.
 \ealignt
Our gradient growth assumption for the regularization parameter $0$ implies that each $\loss(\Pcv{i},\cdot)$ is $c_{\obj}$-strongly convex \citep[Lem.~1]{Nesterov08}.  Therefore,
 \balignt
 \E[\Hess_\beta \loss(\Pemp,S_{\lambda, 0})] =  \frac{n}{n-1}\frac{1}{n}\sum_{i=1}^n\E[\Hess_\beta \loss(\Pcv{i},S_{\lambda, 0})] \psdge c_\obj \frac{n}{n-1} \I_d.
 \ealignt
Applying this inequality along with Cauchy Schwarz and \cref{boundedHessobj} for $\lam$ and  $\lam'$,
we now conclude that 
\balignt
|\lambda- \lambda'|
\leq\frac{n-1}{n}\frac{(C_{\loss,2} + \lambda C_{\reg,2})(C_{\loss,2} + \lambda'C_{\reg,2})}{ c_\obj  } \frac{1}{\|\grad \reg(\est(0))\|_2 }\| \est(\lambda) - \est(\lambda')\|_2.
\ealignt

\subsection{Proof of \cref{initbound}: Loss curvature}
\label{sec:initbound}
Fix any $\lam, \lam' \in \Lam$ with $\lam > \lam'$ and $i \in [n]$, and consider the functions $\varphi_2 = \obj(\Pcv{i},\cdot, \lambda')$ and $\varphi_1 = \frac{\lam'}{\lam}\obj(\Pcv{i},\cdot, \lambda)$.  
The gradient growth condition in \cref{curvedobj} implies that $\obj(\Pcv{i},\cdot, \lambda')$ and $\obj(\Pcv{i},\cdot, \lambda)$ are $c_\obj$-strongly convex.  
Hence, $\varphi_2$ admits a $\nu_{\varphi_2}(r) = \frac{c_m}{2} r^2$ error bound, and $\varphi_1$ admits a $\nu_{\varphi_1}(r) = \frac{\lam'}{\lam}\frac{c_m}{2} r^2$ error bound.
The result now follows immediately from the optimizer comparison bound \cref{eq:opt_comp_error_bound}.

\section{Proof of \cref{prop:assesment_failure}}
\label{app:assesment_failure}

We write $\E_n[z_i] \defeq \frac{1}{n}\sum_{i=1}^n z_i$ to denote a sample average. 
 Consider the Lasso estimator
\balignt
\est(\lambda) 
    \defeq \argmin_\beta  \tfrac{1}{2n}\sum_{i=1}^n (\beta - z_i)^2 + \lambda |\beta| 
= \max(\bar{z} - \lambda,0).
\ealignt
Define $\epsilon_i = z_i - \bar{z}$.
The leave-one-out mean $\bar z_{-i}=\tfrac{1}{n}\sum_{j\neq i} z_i$ is equal to $\bar{z} - \tfrac{z_i}{n}$.
The IJ estimate $\acvijest{i}(\lambda)$ is given by
$$
\acvijest{i}(\lambda) = \begin{cases}
0 & \lambda \geq \bar z\\
\est(\lambda) - \frac{z_i - \est(\lambda)}{n}=
\bar z - \lambda - \frac{\epsilon_i + \lambda}{n} & \textrm{else.}
\end{cases}
$$
The IJ approximate cross-validation estimate for this estimator is therefore given by
$$
2\ACVIJ(\lambda)= \E_n[(z_i-\acvijest{i}(\lambda) )^2] = \begin{cases}
    \bar z^2 +1 & \lambda \geq \bar z\\
    (\lambda^2 + 1) (1+\tfrac{1}{n})^2 & \textrm{else.}
\end{cases}
$$

By construction of our dataset, $\lambda \geq 0 > - \bar{z}_{-i}$ for all $i$. Now,
the leave-one-out estimator of $\beta$ is given by
$\cvest{i}(\lam)=\max\left(\bar{z} -\tfrac{n}{n-1} \lambda - \tfrac{1}{n-1} \epsilon_i, 0\right), $
and the leave-one-out CV estimate is given by
\balign
  &2\CV(\lambda) %
= \E_n[(\bar{z} - z_i)^2] + (\bar{z} - \est(\lam))^2\\  
&+ 2 \E_n[(\cvest{i}(\lam)- \est(\lam)) z_i]
+ \E_n[(\cvest{i}(\lam)- \est(\lam))^2]\\
&= 1 + \min(\bar{z}, \lambda)^2\\
&+ \E_n[\max (\left(\bar{z} -\tfrac{n}{n-1} \lambda - \tfrac{1}{n-1} \epsilon_i, 0\right)  - (\max(\bar{z} - \lambda,0))\cdot (\bar z + \epsilon_i)]\\
&+\E_n[(\max \left(\bar{z} -\tfrac{n}{n-1} \lambda - \tfrac{1}{n-1} \epsilon_i, 0\right)  - (\max(\bar{z} - \lambda,0))^2].
\ealign
Evaluating these expressions at $\lambda = \bar{z}$, we get
$\cvest{i}(\bar{z})=\max\left( - \tfrac{z_i}{n-1}, 0\right), $ so that
\balign
2\ACVIJ(\bar{z}) &= \bar{z}^2 + 1,\\
  2\CV(\bar{z}) 
&= 1 + \bar{z}^2 + 
\E_n[\max(- \tfrac{z_i}{n-1},0) \cdot  z_i]\\
&+ \E_n[\max(- \tfrac{z_i}{n-1},0)^2]
\ealign
and thus
\balign
  &2(\ACVIJ(\bar{z}) - \CV(\bar{z})) \\
  &=\E_n[\max(- \tfrac{z_i}{n-1},0) \cdot  z_i]+ \E_n[\max(- \tfrac{z_i}{n-1},0)^2]\\
  &= \E_n[z_i^2 \cdot \bs 1(z_i<0)] \cdot \frac{n}{(n-1)^2}.
\ealign
Our dataset was constructed such that
\balign
    \P_n (\epsilon_i < 0)&=1/2 \\
    \E_n[\epsilon_i | \epsilon_i <0] &= \sqrt{2/\pi}\\ %
    \E_n[\epsilon_i^2 | \epsilon_i <0] &= 1\\
    \E_n[z_i^2 \cdot \bs 1(z_i<0)]
    &= \E_n[(\bar z^2 + 2 \bar z \epsilon_i +\epsilon_i^2) \cdot \bs 1(\epsilon_i<0)]\\
    &=\tfrac{1}{2} (\bar z^2 - 2 \bar z \sqrt{2/\pi} +1),
\ealign
and thus
\balignt
  \ACVIJ(\bar{z}) - \CV(\bar{z}) 
  &= \frac{n}{4(n-1)^2}
  \left ( 1 - 2\bar{z} \sqrt{2/\pi} + \bar z^2\right ).
\ealignt
To make $\lambda = \bar{z}$ the $\ACVIJ$ minimizing choice in this example (a condition we have not assumed thus far), it suffices to have $\bar{z} \leq \sqrt{2/n}$.
For the choice $\bar{z} = \sqrt{2/n}$, we get
$$\ACVIJ(\bar{z}) - \CV(\bar{z}) = \frac{n}{4(n-1)^2}
  \left ( 1 - \frac{4}{\sqrt{n\pi}} + \tfrac{2}{n}\right ).$$

\section{Proof of \cref{prop:patchedlasso_failure}}
\label{app:patchedlasso_failure}

We write $\E_n[z_i] \defeq \frac{1}{n}\sum_{i=1}^n z_i$ to denote a sample average. 
 Consider the patched Lasso estimator
$$\est(\lambda) 
    \defeq \argmin_\beta  \tfrac{1}{2n}\sum_i (\beta - z_i)^2 + \lambda \min(|\beta|,\tfrac{\delta}{2} +\tfrac{\beta^2}{2\delta})
= \max\left(\bar{z} - \lambda, \tfrac{\bar{z}}{1+ \lambda/\delta}\right).$$
Define $\epsilon_i = z_i - \bar{z}$.
The leave-one-out mean is equal to
$\bar{z}_{-i} = \bar{z} - \tfrac{\epsilon_i}{n}.$
For $\epsilon_i/n < \bar{z}$, we have
$\cvest{i}(\lambda) -\est(\lambda) =
-\tfrac{\epsilon_i}{n-1}
$
if $\epsilon_i<0$
and
$\cvest{i}(\lambda)-\est(\lambda) =
-\tfrac{\epsilon_i}{n-1} \tfrac{1}{1+ \lambda/\delta}
$
if $\epsilon_i>0$.
Considering left hand derivatives, we also have
$\acvijest{i}(\lambda) -\est(\lambda) =
-\tfrac{\epsilon_i}{n}.
$

For $\lambda=\delta$ and $\bar{z} = 2\delta$, we get $\est(\lambda) = \delta$, and
and thus, by our choice of dataset, 
\balign
  \ACVIJ(\bar{z}) - \CV(\bar{z}) 
  &= \tfrac{1}{2}\E_n\left[(Z_i - \acvijest{i}(\lambda))^2 -
  (Z_i - \cvest{i}(\lambda))^2\right]\\
  &= \tfrac{1}{2}\E_n\left[(2\delta + \epsilon_i - \delta + \tfrac{\epsilon_i}{n})^2 -
  (2\delta + \epsilon_i - \delta +\tfrac{1}{2} \tfrac{\epsilon_i}{n-1})^2
  | \epsilon_i >0\right]\\
  &\quad+ \tfrac{1}{2}\E_n\left[(2\delta + \epsilon_i - \delta + \tfrac{\epsilon_i}{n})^2 -
  (2\delta + \epsilon_i - \delta + \tfrac{\epsilon_i}{n-1})^2
  | \epsilon_i <0\right]\\
  &= \tfrac{1}{2}\E_n\left[(\delta + \epsilon_i (1+\tfrac{1}{n}))^2 -
  (\delta + \epsilon_i(1
  +\tfrac{1}{2(n-1)}))^2
  | \epsilon_i >0\right] +O(\tfrac{1}{n^2})\\
  &=\delta \E_n[\epsilon_i | \epsilon_i>0] \cdot
  \tfrac{1}{n}  
  +O(\tfrac{1}{n^2})\\
  &= \delta \sqrt{2/\pi} \cdot
  \tfrac{1}{n}  
  +O(\tfrac{1}{n^2}).
\ealign

 \section{Proof of \cref{Thm:ProxACVAssessment}: \ProxACV-\CV assessment error}\label{app:proofProxACVAssessment}
The optimization perspective adopted in this paper naturally points towards generalizations of the proximal estimator \cref{eq:proxacvest}. In particular, stronger assessment guarantees can be provided for (regularized) higher-order Taylor approximations of the objective function. For example, for $p \geq 2$, we may define the $p$-th order approximation
\balignt
\proxHOACV(\lambda) &\defeq \frac{1}{n} \sum_{i=1}^n \loss(z_i, \hoproxacvest(\lambda)) \qtext{with}\\
\hoproxacvest(\lambda) &\defeq \argmin_{\beta\in\reals^d} \{\hat{\loss}_{p}(\Pcv{i},\beta;\est(\lambda)) + \lam \reg(\beta)\}
\ealignt
which recovers \ProxACV~\eqref{eq:ACVprox} and the \ProxACV estimator~\cref{eq:proxacvest} in the setting $p=2$. We also define the regularized $p$-th order approximation
\balignt
\proxRHOACV(\lambda) &\defeq \frac{1}{n} \sum_{i=1}^n \loss(z_i, \rhoproxacvest(\lambda)) \qtext{with}\\
\rhoproxacvest(\lambda) &\defeq \argmin_{\beta\in\reals^d} \left\{\hat{\loss}_{p}(\Pcv{i},\beta;\est(\lambda)) + \frac{\Lip{}(\nabla_\beta^p \loss(\Pcv{i}, \cdot))}{p+1}\twonorm{\est(\lambda) - \beta}^{p+1}+ \lam\reg(\beta)\right\},
\ealignt
where $\hat{\loss}_{p}(\Pcv{i},\cdot;\est(\lambda))$ is a $p$-th order Taylor expansion of the loss $\loss_{p}(\Pcv{i},\cdot)$ about $\est(\lam)$, that is, $\hat{f}_{p}(\beta;\est(\lambda))\defeq \sum_{k=0}^{p} \frac{1}{k!} \nabla^k f(\est(\lambda))(\beta - \est(\lambda))^{\otimes k} $. To analyze both, we will make use of the following assumptions which generalize \cref{proxcurvedobj}.
\begingroup
\setcounter{tmp}{\value{assumption}}
\setcounter{assumption}{\value{assumption}-4} 
\renewcommand\theassumption{\arabic{assumption}g}
\begin{assumption}[Curvature of proximal Taylor approximation] 
\label{proxhocurvedtaylorfull} 
For some $p, q, c_\obj > 0$, all $i \in [n]$, and all $\lam$ in a given $\Lambda \subseteq [0,\infty]$,
 $\widehat{\loss}_p(\Pcv{i},\cdot,\lambda;\est(\lambda)) + \lambda \reg$   has $\nu(r) = c_{\obj} r^q$ gradient growth. %
\end{assumption}
\setcounter{assumption}{\value{assumption}-1} 
\renewcommand\theassumption{\arabic{assumption}h}
\begin{assumption}[Curvature of regularized proximal Taylor approximation] 
\label{proxrhocurvedtaylorfull} 
For some $p, q, c_\obj > 0$, all $i \in [n]$, and all $\lam$ in a given $\Lambda \subseteq [0,\infty]$,
 $\widehat{\loss}_p(\Pcv{i},\cdot,\lambda;\est(\lambda))+ \frac{\Lip{}(\grad_\beta^p \loss(\Pcv{i},\cdot, \lambda))}{p+1}\twonorm{\cdot - \est(\lambda)}^{p+1} + \lambda \reg$   has $\nu(r) = c_{\obj} r^q$ gradient growth. 
\end{assumption}
\endgroup

 \cref{Thm:ProxACVAssessment} will then follow from the following more general statement, proved in \cref{sec:proof-proxacv_cvest_bound}.
\begin{theorem}[$\proxHOACV$-\CV and $\proxRHOACV$-\CV assessment error]\label{proxacv_cvest_bound}
If \cref{hocurvedobjfull,proxhocurvedtaylorfull,HesslossLipschitz} hold for some $\Lambda \subseteq [0,\infty]$, then, for all $\lam \in\Lam$ and $i\in[n]$,
 \begin{subequations}
\balignt
\twonorm{\hoproxacvest(\lambda) - \cvest{i}(\lambda)}^{q-1}
    &\leq 
       \tilde{\kappa}_p \twonorm{\cvest{i}(\lambda)- \est(\lambda) }^p \qtext{with}  \tilde{\kappa}_p\defeq  \frac{C_{\loss,{p+1}}}{p!c_\obj }.
          \label{eq:pacv-cv-est-ho-bound1}
\ealignt
If \cref{hocurvedobjfull,proxhocurvedtaylorfull,HesslossLipschitz} hold for some $\Lambda \subseteq [0,\infty]$, then, for all $\lam \in\Lam$ and $i\in[n]$, 
\balignt
\twonorm{\rhoproxacvest(\lambda) - \cvest{i}(\lambda)}^{q-1} 
    &\leq 
       2\tilde{\kappa}_p \twonorm{\cvest{i}(\lambda)- \est(\lambda) }^p. 
          \label{eq:pacv-cv-est-rho-bound1}
        \ealignt
\end{subequations}
  If \cref{HesslossLipschitz,gradlossboundmixed,hocurvedobjfull,proxhocurvedtaylorfull} hold for some $\Lambda \subseteq [0,\infty]$ and each $(s,r) \in \{(0,\frac{p+(q-1)^2}{(q-1)^2}), (1,\frac{2p}{(q-1)^2}), (1,\frac{p+q-1}{(q-1)^2})\}$, then, for all $\lam \in \Lam$, %
  \begin{subequations}
\balignt
&|\proxHOACV(\lambda) - \CV(\lambda)| \\
	&\leq  
		\frac{1}{n^{\frac{p}{(q-1)^2}}} \frac{(\tilde{\kappa}_p)^{\frac{1}{q-1}}}{c_\obj^{\frac{p}{(q-1)^2}}} \gradlossboundmixed{0}{\frac{p+(q-1)^2}{(q-1)^2}}  
		+ \half\frac{1}{n^{\frac{2p}{(q-1)^2}}} \frac{(\tilde{\kappa}_p)^{\frac{2}{q-1}}}{c_\obj^{\frac{2p}{(q-1)^2}}} \gradlossboundmixed{1}{\frac{2p}{(q-1)^2}} 
		+ \frac{1}{n^{\frac{p+q-1}{(q-1)^2}}}\frac{(\tilde{\kappa}_p)^{\frac{1}{q-1}}}{c_\obj^{\frac{p+q-1}{(q-1)^2}}}  \gradlossboundmixed{1}{\frac{p+q-1}{(q-1)^2}}
	\qtext{and} 
	\label{eq:pacv-cv-ho-bound2_1} 
	\ealignt
	If \cref{HesslossLipschitz,gradlossboundmixed,hocurvedobjfull,proxrhocurvedtaylorfull} holds for %
   $\Lambda$ and each $(s,r) \in \{(0,\frac{p+(q-1)^2}{(q-1)^2}), (1,\frac{2p}{(q-1)^2}), (1,\frac{p+q-1}{(q-1)^2})\}$, then, for all $\lam \in \Lam$, 
	\balignt
&|\proxRHOACV(\lambda) - \CV(\lambda)| \\
	&\leq 
		\frac{1}{n^{\frac{p}{(q-1)^2}}} \frac{(2\tilde{\kappa}_p)^{\frac{1}{q-1}}}{c_\obj^{\frac{p}{(q-1)^2}}} \gradlossboundmixed{0}{\frac{p+(q-1)^2}{(q-1)^2}}  
		+ \half\frac{1}{n^{\frac{2p}{(q-1)^2}}} \frac{(2\tilde{\kappa}_p)^{\frac{2}{q-1}}}{c_\obj^{\frac{2p}{(q-1)^2}}} \gradlossboundmixed{1}{\frac{2p}{(q-1)^2}} 
		+ \frac{1}{n^{\frac{p+q-1}{(q-1)^2}}}\frac{(2\tilde{\kappa}_p)^{\frac{1}{q-1}}}{c_\obj^{\frac{p+q-1}{(q-1)^2}}}  \gradlossboundmixed{1}{\frac{p+q-1}{(q-1)^2}}.
	\label{eq:pacv-cv-ho-bound2}
\ealignt
\end{subequations}
\end{theorem}
\cref{Thm:ProxACVAssessment} follows from \cref{proxacv_cvest_bound} with $p = q = 2$ since \cref{proxcurvedobj} (with $0\in\Lam$) implies $\mu = c_\obj$ strong convexity for  $\widehat{\loss}_2(\Pcv{i},\cdot,\lambda;\est(\lambda))$.  Since $\reg$ is convex, we further have $\mu$ strong convexity and hence $\nu(r) = \mu r^2$ gradient growth for  $\widehat{\loss}_2(\Pcv{i},\cdot,\lambda;\est(\lambda)) + \lam\reg(\cdot)$ for each $\lam\in\Lam$. 

\subsection{Proof of \cref{proxacv_cvest_bound}: $\proxHOACV$-\CV and $\proxRHOACV$-\CV assessment error} \label{sec:proof-proxacv_cvest_bound}
\subsubsection{Proof of \cref{eq:pacv-cv-est-ho-bound1} and \cref{eq:pacv-cv-est-rho-bound1}: Proximity of $\proxHOACV$, $\proxRHOACV$, and \CV estimators} 
The proofs follow exactly as in \cref{proof-acv-cv-est-ho-bound1_1,proof-acv-cv-est-ho-bound1} if we take $\varphi (x)= \loss (x) $, $\widehat{\varphi}_p(x;w) = \widehat{\loss}_p(x;w) $, $\varphi_0(x) = \reg (x)$, and $ w= \est(\lambda)$ and invoke \cref{proxhocurvedtaylorfull,proxrhocurvedtaylorfull,HesslossLipschitz}
in place of \cref{hocurvedtaylorfull,rhocurvedtaylorfull,LipschitzpObj}, respectively.
 \subsubsection{Proof of \cref{eq:pacv-cv-ho-bound2,eq:pacv-cv-ho-bound2_1}: Proximity of $\proxHOACV,\proxRHOACV$, and \CV} 
 This proofs follow exactly as in  follows directly from the proof contained in \cref{sec:acvp-cv,sec:acvp-cv-1} if we substitute  \cref{eq:pacv-cv-est-ho-bound1,eq:pacv-cv-est-rho-bound1} for \cref{eq:acv-cv-est-ho-bound1,eq:acv-cv-est-rho-bound1} respectively.

 \section{Proof of \cref{thm:proxacv-close-proxacvij}: $\ProxACVIJ$-\ProxACV assessment error}\label{app:proxacv-close-proxacvij}
 We will prove the following more detailed statement from which \cref{thm:proxacv-close-proxacvij} immediately follows.
\begin{theorem}[$\ProxACVIJ$-\ProxACV assessment error]\label{thm:pacv-approximates-pacvij-full} If \cref{proxcurvedobj} holds for $\Lambda \subseteq [0,\infty]$ with $0 \in \Lam$, then,
for each $\lam \in\Lambda$,
\balignt
\twonorm{ \proxacvijest{i}(\lambda) - \proxacvest{i}(\lambda) }
    \leq 
    \frac{\opnorm{\Hess_\beta \loss(z_i,\est(\lambda))} \twonorm{\grad_\beta \loss(z_i,\est(\lambda))}}{c_{\obj}^2n^2}
    \label{eq:pacvest-pacvestij-bound}
\ealignt
If, in addition, \cref{gradlossboundmixed} holds for $\Lambda$ and each $(s,r) \in \{(1,2), (2,2), (3,2)\}$, then 
\balignt
|\ProxACV^{\text{\em IJ}}(\lambda) - \ProxACV(\lambda)| &\leq \frac{1}{n^2c_\obj^2} \gradlossboundmixed{1}{2}+  \frac{1}{2n^4c_\obj^4}\gradlossboundmixed{3}{2} + \frac{1}{n^3c_\obj^3} \gradlossboundmixed{2}{2}
\label{eq:pacv_diff_pacvij}. %
\ealignt 
\end{theorem}
 
 \subsection{Proof of~\cref{eq:pacvest-pacvestij-bound}: Proximity of $\ProxACVIJ$ and \ProxACV estimators}
 The concavity of the minimum eigenvalue, Jensen's inequality, and \cref{proxcurvedobj} with $0\in\Lam$ imply that 
 \balignt
 \mineig(\acvhess{\loss}) = \mineig(\frac{n}{n-1}\frac{1}{n}\sum_{i=1}^n \acvhess{\loss,i})
 \geq \frac{n}{n-1}\frac{1}{n}\sum_{i=1}^n\mineig(\acvhess{\loss,i})
 \geq \frac{n}{n-1} c_\obj
 \ealignt
 for $\acvhess{\loss} = \Hess_\beta \loss(\Pemp, \est(\lambda))$ and $\acvhess{\loss,i} = \Hess_\beta \loss(\Pcv{i}, \est(\lambda))$.
 Moreover, \cref{proxcurvedobj} with $0\in\Lam$ implies $\mu = c_\obj$ strong convexity for  $\widehat{\loss}_2(\Pcv{i},\cdot,\lambda;\est(\lambda))$; since $\reg$ is convex, we further have $\mu$ strong convexity and hence $\nu(r) = \mu r^2$ gradient growth for  $\widehat{\loss}_2(\Pcv{i},\cdot,\lambda;\est(\lambda)) + \lam\reg(\cdot)$ for each $\lam\in\Lam$. 
 Hence, we may apply the Proximal Newton Comparison \cref{lem:proximal-comparison} to obtain
 \balignt
\twonorm{ \proxacvijest{i}(\lambda) - \proxacvest{i}(\lambda) }
&\leq \frac{1}{c_\obj} \opnorm{\Hess_\beta \loss(\Pemp,\est(\lambda)) - \Hess_\beta \loss(\Pcv{i},\est(\lambda) )}\twonorm{\proxacvest{i}(\lambda) - \est(\lambda) }\\
&\leq  \frac{1}{n}\frac{1}{c_\obj}\opnorm{\Hess_\beta \loss(z_i,\est(\lambda)) }\twonorm{\proxacvest{i}(\lambda) - \est(\lambda) }.
\ealignt
To complete the bound, we note that $\est(\lam) = \prox_{\lambda \reg}^{\acvhess{\loss,i}}( \est(\lambda) - \acvhess{\loss,i}^{-1}\grad_\beta\loss(\Pemp, \est(\lam)))$ and use the 1-Lipschitzness of the $\prox$ operator to conclude that
\balignt\label{eq:proxacvest-est-proximity}
\twonorm{\proxacvest{i}(\lambda) - \est(\lambda) }
    \leq 
    \twonorm{\acvhess{\loss,i}^{-1}(\grad_\beta\loss(\Pemp, \est(\lam)) - \grad_\beta\loss(\Pcv{i}, \est(\lam)))}
    \leq \frac{1}{n c_\obj} \twonorm{\grad_\beta\loss(z_i, \est(\lam))}.
\ealignt

\subsection{Proof of~\cref{eq:pacv_diff_pacvij}: Proximity of $\ProxACVIJ$ and \ProxACV }
Fix any $\lambda\in\Lambda$.
To control the discrepancy between $\ProxACV(\lambda)$ and $\ProxACVIJ(\lambda)$, we first rewrite the difference using Taylor's theorem with Lagrange remainder: 
\balignt
  \ProxACV(\lambda)  -\ProxACV^{\text{IJ}}(\lambda)
  &= \frac{1}{n}\sum_{i=1}^n\loss(z_i,\proxacvest{i}(\lambda) ) - \loss(z_i,\proxacvijest{i}(\lambda)) \\
 &= \frac{1}{n}\sum_{i=1}^n\inner{\grad_\beta \loss(z_i, \cvest{i}(\lambda) )}{ \proxacvest{i}(\lambda)-\proxacvijest{i}(\lambda)}
 \\
  &\quad +\half\Hess_\beta \loss(z_i, \tilde{s}_{i})[\proxacvest{i}(\lambda)-\proxacvijest{i}(\lambda)]^{\otimes 2}
 \ealignt
 for some $\tilde{s}_{i}\in \{ t\proxacvest{i}(\lambda)+ (1- t)\proxacvijest{i}(\lambda) :t \in [0,1]\}$.
 We next use the mean-value theorem to expand each function $\inner{\grad_\beta \loss(z_i, \cdot)}{\proxacvest{i}(\lambda)-\proxacvijest{i}(\lambda)}$ around the full-data estimator $\est(\lambda)$:
 \balignt
   \ProxACV(\lambda) - \ProxACV^{\text{IJ}}(\lambda)  
  &= \frac{1}{n}\sum_{i=1}^n\inner{\grad_\beta \loss(z_i, \est(\lambda) )}{ \proxacvest{i}(\lambda)-\proxacvijest{i}(\lambda)}
 \\&\quad + \half\Hess_\beta \loss(z_i, \tilde{s}_{i})[\proxacvest{i}(\lambda)-\proxacvijest{i}(\lambda)]^{\otimes 2} \notag \\
  &+ \inner{\Hess_\beta \loss(z_i, s_i)(\cvest{i}(\lambda) - \est(\lambda))}{\proxacvest{i}(\lambda)-\proxacvijest{i}(\lambda)}
 \ealignt
  for some  $s_{i}\in \{ t\est(\lambda)+ (1- t)\proxacvijest{i}(\lambda) :t \in [0,1]\}$.
Finally, we invoke Cauchy-Schwarz, the definition of the operator norm, the estimator proximity results \cref{eq:ho_curve_bound_1,eq:pacvest-pacvestij-bound}, and \cref{gradlossboundmixed} to obtain
\balignt
 	| \ProxACV(\lambda) - \ProxACV^{\text{IJ}}(\lambda) | 
\leq 
	\,&\frac{1}{n} \sum_{i=1}^n \twonorm{\grad_\beta \loss(z_i, \est(\lambda))}\twonorm{\proxacvest{i}(\lambda) -\proxacvijest{i}(\lambda)}\\
	&+\half\opnorm{\Hess_\beta \loss(z_i, \tilde{s}_{i}) } \twonorm{\proxacvest{i}(\lambda)-\proxacvijest{i}(\lambda)}^2\\
	&+\opnorm{\Hess_\beta \loss(z_i, s_{i}) } \twonorm{\cvest{i}(\lambda) - \est(\lambda) }\twonorm{\proxacvest{i}(\lambda)-\proxacvijest{i}(\lambda))}\\
&\leq 
    \frac{1}{n^2c_\obj^2} \frac{1}{n}\sum_{i=1}^n\opnorm{\Hess_\beta \loss(z_i,\est(\lambda)) }\twonorm{\grad_\beta \loss(z_i, \est(\lambda))}^2  \\
& \quad + 
    \frac{1}{2n^4c_\obj^4}\frac{1}{n} \sum_{i=1}^n  \opnorm{\Hess_\beta \loss(z_i, \tilde{s}_{i}) } \opnorm{\Hess_\beta \loss(z_i,\est(\lambda)) }^2\twonorm{\grad_\beta \loss(z_i,\est(\lambda))}^2\\
&\quad + 
    \frac{1}{n^3c_\obj^3} \frac{1}{n}\sum_{i=1}^n  \opnorm{\Hess_\beta \loss(z_i, \tilde{s}_{i}) }  \opnorm{\Hess_\beta \loss(z_i,\est(\lambda)) }\twonorm{\grad_\beta \loss(z_i,\est(\lambda))}^2 \\
&\leq \frac{1}{n^2c_\obj^2} \gradlossboundmixed{1}{2}+  \frac{1}{2n^4c_\obj^4}\gradlossboundmixed{3}{2} + \frac{1}{n^3c_\obj^3} \gradlossboundmixed{2}{2}.
\ealignt

\section{Proof of \cref{Thm:ProxACVSelection}: \ProxACV-\CV selection error} \label{app:ProxACVSelection}

The first claim follows immediately from the following more detailed version of \cref{Thm:ProxACVSelection}. \begin{theorem}[Weak \ProxACV proximity implies $\est$ proximity]\label{Thm:WeakCurvProxACV-full} 
Suppose 
\cref{HesslossLipschitz,,proxcurvedobj}  hold for some $\Lambda \subseteq [0,\infty]$
and each $(s,r) \in \{(0,2),(1,2)\}$. 
Then, for all $\lambda', \lambda\in\Lambda$ with $\lambda'< \lambda$,
\balignt\label{eq:prox_weak_curvature_bound}
\|\est(\lambda) - \est(\lambda')\|_2^2&\leq C_{1,\lam,\lam'}\big(\frac{4\gradlossboundmixed{0}{2} }{n c_\obj } 
+ \frac{\gradlossboundmixed{1}{2}}{n^2c_\obj^2} + \ProxACV(\lambda) - \ProxACV(\lambda')\big),
\ealignt
where
$C_{1,\lam,\lam'} = \frac{2}{c_\obj}\frac{\lam - \lam'}{\lam + \lam'}\frac{n-1}{n}$.
\end{theorem}
\begin{proof}
Fix any $\lambda', \lambda\in\Lambda$ with $\lambda'< \lambda$.
We begin by writing the difference in estimator training losses as a difference in $\ProxACV$ values plus a series of error terms:
\balignt \label{eq:delta_proxacv_decomposition}
\loss(\Pemp, \est(\lambda))  - \loss(\Pemp, \est(\lambda'))
    &= 
    \ProxACV(\lambda) - \ProxACV(\lambda') + 
    \Delta T_1 
    - \Delta T_2 - \Delta T_3 
\ealignt
for 
\balignt
\Delta T_1 & \defeq \reallywidehat{\ProxACV}(\lambda) - \ProxACV(\lambda) + 
    \ProxACV(\lambda') - \reallywidehat{\ProxACV}(\lambda')\\
\reallywidehat{\ProxACV}(\lambda) 
    &\defeq \loss(\Pemp, \est(\lambda))
    + \frac{1}{n}\sum_{i=1}^n \inner{\grad_\beta\loss(z_i, \est(\lambda))}{\proxacvest{i}(\lam) - \est(\lambda)} \\
\Delta T_2 & \defeq \frac{1}{n}\sum_{i=1}^n \inner{\grad_\beta\loss(z_i, \est(\lambda))-\grad_\beta\loss(z_i, \est(\lambda'))}{\proxacvest{i}(\lam) - \est(\lambda)} \\
\Delta T_3 &\defeq \frac{1}{n}\sum_{i=1}^n \inner{\grad_\beta\loss(z_i, \est(\lambda'))}{\proxacvest{i}(\lam) - \est(\lambda) - (\proxacvest{i}(\lam') - \est(\lambda')) }.
\ealignt
We will frequently use the bound \cref{eq:proxacvest-est-proximity} which follows from \cref{proxcurvedobj} and implies
\balignt
\twonorm{\est(\lambda) - \proxacvest{i}(\lambda)}
    \leq \frac{1}{n}\frac{1}{c_\obj}\twonorm{\grad_\beta \loss(z_i,\est(\lambda))}
\ealignt
for each $i\in[n]$.

To bound $\Delta T_1$, we first employ Taylor's Theorem with Lagrange remainder, 
    \balignt
    &\ProxACV(\lambda) - \reallywidehat{\ProxACV}(\lambda)
    = 
    \frac{1}{2n}\sum_{i=1}^n \Hess_\beta\loss(z_i, s_i)[\proxacvest{i}(\lambda) - \est(\lambda)]^{\otimes 2}
\ealignt
for some $
s_{i} 
    \in \lineseg_i = \{ t\est(\lambda)+ (1- t)\cvest{i}(\lambda) :t \in [0,1]\} $. 
 Next we apply the definition of the operator norm, the bound \cref{eq:proxacvest-est-proximity}, and \cref{gradlossboundmixed} with $(s,r) = (1,2)$
   \balignt
    |\ProxACV(\lambda) - \reallywidehat{\ProxACV}(\lambda)| 
    &\leq \frac{1}{2n}\sum_{i=1}^n\opnorm{ \Hess_\beta\loss(z_i, s_i)} \twonorm{\proxacvest{i}(\lam) - \est(\lambda)}^{2}\\
     & \leq \half  \frac{1}{n}\sum_{i=1}^n\opnorm{ \Hess_\beta\loss(z_i, s_i)}\left( \frac{1}{n^2}\frac{1}{c_\obj^2}\twonorm{\grad_\beta\loss(z_i,\est(\lambda))}^2 \right)\\
     &\leq \frac{1}{n^2}\frac{1}{2c_\obj^2}\gradlossboundmixed{1}{2}   
    \ealignt
Since an identical bound holds for $\lam'$, we have
\balignt
|\Delta T_1|
    \leq \frac{1}{n^2}\frac{1}{c_\obj^2}\gradlossboundmixed{1}{2}.
\ealignt

To bound $\Delta T_2$ and $\Delta T_3$, we apply Cauchy-Schwarz, the triangle inequality, the bound \cref{eq:proxacvest-est-proximity}, the arithmetic-geometric mean inequality, and \cref{gradlossboundmixed} with $(s,r) = (0,2)$ to find
\balignt
|\Delta T_2| 
&= \frac{1}{n}|\frac{1}{n} \sum_{i=1}^n \inner{ \grad_\beta\loss(z_i, \est(\lambda))-\grad_\beta\loss(z_i, \est(\lambda'))}{\proxacvest{i}(\lambda) - \est(\lambda)}| \\
&\leq \frac{1}{n} \sum_{i=1}^n\twonorm{\proxacvest{i}(\lambda) - \est(\lambda)}\left(\twonorm{\grad_\beta\loss(z_i, \est(\lambda))} +  \twonorm{\grad_\beta\loss(z_i, \est(\lambda'))}\right)\\
&\leq \frac{1}{n} \sum_{i=1}^n \frac{1}{n}\frac{1}{c_\obj}\twonorm{\grad_\beta \loss(z_i,\est(\lambda))}\left(\twonorm{\grad_\beta\loss(z_i, \est(\lambda))} +  \twonorm{\grad_\beta\loss(z_i, \est(\lambda'))}\right)\\
&\leq \frac{1}{n}\frac{1}{c_\obj} \frac{1}{n} \sum_{i=1}^n \frac{3}{2}\twonorm{\grad_\beta \loss(z_i,\est(\lambda))}^2 + \half\twonorm{\grad_\beta \loss(z_i,\est(\lambda'))}^2 
\leq \frac{2}{n}\frac{\gradlossboundmixed{0}{2} }{c_\obj } \qtext{and} \\
|\Delta T_3| 
&\leq \frac{1}{n}\sum_{i=1}^n \twonorm{\grad_\beta\loss(z_i, \est(\lambda'))}\left(\twonorm{\proxacvest{i}(\lambda) - \est(\lambda)} + \twonorm{\proxacvest{i}(\lambda') - \est(\lambda') } \right)\\
&\leq \frac{1}{n} \sum_{i=1}^n \frac{1}{n}\frac{1}{c_\obj}\twonorm{\grad_\beta \loss(z_i,\est(\lambda))}\left(\twonorm{\grad_\beta\loss(z_i, \est(\lambda))} +  \twonorm{\grad_\beta\loss(z_i, \est(\lambda'))}\right)
\leq \frac{2}{n}\frac{\gradlossboundmixed{0}{2} }{c_\obj }.
\ealignt 
The advertised result \cref{eq:prox_weak_curvature_bound} now follows by combining \cref{initbound} with the loss difference decomposition \cref{eq:delta_proxacv_decomposition} and the component $\Delta T_1,\Delta T_2, $ and $\Delta T_3$ bounds.
\end{proof}

The second claim in \cref{Thm:ProxACVSelection} follows from~\cref{Thm:WeakCurvProxACV-full} and the following lemma.
\begin{lemma}\label{lem:bound_PACV} 
If 
 \cref{gradlossboundmixed,,proxcurvedobj,HesslossLipschitz} hold for some $ \Lambda\subseteq [0,\infty]$ and each %
$(s,r) \in\{ (0,3),(1,3),(1,4)\}$. 
If $\lambda_\ProxACV \in \argmin_{\lam \in \Lam} \ProxACV(\lam)$ and $\lam_\CV \in \argmin_{\lam \in \Lam} \CV(\lam)$, then 
\balignt
0 \leq \ProxACV(\lambda_\CV) - \ProxACV(\lambda_\ACV) \leq \hspace{-.075cm}
   \frac{C_{\loss,{3}}}{n^2} \big(\frac{\gradlossboundmixed{0}{3}}{c_\obj^3}\hspace{-.075cm} +\hspace{-.075cm} \frac{\gradlossboundmixed{1}{3}}{nc_\obj^{4}}\hspace{-.075cm} +\hspace{-.075cm} \frac{C_{\loss,{3}}\gradlossboundmixed{1}{4}}{4n^2c_\obj^6}\hspace{-.025cm}\big).
\ealignt
\end{lemma}

\begin{proof}
Since $\lambda_\CV$ minimizes \CV and $\lambda_\ProxACV$ minimizes \ProxACV,
\balignt
0 &\leq \ProxACV(\lambda_\CV) - \ProxACV(\lambda_\ProxACV)\\
&\leq \ProxACV(\lambda_\CV)  - \ProxACV(\lambda_\ProxACV) + \CV(\lambda_\ProxACV) - \CV(\lambda_\CV).
\ealignt
The result now follows from two applications of \cref{Thm:ProxACVAssessment}.
\end{proof}

\section{Proof of \cref{prop:lack-curvature}: $O(1/\sqrt{n})$ error bound is tight}\label{app:lack-curvature}
For each $i\in [n]$, define $\bar{z}_i = \bar{z} - \frac{1}{n} z_i$. For the target objective,
\balignt
\ProxACV(\lam) = \CV(\lam)  = \frac{1}{n} \sum_{i=1}^n \half (\cvest{i}(\lam) - z_i)^2  \qtext{for all} \lam \in [0,\infty],
\ealignt
 and a straightforward calculation shows 
\balignt
\est(\lambda) &= \begin{cases} \bar{z} - \lam & \text{if $\bar{z}> \lam$}\\
\bar{z} + \lam & \text{if $\bar{z}<- \lam$}\\
0 & \text{otherwise}
\end{cases}\qtext{and}
\cvest{i}(\lambda) =  \begin{cases} \bar{z}_i - \lam & \text{if $\bar{z}_i> \lam$}\\
\bar{z}_i + \lam & \text{if $\bar{z}_i<- \lam$}\\
0 & \text{otherwise}
\end{cases}
\ealignt
for each $i \in[n]$.
Hence, for each $i$, $\cvest{i}(0) = \bar{z}_i$,
\balignt
\cvest{i}(\bar{z}) 
	=  
	\begin{cases} 
		- \frac{1}{n} z_i & \text{if $0 > z_i$}\\
		\bar{z}_i + \bar{z} & \text{if $2n\bar{z}< z_i$}\\
		0 & \text{otherwise}
	\end{cases},
	\qtext{and}
	\cvest{i}(\bar{z}) - z_i
	= 
	\begin{cases} 
		- \frac{n+1}{n} z_i  & \text{if $0 > z_i$}\\
		\bar{z}_i  - z_i  + \bar{z} & \text{if $2n\bar{z}< z_i$}\\
		- z_i & \text{otherwise}
	\end{cases}.
\ealignt
Let $C_1 = \{ i \in [n] \,:\, z_i \leq 0 \}$, $C_2 = \{ i \in [n]\,:\, z_i > 2\sqrt{2n} = 2n\bar{z}\}$ and $C_3 = \{ i \in [n] \,:\, z_i \not\in C_1 \cup C_2 \}$.
We have selected our dataset so that $C_2$ is empty.
Therefore
\balignt
2\ProxACV(\bar{z}) = 
	\frac{1}{n} \sum_{i \in C_1}  \frac{(n+1)^2}{n^2} z_i^2 
	+ \frac{1}{n} \sum_{i \in C_3}  z_i^2 
	= \frac{(n+1)^2}{n^2} \frac{1}{n} \sum_{i=1}^n z_i^2 -  (\frac{2}{n} + \frac{1}{n^2}) \frac{1}{n} \sum_{i \in C_3}  z_i^2 .
\ealignt
Meanwhile,
\balignt
2\ProxACV(0)
	&= \frac{1}{n} \sum_{i=1}^n (\bar{z}_i  - z_i)^2\\
	&= \frac{1}{n} \sum_{i=1}^n (\bar{z}  - (1+1/n)\bar{z})^2 + (1+1/n)^2(\bar{z}  - z_i)^2 \\
	&= \frac{1}{n^2} \bar{z}^2 +  \frac{(n+1)^2}{n^2}  \frac{1}{n} \sum_{i=1}^n z_i^2  -  \frac{(n+1)^2}{n^2}  \bar{z}^2 \\
	&= \frac{(n+1)^2}{n^2} \frac{1}{n} \sum_{i=1}^n z_i^2  - (1+ \frac{2}{n})\bar{z}^2.
\ealignt
Hence,
\balignt
2\ProxACV(0) - 2\ProxACV(\bar{z}) 
	= (\frac{2}{n} + \frac{1}{n^2}) \frac{1}{n} \sum_{i \in C_3}  z_i^2 - (1+ \frac{2}{n})\bar{z}^2
	= (\frac{2}{n} + \frac{1}{n^2}) \frac{a^2 }{2} - (1+ \frac{2}{n})\bar{z}^2
	= \frac{5}{n^2}.
\ealignt
\section{Additional Experiment Details}
\label{app:exp_details}

\subsection{\ProxACV versus \ACV and $\ACVIJ$}

This section provides additional experimental details for the experiment of \cref{sec:logistic}.
In this experiment, we use the exact experimental setup and code of \citep[App.\,F]{Stephenson2019sparse} with a modified number of datapoints ($n=150$).
Specifically, we employ an $\ell_1$ regularized logistic regression objective with $150$ feature coefficients plus an intercept coefficient. The data matrix of covariates is generated with i.i.d.\ $N(0,1)$ entries, and binary labels for each datapoint are generated independently from a logistic regression model with ground truth $\beta^\ast$ having its first $75$ entries drawn i.i.d. $N(0,1)$ and the rest set to zero. 
We solve the proximal Newton steps for \ProxACV using FISTA~\citep{FISTA}.

We compare with the non-smooth \ACV and $\ACVIJ$ extensions studied by \citep{ObuchiCV2016,ObuchiCV2018,rad2019scalable,wang2018,Stephenson2019sparse} and defined by restricting $\acvest{i}(\lambda)$ and $\acvijest{i}(\lambda)$ to have support only on  $ \hat S = \text{support}(\est(\lambda))$ and setting
\balignt
[\acvest{i}(\lambda)]^{\hat S} 
   \hspace{-.075cm} &=\hspace{-.075cm} [\est(\lambda)]^{\hat S}\hspace{-.075cm} +\hspace{-.075cm} \frac{1}{n}[\acvhess{\loss,i}^{\hat S,\hat S}]^{-1} [\grad_\beta \loss(z_i, \est(\lambda))]^{\hat S} \\
[\acvijest{i}(\lambda)]^{\hat S} 
   \hspace{-.075cm} &=\hspace{-.075cm}  [\est(\lambda)]^{\hat S}\hspace{-.075cm} +\hspace{-.075cm} \frac{1}{n}[\acvhess{\loss}^{\hat S,\hat S}]^{-1} [\grad_\beta \loss(z_i, \est(\lambda))]^{\hat S},\label{eq:non-smoothACV}
 \ealignt
 where $X^{\cdot, \hat S}$ denotes the submatrix of $X$ with column indices in $\hat S$, where $\acvhess{\loss}$ and $\acvhess{\loss,i}$ are given by~\cref{eq:proxacvij} and~\cref{eq:proxacvest}, respectively.

\subsection{\ProxACV Speed-up}
This section provides additional experimental details for the experiment of \cref{sec:speedup}.
For this experiment, we employed the standard graphical Lasso objective,
\balignt
\obj(\Pemp, \beta, \lambda) =  - \log \det \beta + \text{tr}(\beta S) + \lambda \sum_{j,k = 1}^p |\beta_{jk}|, \\
\obj(\Pcv{i}, \beta, \lambda) =  - \log \det \beta + \text{tr}(\beta S_{-i}) + \lambda \sum_{j,k = 1}^p |\beta_{jk}|, 
\ealignt 
where $\beta$ is now a positive-definite matrix in $\reals^{p\times p}$, $S = \frac{1}{n-1}\sum_{i=1}^n (z_i - \mu)(z_i - \mu)^\top, $ $S_{-i} = \frac{1}{n-1}\sum_{j\neq i} (z_j - \mu_{-i})(z_j - \mu_{-i})^\top, $
 for $\mu = \frac{1}{n}\sum_{i=1}^n z_i$, and $\mu_{-i} = \frac{1}{n-1}\sum_{j\neq i} z_j$.

\end{document}

%% file: lwm-macros.tex
\def\balign#1\ealign{\begin{align}#1\end{align}}
\def\baligns#1\ealigns{\begin{align*}#1\end{align*}}
\def\balignat#1\ealign{\begin{alignat}#1\end{alignat}}
\def\balignats#1\ealigns{\begin{alignat*}#1\end{alignat*}}
\def\bitemize#1\eitemize{\begin{itemize}#1\end{itemize}}
\def\benumerate#1\eenumerate{\begin{enumerate}#1\end{enumerate}}

\newenvironment{talign*}
 {\csname align*\endcsname}
 {\endalign}
\newenvironment{talign}
 {\csname align\endcsname}
 {\endalign}

\def\balignst#1\ealignst{\begin{talign*}#1\end{talign*}}
\def\balignt#1\ealignt{\begin{talign}#1\end{talign}}
\newcommand{\notate}[1]{\textcolor{blue}{\textbf{[#1]}}}

\newcommand{\qtext}[1]{\quad\text{#1}\quad}

\let\originalleft\left
\let\originalright\right
\renewcommand{\left}{\mathopen{}\mathclose\bgroup\originalleft}
\renewcommand{\right}{\aftergroup\egroup\originalright}

\def\tinycitep*#1{{\tiny\citep*{#1}}}
\def\tinycitealt*#1{{\tiny\citealt*{#1}}}
\def\tinycite*#1{{\tiny\cite*{#1}}}
\def\smallcitep*#1{{\scriptsize\citep*{#1}}}
\def\smallcitealt*#1{{\scriptsize\citealt*{#1}}}
\def\smallcite*#1{{\scriptsize\cite*{#1}}}

\def\mbi#1{\boldsymbol{#1}} %

\def\mbb#1{\mathbb{#1}}

\def\textsum{{\textstyle\sum}} %
\def\reals{\mathbb{R}} %

\def\<{\left\langle} %
\def\>{\right\rangle}

\def\defeq{\triangleq} %
\newcommand{\bs}{\backslash} %
\def\half{\frac{1}{2}}

\newcommand{\textfrac}[2]{{\textstyle\frac{#1}{#2}}}

\newcommand{\psdle}{\preccurlyeq}
\newcommand{\psdge}{\succcurlyeq}

\def\norm#1{\left\|{#1}\right\|} %
\newcommand{\onenorm}[1]{\norm{#1}_1} %
\newcommand{\twonorm}[1]{\norm{#1}_2} %
\newcommand{\infnorm}[1]{\norm{#1}_{\infty}} %
\newcommand{\opnorm}[1]{\norm{#1}_{\mathrm{op}}} %
\def\staticnorm#1{\|{#1}\|} %
\newcommand{\inner}[2]{\langle{#1},{#2}\rangle} %

\def\mineig#1{\lambda_{\mathrm{min}}\left({#1}\right)}

\def\indic#1{\mbb{I}\left[{#1}\right]} %
\def\staticindic#1{\mbb{I}[{#1}]} %
\def\E{\mbb{E}} %

\def\Esub#1{\E_{#1}}

\def\P{\mbb{P}} %

\newcommand{\grad}{\nabla} %
\newcommand{\Hess}{\nabla^2} %

\newcommand{\toprob}{\stackrel{p}{\to}}

\newcommand{\iid}{\textrm{i.i.d.}\xspace}

\providecommand{\argmin}{\mathop\mathrm{arg min}}

\providecommand{\dom}{\mathop\mathrm{dom}}
\providecommand{\diag}{\mathop\mathrm{diag}}

\ifdefined\nonewproofenvironments\else
\ifdefined\ispres\else
\newtheorem{theorem}{Theorem}
\newtheorem{lemma}[theorem]{Lemma}

\newtheorem{definition}{Definition}

\renewenvironment{proof}{\noindent\textbf{Proof}\hspace*{1em}}{\qed\\}
\newenvironment{proof-sketch}{\noindent\textbf{Proof Sketch}
  \hspace*{1em}}{\qed\bigskip\\}
\newenvironment{proof-idea}{\noindent\textbf{Proof Idea}
  \hspace*{1em}}{\qed\bigskip\\}
\newenvironment{proof-of-lemma}[1][{}]{\noindent\textbf{Proof of Lemma {#1}}
  \hspace*{1em}}{\qed\\}
\newenvironment{proof-of-theorem}[1][{}]{\noindent\textbf{Proof of Theorem {#1}}
  \hspace*{1em}}{\qed\\}
\newenvironment{proof-attempt}{\noindent\textbf{Proof Attempt}
  \hspace*{1em}}{\qed\bigskip\\}

\fi

\newtheorem{proposition}[theorem]{Proposition}

\newtheorem{assumption}{Assumption}
\fi